%% file: main.tex
\documentclass[10pt]{article} 
\usepackage[accepted]{tmlr}

\input{math_commands}
\input{macros}

\input{paper_macros}
\usepackage[acronym]{glossaries}
\makeglossaries

\usepackage{amsmath}
\usepackage{amssymb}
\usepackage{amsthm}
\usepackage{dsfont}
\usepackage{mathtools}
\usepackage{arydshln}
\allowdisplaybreaks

\usepackage{graphicx}
\usepackage{wrapfig}

\usepackage{booktabs}
\usepackage{caption}

\usepackage{xcolor}
\usepackage{parskip} 
\usepackage{soul} 
\sethlcolor{yellow}
\usepackage{tcolorbox}

\usepackage{hyperref}
\usepackage{cleveref}

\usepackage{listings}
\usepackage{fancyvrb}
\fvset{fontsize=\normalsize}

\lstdefinestyle{mystyle}{
    commentstyle=\color{OliveGreen},
    keywordstyle=\color{BurntOrange},
    numberstyle=\tiny\color{black!60},
    stringstyle=\color{MidnightBlue},
    basicstyle=\ttfamily,
    breakatwhitespace=false,
    breaklines=true,
    captionpos=b,
    keepspaces=true,
    numbers=left,
    numbersep=5pt,
    showspaces=false,
    showstringspaces=false,
    showtabs=false,
    tabsize=2
}
\usepackage{tikz}
\usetikzlibrary{shapes,decorations,arrows,calc,arrows.meta,fit,positioning}
\tikzset{
    -Latex,auto,node distance =1 cm and 1 cm,semithick,
    state/.style ={circle, draw, minimum width = 0.7 cm},
    detstate/.style ={rectangle, draw, minimum width = 0.7 cm, minimum height = 0.7 cm},
    point/.style = {circle, draw, inner sep=0.04cm,fill,node contents={}},
    bidirected/.style={Latex-Latex,dashed},
    el/.style = {inner sep=2pt, align=left, sloped}
}
\usepackage[algoruled, algo2e]{algorithm2e}
\usepackage[utf8]{inputenc} 
\usepackage[T1]{fontenc}    
\usepackage{hyperref}       
\usepackage{url}            
\usepackage{booktabs}       
\usepackage{amsfonts}       
\usepackage{nicefrac}       
\usepackage{microtype}      
\usepackage{xcolor}         
\usepackage{multirow}
\usepackage{graphicx}
\usepackage{subcaption}
\usepackage{adjustbox}
\usepackage{cleveref}
\crefname{assumption}{assumption}{assumptions}
\usepackage{wrapfig}
\usepackage{algorithm}
\usepackage{algpseudocode}
\setlist[enumerate]{leftmargin=15pt}
\setlist[itemize]{leftmargin=20pt}

\title{Open-Set Domain Adaptation Under Background Distribution Shift: Challenges and A Provably Efficient Solution}

\author{\name Shravan Chaudhari \email schaud35@jh.edu \\
\addr Department of Computer Science, Johns Hopkins University, Baltimore, MD, USA.
\AND
\name Yoav Wald \email yoav.wald@technion.ac.il \\
\addr Faculty of Data and Decision Sciences, Technion, Haifa, Israel. \\
\addr Center for Data Science, New York University, New York, NY, USA. \\
\AND
\name Suchi Saria \email ssaria1@jhu.edu\\
\addr Department of Computer Science, Johns Hopkins University, Baltimore, MD, USA.\\
\addr Bayesian Health, New York, NY, USA.
}

\def\confmonth{05}  
\def\confyear{2026} 
\def\openreview{\url{https://openreview.net/forum?id=uAJDta7VaQ}} 

\begin{document}

\maketitle

\begin{abstract}
As we deploy machine learning systems in the real world, a core challenge is to maintain a model that is performant even as the data shifts.
Such shifts can take many forms: new classes may emerge that were absent during training, a problem known as open-set recognition, and the distribution of known categories may change.
Guarantees on open-set recognition are mostly derived under the assumption that the distribution of known classes, which we call \emph{the background distribution}, is fixed. In this paper we develop \ours{}, a method that is guaranteed to solve open-set recognition even in the challenging case where the background distribution shifts. We prove that the method works under benign assumptions that the novel class is separable from the non-novel classes, and provide theoretical guarantees that it outperforms a representative baseline in a simplified overparameterized setting.
We develop techniques to make \ours{} scalable and robust, and perform comprehensive empirical evaluations on image and text data. The results show that \ours{} significantly outperforms existing open-set recognition methods under background shift. Moreover, we provide new insights into how factors such as the size of the novel class influences performance, an aspect that has not been extensively explored in prior work. Code is available at \url{https://github.com/Shra1-25/CoLOR}.

\end{abstract}


\input{sections-rewrite/1_introduction_rewrite_2}
\input{sections/2_related_work}

\input{sections-rewrite/3_setting-rewrite}

\input{sections-rewrite/4_PU_theory}

\input{sections/4_adaptation_representation}

\input{sections/5_experiments}
\input{sections/5a_Results}

\input{sections/6_Discussion}
\input{sections/7_Conclusion}
\input{sections/acknowledgments}

\bibliography{main}
\bibliographystyle{tmlr}

\newpage
\appendix

\section{Appendix} \label{sec:appendix}
\input{sections/notations}

\input{sections-rewrite/subappendix}

\input{sections/appendix}

\begin{center}  
  {\tt {\textbackslash}usepackage[accepted]\{tmlr\}}. 
\end{center}

\end{document}

%% file: math_commands.tex
\usepackage{amsmath,amsfonts,bm}

\def\eqref#1{equation~\ref{#1}}

\def\1{\bm{1}}

\def\rvw{{\mathbf{w}}}
\def\rvx{{\mathbf{x}}}

\DeclareMathAlphabet{\mathsfit}{\encodingdefault}{\sfdefault}{m}{sl}
\SetMathAlphabet{\mathsfit}{bold}{\encodingdefault}{\sfdefault}{bx}{n}

\def\gA{{\mathcal{A}}}

\def\gH{{\mathcal{H}}}

\def\gL{{\mathcal{L}}}

\def\gN{{\mathcal{N}}}

\def\gS{{\mathcal{S}}}
\def\gT{{\mathcal{T}}}

\def\gX{{\mathcal{X}}}
\def\gY{{\mathcal{Y}}}

\def\sR{{\mathbb{R}}}

\newcommand{\E}{\mathbb{E}}

\DeclareMathOperator*{\argmin}{arg\,min}

\renewcommand{\mid}{~\vert~}

\newcommand{\cL}{\mathcal{L}}
\newcommand{\cN}{\mathcal{N}}

\newcommand{\bbR}{\mathbb{R}}

%% file: macros.tex
\usepackage{comment}
\usepackage{amsmath}
\usepackage{amssymb}
\usepackage{amsthm}
\usepackage{mathtools}
\usepackage{enumitem}
\usepackage{xspace}

\newtheorem{definition}{Definition}
\newtheorem*{definition*}{Definition}
\newtheorem*{lemma*}{Lemma}
\newtheorem*{theorem*}{Theorem}
\newtheorem{lemma}{Lemma}

\newtheorem{theorem}{Theorem}
\newtheorem{proposition}{Proposition}
\newtheorem{assumption}{Assumption}

\newcommand{\x}{{\mathbf x}}

\newcommand{\w}{{\mathbf w}}

\newcommand{\reals}{\mathbb{R}}

\newcommand{\N}{\mathcal{N}}

\DeclarePairedDelimiterX{\inp}[2]{\langle}{\rangle}{#1, #2}
\DeclarePairedDelimiterX{\abs}[1]{|}{|}{#1}

\renewcommand{\eqref}[1]{Equation~(\ref{#1})}

\newlength\myindent 
\makeatletter
\newcommand*{\@indep}[2]{%
  \sbox0{$#1\perp\m@th$}
  \sbox2{$#1=$}
  \sbox4{$#1\vcenter{}$}
  \rlap{\copy0}
  \dimen@=\dimexpr\ht2-\ht4-.2pt\relax
  \kern\dimen@
  {#2}%
  \kern\dimen@
  \copy0 
} 
\makeatother

%% file: paper_macros.tex
\newcommand{\auroc}{AUROC}
\newcommand{\auprc}{AUPRC}
\newcommand{\oscr}{OSCR}
\newcommand{\source}{Source-only}
\newcommand{\conoc}{ConOp}
\newcommand{\ours}{CoLOR}
\newcommand{\pulse}{PULSE}
\newcommand{\shot}{SHOT}
\newcommand{\zoc}{ZOC}

\newcommand{\arpl}{ARPL}
\newcommand{\cac}{CAC}

\newcommand{\osda}{Open Set Domain Adaptation}
\newcommand{\osr}{OSDA}
\newcommand{\uPUnew}{uPU}
\newcommand{\nnPUnew}{nnPU}
\newcommand{\DDnew}{DD}
\newcommand{\oursold}{CoNoC}
\newcommand{\BODA}{BODA}
\newcommand{\DD}{DD}
\newcommand{\uPU}{uPU}
\newcommand{\nnPU}{nnPU}
\newcommand{\anna}{ANNA}
\newcommand{\resnet}{ResNet50}
\newcommand{\vit}{ViT-L/14}
\newcommand{\problem}{Open-Set Domain Adaptation}

\newcommand{\Psource}{P_{\gS}}
\newcommand{\Ptarget}{P_{\gT}}
\newcommand{\datasource}{S_{\gS}}
\newcommand{\datatarget}{S_{\gT}}
\newcommand{\Plabel}[1]{P_{\gT, #1}}
\newcommand{\mub}{\boldsymbol{\mu}}
\newcommand{\etab}{\boldsymbol{\eta}}
\newcommand{\attendto}[1]{{\color{cyan}{}}} 
\newcommand{\yw}[1]{{\color{purple}{#1}}}

\newcommand{\notcheckmark}{{\cmark}\textsuperscript{\textcolor{black}{\kern-0.75em{\bf---}}}}

\usepackage{pifont}
\newcommand{\cmark}{\ding{51}}%
\renewcommand{\yw}[1]{\ignorespaces}

%% file: sections-rewrite/1_introduction_rewrite_2.tex
\section{Introduction}\label{sec:introduction}
Adapting Machine Learning models to shifts in the data distribution is pivotal to ensuring their robustness and safety in real world applications \citep{quinonero2008dataset, pmlr-v139-koh21a} for fields like healthcare \citep{finlayson2021clinician}, autonomous driving \citep{filos2020can, wong2020identifying} and more broadly in computer vision \citep{bendale2016towards}. Maintaining robustness to distribution shift in classification problems includes both identifying familiar objects under new conditions, while detecting the long tail of object categories that has not been observed in the past. In this work we study the problem of adaptation when these two types of distribution shift occur concurrently. Specifically, we address the emergence of novel categories \citep{panareda2017open} alongside shifting distributions of known classes, a scenario referred to as Open-Set Domain Adaptation (OSDA).  

\begin{figure}[h!]
    \centering
    \includegraphics[width=0.99\linewidth]{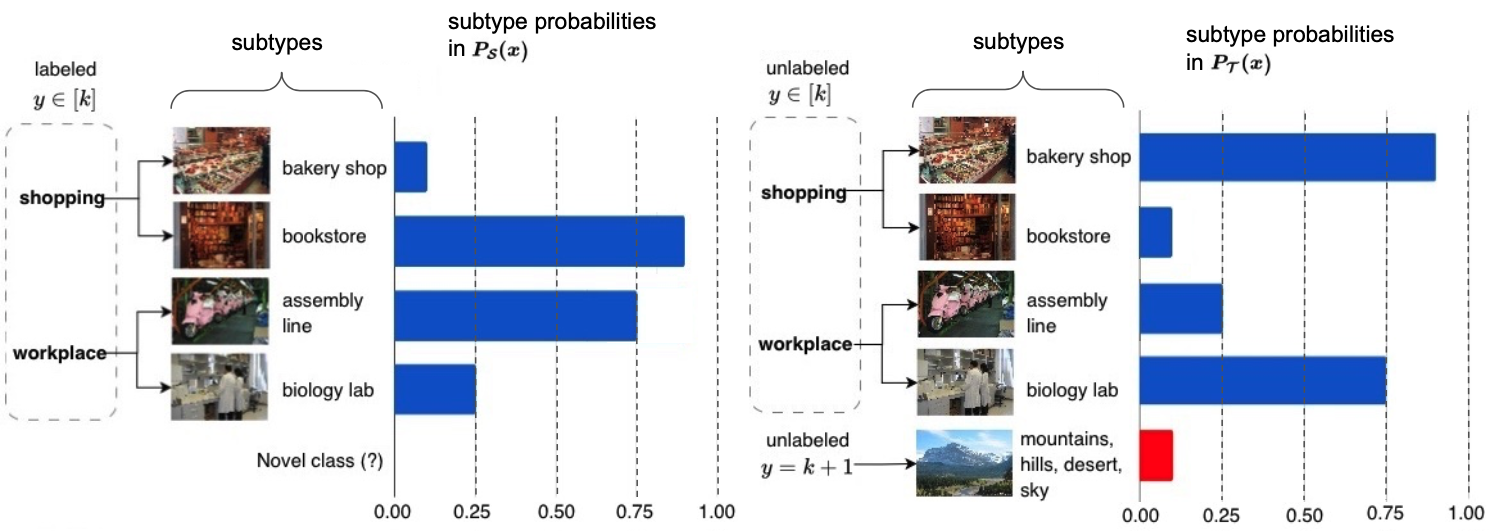}
    \caption{An example of \osr{} with background shift. There are two known classes, `shopping' and `workplace'. Each class can further have subtypes, but the learner has no access to the labels of these subtypes. \textbf{Left:} Source distribution with all samples from $k=2$ known categories with distribution $\Psource(x)$. \textbf{Right} Target distribution $\Ptarget(x)$ with samples from known categories as well as new classes (mountains, hills, etc.). $\Ptarget(x)$  consists of varying proportions of known categories and their latent subtypes along with shift due to emerging novel classes. Unlike pure label, covariate or conditional shift, background shift allows any of $\Psource(y)$, $\Psource(x|y)$ or $\Psource(x)$ to change in $\Plabel{}$ as long as the support overlap and separability assumptions are followed.}
    \label{fig:illustration}
\end{figure}
Let our training data be $\datasource$. The core task here is to detect the novel classes in target data $\datatarget$ collected under some conditions different from those observed at training time \citep{panareda2017open} while maintaining good performance over the existing classes. The emergence of a novel class in itself constitutes a shift between $\datasource$ and $\datatarget$, however, in real world scenarios \emph{it is likely not the only distribution shift that occurs} \citep{garg22adaptation, wald2023birds}. Our study concentrates on scenarios where the novel class exhibits a notable degree of separability from the known classes, and a distribution shift with overlapping supports between source and target distributions exists within the known classes, which we refer to as `background shift'. Figure \ref{fig:illustration} depicts a scenario of background shift within indoor scene images taken from the SUN397 dataset \citep{jianxiong2010sun397} while the novel class belongs to an outdoor scene category. Arguably, this is a very common case in domains such as healthcare, e.g. classifying histopathology slides for known and novel tumor types, where the background shift is due to differences in geographic locations \citep{pocevivciute2025out}. 

Several aspects of our work are understudied in the \osr{} literature; \emph{(i)} formally characterizing the shifts of interest; \emph{(ii)} developing methods with formal guarantees (finite sample bounds in our case) that are also effective in practice; \emph{(iii)} examining the effect of the novel class size on performance. Most research primarily address open-set recognition instead of \osr{} because they assume there is no distribution shift for known classes \citep{ge2017openmax, Neal_2018_ECCV, Lin_2019, Chen_2021, zeng-etal-2021-modeling, vaze2022openset, Esmaeilpour_Liu_Robertson_Shu_2022}. Other methods tackle \osr{} without formally characterizing `domain shifts' \citep{liang2021reallyneedaccesssource, li2021domainconsensus}. They typically validate their methods empirically on benchmarks such as Office-Home \citep{VenkateswaraECP17deephashing} or VisDA \citep{visda2017}, where the shift is simulated by changing the domain from real to sketch images or from synthetic to real world images. However, since these methods lack formal guarantees and rely on limited types of simulated shifts, their effectiveness beyond these datasets remain uncertain.
\citet{garg22adaptation} tackle \osr{} under shifting class probabilities in $\datatarget$ (label shift) and show improved results compared to prior methods. To this end, our theoretical and empirical results emphasize the regimes where our method is expected to outperform alternatives that do not address background shift in \osr{} properly. In particular, we show that our approach is more effective when the novel class constitutes a relatively small proportion of the target distribution—an important scenario in practice. Reliably identifying novel classes enables updating models or deferring to human oversight in safety-critical settings such as healthcare and autonomous driving \citep{amodei2016concrete, hendrycks2016baseline}.

Our contributions to these aspects of OSDA are thus the following:
\begin{itemize}[leftmargin=0.5cm]
	\item \textbf{Leverage and improve constrained learning rule for \osr{}:}
    We develop \ours{} (Constrained Learning for Open-set Recognition), a scalable method for \osr{} (\cref{sec:solutions}) that extends \citet{wald2023birds}'s small-model, no-known-class setting via architectural and algorithmic improvements effective for large models and high-dimensional image and text data. Our theory explains why and when the method performs favorably in the large-scale regime.
    \item{\textbf{Formal results on identifiability under background shift:}}
    We characterize sufficient conditions for identifiability of the novel class and analyze sample efficiency of \ours{} on a problem capturing key aspects of \osr{} with expressive models (\cref{sec:overparam_theory}). While \citet{garg22adaptation} solve \osr{} under label shift with assumptions weaker than our separability, we prove these fail under the more general background shift (Lemma~\ref{lem:impossibility_lemma}). In a linear-Gaussian framework, we further characterize regimes where \ours{} significantly outperforms a standard domain-discriminator baseline, especially when novel classes are rare.
    
    \item \textbf{Comprehensive empirical validation across diverse data modalities and datasets:} 
    Extensive experiments across image and text classification (\cref{sec:experiments}) confirm our theoretical findings. We observe that background shift substantially degrades \osr{} baselines, whereas our enhanced adaptation of the constrained learning rule, \ours{}, achieves superior performance on CIFAR100 \citep{Krizhevsky09learningmultiplecifar}, Amazon Reviews \citep{ni_justifying_2019} and SUN397 \citep{jianxiong2010sun397}. Notably, \ours{} remains robust even when the novel class constitutes a small portion of $\datatarget$, where other methods tend to struggle.
\end{itemize}

\textbf{Paper roadmap.} \cref{sec:prob_def} formally defines the problem; \cref{section:identifiability} establishes impossibility results and the assumptions under which it is solvable; \cref{sec:conoc_review} introduces a 0--1 loss learning rule in a linear Gaussian setting, and \cref{sec:overparam_theory} extends it to the overparameterized regime relevant to large-scale models, characterizing when it outperforms standard baselines; \cref{sec:solutions,sec:experiments} present the scalable algorithm and supporting experiments, with limitations discussed in \cref{sec:conclusion}.


%% file: sections/2_related_work.tex
\section{Related Work} \label{sec:related_lit}
Our work is closely related to several lines of work involving novelty detection, summarized below. \\
\textbf{Out-of-Distribution (OOD) Detection and Open Set Recognition:} Most works on detecting novel instances locate novelties within $\datatarget$ without adapting to it, using $\datasource$ alone to compute an ``anomaly score'' $s(\rvx)$ that ranks examples by their likelihood of being novel \citep{ruff2021unifying, Esmaeilpour_Liu_Robertson_Shu_2022}. Open Set Recognition (OSR) \citep{bendale2016towards, ge2017openmax, liu2018open, xu2019open, saito2018open, vaze2022openset} also detects novelties while classifying known classes correctly, but the learner has access to unlabeled $\datatarget$, allowing $s(\rvx)$ to adapt to the specific novel class; some OSR methods further refine $s(\rvx)$ via generative models on known instances \citep{Neal_2018_ECCV, Kong_2021_ICCV_opengan}. However, both OSR and OOD methods degrade under shift between $\Psource$ and $\Plabel{[k]}$ \citep{scholkopf2001estimating, cao2023anomalydetectiondistributionshift}, motivating shift-adaptive methods. \\
\textbf{\osda{}:} Works that seek to adapt to $\Plabel{[k]}$, while correctly classifying known instances and detecting novel ones, fall into this category.
Recent \osr{} works primarily focus on narrow domain shifts such as real-to-sketch or synthetic-to-real \citep{wen24crossdomain, panareda2017open, choe2024osdasemanticsegmentation, Hur2023unida, Zhu2023unida}, limiting generalizability to broader shifts like background shift; \citet{garg22adaptation} show these methods are not versatile across datasets and propose a method specifically for label shift. Other works leverage multi-modal foundation models like ChatGPT, DALL-E and CLIP for \osr{} \citep{qu2023lmc, Esmaeilpour_Liu_Robertson_Shu_2022}, but their pretraining data is not publicly accessible, making it infeasible to curate truly novel classes or shifts and to evaluate fairly. Most \osr{} methods lack formal guarantees or a clear characterization of the shifts they address; well-studied Positive and Unlabeled (PU) learning offers a framework to develop more generalizable methods with guarantees \citep{garg22adaptation, wald2023birds}.\\
\textbf{PU-Learning Under Distribution Shift:} It is formally equivalent to OSDA where k=1. In the absence of distribution shift, effective solutions with formal guarantees already exist, primarily based on adjustments to ``Domain Discriminator'' models (that are trained to distinguish source and target distributions i.e. $\Psource$ and $\Ptarget$ respectively) \citep{garg2021mixture, duplessis2014analysis, blanchard2010semi, elkan2008learning}. 
Under distribution shift, domain discriminator type methods still have guarantees when given infinite data \citep{gerych2022recovering}. In practice, however, we will see that they underperform (\cref{table:results_sumary}), even when the separability assumption (see \cref{ass:separability}) holds.
\citet{garg22adaptation} propose a $k$-way PU learning based method for identifying unseen classes under special case of label shift among $k$ known classes, given certain assumptions. We show that these assumptions are insufficient to handle the more general case of background shift, where additional conditions are required for reliable recovery.  
Furthermore, \citet{wald2023birds} propose an algorithm with finite-sample guarantees and demonstrate its performance on small-scale models. In contrast, our formal analysis targets the large-scale, overparameterized learning regime, where their results no longer apply \citep{belkin2019reconciling}. We extend these ideas to the setting of \osr{} under background shift.

%% file: sections-rewrite/3_setting-rewrite.tex
\section{Open-Set Domain Adaptation under Background Shift – Problem Definition and Identifiability} \label{sec:setting}
In this section, we review the problem setting and develop the theory that motivates our solution. \Cref{sec:prob_def} reviews the problem definition and \cref{section:identifiability} discusses conditions for identifiability, such that a learner provided with unlimited data can solve the problem. Notably, we show that conditions studied in prior work for the case of label shift, are insufficient for identifiability under background shift. 

\subsection{Problem Definition}
\label{sec:prob_def}
For a prediction task with $k$ classes, we are interested in detecting the emergence of a novel class where $y=k+1$, while classifying a set of known classes $y=1, \ldots, k$. To treat this formally, we assume a learner is provided with two datasets, the training set $\datasource = \{ \rvx_i, y_i \}_{i=1}^{N_{\gS}}$ and an unlabelled target dataset $\datatarget = \{\rvx_i\}_{i=1}^{N_{\gT}}$. The datasets $\datasource$ and $\datatarget$ are sampled i.i.d from joint distribution over $(\rvx\in \gX,y\in\gY)$, $\Psource$ and $\Ptarget$ respectively, which take on values in $\gX, \gY=[k+1]$ and $\Psource(y=k+1) = 0$. $\Ptarget$ is a mixture distribution:
\begin{align} \label{eq:ptarget}
\Ptarget(\rvx, y) = (1-\alpha)\Plabel{[k]}(\rvx, y) + \alpha\Plabel{k+1}(\rvx).
\end{align}
Here we use the notation $\alpha := \Ptarget(y=k+1)$, $\Plabel{k+1}(\rvx) := \Ptarget(\rvx \mid y=k+1)$, and $\Plabel{[k]}(\rvx, y) := \Ptarget(\rvx, y \mid y \neq k+1)$.
Our goal is to learn a model that minimizes the error over all classes, including the novel $k+1$-th class. All errors will be measured w.r.t $\Ptarget$. To do this, we learn a model $h: \gX\rightarrow \sR^{k+1}$ that gives a score for each class. We also assume that a novelty detection score $h_{\mathrm{novel}}: \gX\rightarrow \sR$, derived from the overall model, is given and we will define a binary decision threshold at $0$ where $h_{\mathrm{novel}}(\rvx) > 0$ means $\rvx$ is classified as a novelty.\footnote{The threshold value is arbitrary and can be changed depending on the specific setting.} For example, $h_{\mathrm{novel}}(\rvx) = h(\rvx)_{k+1} - \max_{\hat{y}\in{[k]}} h(\rvx)_{\hat{y}}$ or $h(\rvx)_{k+1}$ when $h(\rvx)$ returns a probability distribution.

The success of an \osr{} model may be determined by its classification accuracy on data from $\Ptarget$, $\mathcal{R}^{l_{01}}_\mathcal{T}(h) = \E_{y,\rvx\sim \Ptarget}[\mathrm{arg}\max_{\hat{y}}\{h(\rvx)_{\hat{y}}\} \neq y]$, but since we often wish to emphasize novelty detection and not treat the novel class like any other known class, several other notions of success have been defined in the literature. Our theory focuses on the simple case in which $k=1$, where such considerations are unnecessary, and standard metrics for binary classification can be used.
In our experiments, as we mention again in \cref{sec:experiments}, we also use the \oscr{} metric of \citet{dhamija2018oscr}. Below we formally define the problem setting and other performance metrics we use for our theory. 
\begin{definition} [Open Set Domain Adaptation with Background Shift]\label{def:osda_problem}
An \osr{} problem with hypothesis class $\mathcal{H}$ is a tuple $\langle \Psource(\rvx, y), \Plabel{[k]}(\rvx,y),\Plabel{k+1}(\rvx), \alpha, N_\mathcal{S}, N_\gT, N_{\gT,0}, N_{\gT,1} \rangle$, where we are given $N_\gS$ and $N_{\gT}$ i.i.d examples from $\Psource$ and $\Ptarget$ (defined in \cref{eq:ptarget}) respectively. The problem undergoes \textbf{background shift} whenever $\Psource(\rvx, y) \neq \Plabel{[k]}(\rvx, y)$.

We further define $\beta(h; t) = \mathbb{E}_{\rvx\sim\datasource}\left[\mathbf{1}\left[ h_{\mathrm{novel}}(\rvx) > t \right]\right]$, $\alpha(h) = \mathbb{E}_{\rvx\sim\datatarget}\left[\mathbf{1}\left[ h_{\mathrm{novel}}(\rvx) > t \right]\right]$ as the False Positive Rate (FPR) and recall respectively, of a novelty detector $h_{\mathrm{novel}}$ derived from $h\in\mathcal{H}$ where $t\in{\bbR}$ is a scalar decision threshold. For notational convenience, we will define $\beta(h):=\beta(h; 0)$ and likewise for $\alpha(h)$.

\end{definition}
Note that we use background shift as a general term for shifts in the distribution of known classes, i.e. between $\Psource$ and $\Plabel{[k]}$, that includes common instances such as label, covariate, and conditional shift, studied in domain adaptation \citep{lalou2025skadabench}. The term `background shift' is used instead of generic `shift', to emphasize the open-set aspect of the problem which is our focus, where the emergence of the novel class (via the addition of $\Plabel{k+1}$) together with the background shift constitute the entire shift between $\Psource$ and $\Ptarget$.

Without any assumptions on the relationship between $\Psource$ and $\Ptarget{[k]}$ it is impossible to obtain a guarantee of better than chance accuracy \cite{garg22adaptation, david2010impossibility}. This holds even under strong simplifying assumptions, such as the existence of a model that solves the problem without errors: that is, $\mathcal{R}^{l_{01}}_\mathcal{T}(h^*)=0$ for some hypothesis $h^*$ \footnote{see Prop.~1 in \citet{garg22adaptation}, Prop.~2.1 in \citet{wald2023birds}, or discussion in \citet{bekker2019beyond} for the results of this flavor.}. In this paper we will work with the overlap and separability assumptions Combining the conditions for background with the above assumption about the existence of an oracle $h^*$, we obtain the separability assumption with which we will work.

\subsection{Necessary and Sufficient Assumptions for OSDA under Background Shift} \label{section:identifiability}


\begin{assumption}[separability] \label{ass:separability}
There exists $h^*\in{\gH}$ such that $\mathcal{R}_{\mathrm{novel}}^{l_{01}}(h^*)=0$. Furthermore, it holds that $ \mathrm{Supp}(\Plabel{[k]}) \subseteq \mathrm{Supp}(\Psource)$ i.e. background shift exists between $\Psource$ and $\Plabel{[k]}$
\end{assumption}
According to this assumption, we consider any background shift between $\Psource$ and $\Ptarget{[k]}$ that maintains support overlap such that $ \mathrm{Supp}(\Plabel{[k]}) \subseteq \mathrm{Supp}(\Psource)$.
Moreover, both parts of the assumption are rather intuitive, and hold at least approximately for many problems we consider in \osr{} or novelty detection, such as detection of novel semantic visual concepts. Let us emphasize two aspects of our assumption.

\textbf{Separability and background shift characteristics in known classes.} The separability assumption is required only for the novel class $Y=k+1$ vs. the known ones, we do not explicitly limit the $k$ known classes to be separable amongst themselves. Such an assumption would have placed the shift between $\Psource$ and $\Plabel{[k]}$ purely in the realm of covariate shift \citep{shimodaira2000improving}. However, background shift with \cref{ass:separability} facilitates more general forms of label shift and covariate shift. In our experiments we mostly create distribution shifts such that $\Plabel{[k]} (X \vert Y) \neq \Psource (X \mid Y)$ which follow the definition of background shift that are underexplored in existing works. 

\textbf{Insufficiency of less stringent assumptions.} While generalization bounds for domain adaptation are well known from seminal works such as \citet{bendavid2010adaptation}, they are less common in the Open Set learning literature, hence let us focus on this aspect of our problem, i.e. detecting the novel class. To the best of our knowledge the only characterized sufficient and necessary conditions for (non-separable) OSDA are those in \citet{garg22adaptation}, which \emph{hold only for label shift}, where it is assumed that $\Psource(X \mid Y=y) = \Ptarget(X \mid Y=y)$ for all $y\in{[k]}$. For this special case, they propose two assumptions which are sufficient (when added on top of the label-shift assumption) to guarantee $h^*$ can be learned from observed data. Their first assumption is \emph{(Strong Positivity)}: there exists $X_{sep}\in{\gX}$ such that $\Plabel{k+1}(X_{sep}) = 0$ and the matrix $[\Psource(\rvx \mid y)]_{\rvx\in{X_{sep}}, y\in{[k]}}$ is full rank and diagonal. We show that once more general shifts than label shifts are allowed, e.g. background/covariate shift, this condition is no longer sufficient and in fact no algorithm can guarantee better-than-chance detection. The proof for this is in \ref{sec:proof}.
\begin{lemma} \label{lem:impossibility_lemma}
Let $\gA$ be an algorithm for \problem. There are distributions $\Psource, \Plabel{[k]}, \Plabel{k+1}$ such that the problem satisfies strong positivity, and $\exists h^*\in{\gH}$ for which $\mathcal{R}^{l_{01}}_{\gT}(h^*)=0$, while $\E_{S_{\gS}, S_{\gT}}\left[ \mathcal{R}^{l_{01}}_{\gT}(\gA(S_{\gS}, S_{\gT})) \right] \geq 0.5$.
\end{lemma}
The second assumption proposed in \citet{garg22adaptation} is separability as defined in \cref{ass:separability}; however, since their assumption is combined with the label shift assumption, the methods they develop are tailored to that scenario and do not apply to our problem. \footnote{\Cref{lem:impossibility_lemma} has a similar flavor to Prop.~3.1 in \citet{wald2023birds}, but it is stronger. Our result shows impossibility under an additional assumption of strong positivity. This is significant in the context we consider here, as \citet{garg22adaptation} give guarantees on \osr{} with label shift under this strong positivity assumption, but our lemma shows that this is impossible under the background shifts we consider.}

\Cref{ass:separability} guarantees that given an infinitely large sample and an optimization oracle, we can learn to identify the novelties. It is also desirable to have guarantees on the required sample size. In this paper we focus on the finite sample guarantees of detecting the novel class in the overparameterized regime, since results on classification of remaining $k$ non-novel classes can be obtained from prior work on domain adaptation, see for instance \citet{ben2010theory} and the survey of \citet{redko2020survey}. 
This is to formally study and understand best practices in learning expressive models for open-set problems under background shift.

%% file: sections-rewrite/4_PU_theory.tex
\section{Constrained Open-Set 
Learning Rules and Domain Discriminators} 
\label{sec:conoc_review}
We start by proposing the underlying statistical learning rule which our method build upon and call it simplified \ours{}. We further analyze and compare it with other standard domain discriminator baseline over a synthetic problem setup using linear Gaussians. We focus on the case $k=1$, i.e., PU-learning \citep{bekker2020learning}, to study the novelty detection aspect of \osr{}. 
The domain discrimination baseline trains a classifier, via Empirical Risk Minimization (ERM), to distinguish between $\Psource$ and $\Ptarget$ such that for a 0-1 loss function $l_{01}$, $h_{\mathrm{\text{DD}}} = \argmin_{h\in\gH}\frac{1}{N_\gS+N_\gT}\big(\sum_{\rvx \in \datasource} l_{01}(h(\rvx),0) + \sum_{\rvx\in\datatarget}l_{01}(h(\rvx),1)\big)$. Similarly, the learning rule for constrained learning is 
\begin{equation}\label{eq:color01}
h_{\mathrm{\text{color}}} = \argmin_{h\in\gH}\frac{1}{N_\gS}\big(\sum_{\rvx\in \datasource}l_{01}(h(\rvx),0)\big)  \text{ s.t. } \frac{1}{N_\gT}\sum_{\rvx\in\datatarget} l_{01}(h(\rvx),1)\leq 1-\hat{\alpha}    
\end{equation}
In practice, these rules are often implemented with the logistic loss
or a close  variation, e.g. \citep{kiryo2017positive}.

As noted in \cref{sec:related_lit}, standard PU-learning methods without distribution shift rely on domain discriminators. Constrained learning rules were studied in \citet{blanchard2010semi} and later extended under distribution shift by \citet{wald2023birds} using small models. 
Their analysis, however, is limited to classical learning regimes that assume infinite data, whereas in practice, we prefer high-capacity models and only have limited data. A regime that is beyond the scope of their theory. In this section, we review both approaches and explain why applying these guarantees to such high-capacity models is non-trivial. Our main result in \cref{sec:overparam_theory} compares constrained learning and domain discrimination in a simplified overparameterized setting that preserves key aspects of open-set training. The model builds on ideas from generalization to minority subgroups \citep{sagawa2020investigation, wald2022malign, puli2023don, nagarajan2021understanding}. Empirical results in \cref{sec:experiments} corroborate conclusions to which the theory points.
\subsection{Overparameterization and Finite Sample Analysis}
\label{sec:finite_sample}
In terms of formal guaranties, it is well known that at the limit of infinite data and a family $\gH$ that is large enough, the minimizer $h_{\mathrm{\text{DD}}}(\rvx)$ will be proportional to the log-odds ratio, $h_{\mathrm{\text{DD}}}(\rvx) = \log(\Ptarget(\rvx)/\Psource(\rvx))$.
Under \cref{ass:separability}, this means that the instances of the novel class will be classified as originating from $\Ptarget$ w.p. $1$, while for $\rvx\sim \Plabel{0}$ ($=\Plabel{[k]}$ since $k=1$) it will be bounded away from $1$ since $\mathrm{Supp}(\Plabel{[k]}) \subseteq \mathrm{Supp}(\Psource)$. Hence, under these favorable conditions, $h_{\mathrm{\text{DD}}}$ will classify the novel class vs. the single known class with accuracy $1$.
In practice, we do not perform perfect optimization with infinitely large datasets and highly expressive models. Therefore, it is desirable to study the sample complexity of different learning rules. Without further assumptions, the sample complexity of the domain discrimination approach resembles those in the domain adaptation literature, which depend on the divergence between $\Psource$ and $\Ptarget$, e.g. \citep{ben2010theory}. \citet{wald2023birds} show that instead of following ERM, solving a constrained learning rule, 
achieves a generalization bound where the error scales with a divergence $d_{\gH, \beta}(\Psource \| \Ptarget)$ which is usually considerably smaller than common divergences in domain adaptation \footnote{Intuitively, $\beta$ is a small number, and $d_{\gH, \beta}(\Psource \| \Ptarget)$ measures how probable a rare event under $\Psource$, i.e. with probability smaller than $\beta$, can become under $\Ptarget$}. Here, $\beta$ is a constant that depends on the Rademacher complexity of $\gH$ and the sample size. Thm 4.3 in \citet{wald2023birds} provide more details.

\textbf{Why existing theory may not reflect large-scale training.} As reflected by dependence on quantities like the Rademacher complexity of $\gH$, the results on constrained learning apply to a regime of lower capacity models that may not be expressive enough to achieve arbitarily low training loss. However, it is rather common to fit much more expressive models, as they often generalize better in practice. Next, we formally study a simple example of overparameterized models, where the number of parameters is larger than the number of training examples. 
This is a common toy model used to study deep networks
that overfit their training data but still generalize well, also known as ``benign overfitting'' 
\citep{belkin2019reconciling}. 
\subsection{Why Should Constrained Learning be Effective? Linear-Gaussian Overparameterized \osr{} Example with $k=1$}
\label{sec:overparam_theory}

To capture key aspects of the problem solved in practice, while maintaining a manageable mathematical analysis, we perform a few simplifications: \textbf{Data and models.} (i) We study linear models $h(\rvx) = \rvw^\top\rvx$ for $\rvw\in{\sR^d}$, where $d > N_{\gS} + N_{\gT}$ so the models are expressive and can interpolate the data; (ii) We focus on a problem with two features (along $\mub$ and $\etab$) and Gaussian noise (see \cref{def:overparam_setting}). Here $\mub$ and $\etab$ denote the two informative directions in the problem: the background shift, i.e. the shift in the known class, is in the direction $\mub$, and its magnitude is proportional to the norm $r_\mu$. While $\etab$ marks the novel class, and higher $r_{\eta}$ corresponds to a larger signal for learning the novel class. The covariance matrices (e.g., $I_d - r_\mu^{-2}\mub\mub^\top$) remove the corresponding rank-one components so the problem satisfies \cref{ass:separability}.
This construction keeps the model analytically tractable 
and this type of simplifications are common in works on benign overfitting and robustness to distribution shifts \citep{nagarajan2021understanding, wald2022malign, muthukumar2021classification}.
\begin{definition} \label{def:overparam_setting}
Let $d>0$, $\mub, \etab\in{\sR^d}$, 
$\sigma=1/\sqrt{d}$, $\alpha\in{(0, 1)}$, $r_\mu = \|\mub\|$, $r_\eta = \|\etab\|$. A \emph{Linear-Gaussian PU-learning problem} is an \osr{} problem (\cref{def:osda_problem}) with $1$ known class where $\Psource = \gN(\mub, \sigma^2(I_d-r_{\eta}^{-2}\etab\etab^\top))$,
$P_{\mathcal{T}, 0} = \gN(-\mub, \sigma^2(I_d-r_{\eta}^{-2}\etab\etab^\top)$ and $P_{\mathcal{T}, 1} =\gN(-\etab, \sigma^2 (I_d - r^{-2}_{\mu}\mub\mub^\top))$.  
\end{definition}
\textbf{Simplified \ours{} and Domain Discriminator methods.} (iii) The last simplification regards which optimization problem to analyze. In an overparameterized problem, there are many possible minimizers of the empirical risk and the constrained risk in \cref{eq:learningrule}. Therefore to analyze the solutions, it is common to consider $h_{\mathrm{\text{DD}}}$ and $h_{\mathrm{\text{color}}}$ as max-margin classifiers \footnote{This is based on results that show the implicit bias of gradient descent and the logistic loss to the max-margin classifier on separable data \citep{soudry2018implicit}} such that: 
\begin{equation*}\label{eq:DD_color_objective_unref}
\begin{aligned}
\rvw_{\text{\text{DD}}} & = \arg\min_{\rvw}\,\|\rvw\|& \\
\text{s.t.}\; & \rvw^\top\rvx \le -1  &\forall \rvx\in \datasource \\
              & \rvw^\top\rvx \ge 1   &\forall \rvx\in \datatarget
\end{aligned}
\qquad
\begin{aligned}
\rvw_{\mathrm{\text{color}}} & = \arg\min_{\rvw}\,\|\rvw\|& \\
\text{s.t.}\;       & \rvw^\top\rvx \le 0  &\forall \rvx\in \datasource \\
                    & \rvw^\top\rvx \ge 1 &\forall \rvx\in \datatarget
\end{aligned}
\end{equation*}
That is, while $\rvw_{\mathrm{\text{DD}}}$ is the max-margin classifier on the entire dataset, $\rvw_{\mathrm{\text{color}}}$ maximizes margin only on $\datatarget$ while constraining $\datasource$ to lie on the correct side of the decision threshold. Our practical implementation of the learning rule ($h_{\mathrm{\text{color}}}$) maximizes the margin on $\datasource$ while constraining $\datatarget$ as it leads to a more stable optimization. In the overparameterized regime, the conditions we lay out below ensure that both optimization problems admit feasible solutions with probability $1$. Now we present our main result for this section.

\begin{theorem}\label{thm:th1}
Consider a Linear-Gaussian PU-learning problem with parameters $\mub, \etab, d > 300$, dataset sizes $N_{\gT} > 10$, $N_{\gS} \geq N_{\gT}$, $N=N_{\gS}+N_{\gT}$, $\alpha \in{[\frac{1}{N_\gT}, \frac{N}{1024 N_\gT})}$ 
and let $\delta\in{(0,1)}$.
For all problems where
\begin{align*}
\min\{r_{\eta}, r_{\mu}\} \geq \tfrac{16}{\sqrt{N_\gT}}, \quad r_{\mu} \leq \tfrac{1}{2}\sqrt{N_{\gT, 0}}, \quad \tfrac{r_{\eta}}{r_{\mu}} \leq \tfrac{1}{8(1+\sqrt{N_{\gT}N_{\gT,1}})}, \quad c_1\log\!\left(\tfrac{c_2}{\delta}\right) \leq \min{\left(\sqrt{N}, \sqrt{\frac{d}{N}}\right)},
\end{align*}
it holds with probability at least $1-\delta$ that $\mathrm{AU-ROC}(\rvw_\mathrm{{\text{DD}}}) < 0.5$ and $\mathrm{AU-ROC}(\rvw_{\mathrm{\text{color}}}) > 0.9$.

\end{theorem}
Detailed derivation of the \cref{thm:th1} claims are provided in \cref{sec:propositions}.
The result shows that a constrained learning rule can be arbitrarily more accurate than a domain discriminator.\footnote{We prove that the AU-ROC of \ours{} exceeds $0.9$, but the proof can be extended to yield arbitrarily high AU-ROC.} This advantage holds even for overparameterized models that generalize well to domain discrimination. The parameter ranges in \cref{thm:th1} indicate when the constrained approach outperforms the baseline, specifically, when the novel-class fraction $\alpha$ is small and the shift magnitude $r_{\mu}$ is large enough but not too large to violate support overlap between $\Psource$ and $\Plabel{0}$. 
The conditions impose lower bounds on the norms of the source and novel class samples to avoid cases with excessive noisy sampling. 
Furthermore, the upper bounds on these norms makes it harder to distinguish between known and novel instances. Intuitively, higher-norm features induce stronger separation margins, making them easier for a classifier to detect. Hence, at lower norms, $\rvw_{\mathrm{\text{color}}}$ outperforms the baseline $\rvw_{\mathrm{\text{DD}}}$. 
We also have a condition on $d$ to ensure overparameterized regime given a lower bound on $N$.

Next, we introduce \ours{}, a practical implementation of the rule in \cref{eq:color01}, for \osr{} under background shift. Experiments in \cref{sec:experiments} confirm superior performance over baselines in regimes consistent with theory.

%% file: sections/4_adaptation_representation.tex
\section{\ours{}: A solution to \osr{} under Background Shift} \label{sec:solutions}
The simplified constrained learning rule proposed in the previous section cannot be directly applied in practice to large-scale models or real world datasets that deviate from the ideal linear-Gaussian setup we considered. Hence, we 
first formulate the learning rule for (\ours{}) for \osr{} under background shift for $k\geq1$ with novel classes of proportion, $\alpha$, w.r.t. target dataset and a tolerance $\beta>0$ for false detection of $S_\gS$ as novelties: 
\begin{equation}\label{eq:learningrule}
\textstyle
h_{\text{color}} = \argmin_{h\in\gH}\!\left(\sum_{i\in\datasource} l_{ce}(h_c(\rvx_i), y_i) + \hat{\beta}_{\mathrm{emp}}(h_{\mathrm{\text{novel}}})\right)
\ \text{s.t.}\ 
\hat{\alpha}_{\mathrm{emp}}(h_{\mathrm{\text{novel}}}) < \hat{\alpha}.
\end{equation}
where $h = [h_c, h_{\mathrm{\text{novel}}}]$, $h_c$ are classification heads for known classes, $h_{\mathrm{\text{novel}}}$ is the novelty detection head, $l_{ce}$ is the supervised cross entropy loss, $\hat{\alpha}_{\mathrm{emp}}(h), \hat{\beta}_{\mathrm{emp}}(h)$ are empirical estimates of $\alpha(h), \beta(h)$ over $S_{\gT}$ and $S_{\gS}$ respectively. 
$\hat{\alpha}$ is some constant corresponding to the target recall. It is a hyperparameter that needs tuning. 
\subsection{Efficient Architecture for Estimating the Novel Class Ratio ($\alpha$)}
To solve the \cref{eq:learningrule}, we follow \citet{wald2023birds, chamon2022constrained, cotter2019optimization} and solve a Lagrangian optimization problem obtained by switching the role of maximization and constraints in \cref{eq:learningrule} and taking its dual, see \cref{eq:eq04lagrangian} for the full objective. However, a na{\"i}ve implementation of this objective has a significant drawback in terms of computational complexity due to hyperparameter optimization. Indeed, prior work on the rate-constrained learning problems we seek to solve (where rate in our problem corresponds to the size of the novel class) is limited to either very small models and datasets \citep{wald2023birds}, or to certain applications such as fairness \citep{zafar2019fairness}, where the desired rate is known a-priori.

The computationally challenging part is that our Lagrangian problem needs to be solved for $L$ candidate values of $\hat{\alpha}_{\mathrm{emp}}\in{\boldsymbol{\alpha}}$, corresponding to the potential estimates of novel class proportion, where $\boldsymbol{\alpha}\in{[0,1]^L}$. We perform grid search over $\boldsymbol{\alpha}$ and choose the model that obtains largest empirical recall estimate ($\hat{\alpha}_{\mathrm{emp}}(h)$) w.r.t. the task of distinguishing between $\Psource$ and $\Ptarget$, while still satisfying the constraint on empirical FPR i.e. false positive rate estimate ($\hat{\beta}$) for the same task as specified in the learning objective. For empirical purposes, we calculate the approximate bound for $\beta$ (constraint on $\hat{\beta}_{\mathrm{emp}}(h)$) using the Rademacher complexity in theorem 1 of \cite{wald2023birds}. We find that setting $\beta=0.01$ is well within the theoretically calculated bounds and works well in practice for all the experiments. Plots in figure \ref{fig:vary_fpr} provide further insights about the impact of varying $\beta$ on \osr{} performance of \ours{}.

To this end we train an architecture $h(\rvx) = w \circ \phi(\rvx)$ where $\phi:\gX \rightarrow \sR^d$ is a shared representation for classification heads, $w^c:\sR^d \rightarrow \sR^{k}$, as well as several novelty heads, $w^\alpha = \{w^\alpha_i:\sR^d \rightarrow \sR\}^{L}_{i=1}$, and each $w^\alpha_i$ corresponds to a candidate value $\hat{\alpha}$ in the search grid $\boldsymbol{\alpha}$. Hence the solution amortizes training time for all candidate values by solving the primal dual optimization problem once. which leads to better performance with large data \& models in practice as we see in section \ref{sec:experiments}.  

\begin{wrapfigure}{r}{0.4\textwidth}
    \centering
    \hspace{-10pt}\includegraphics[scale=0.36]{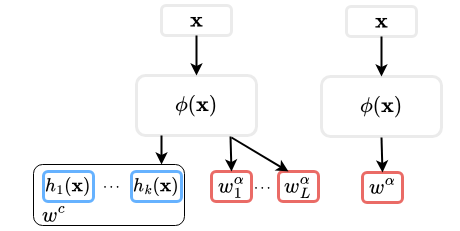}
    \caption{(left) \ours{} architecture for \osr{}, heads $w^a_i$ for multiple recall values and classification heads $w^c$, vs. (right) a network optimizing for novelty detection with single recall value as in \citet{wald2023birds}.}
    \vspace{-15pt}
    \label{fig:color_architecture}
\end{wrapfigure}
\subsection{A simple Extension of Constrained Learning for OSDA}
To account for the known classes $Y=1,\ldots, k$ and $L$ novelty heads, we construct  $w = [w^c, w^\alpha]$, such that $w:\sR^d\rightarrow \sR^{L+k}$. $w^\alpha$ is responsible for detecting novelties with various proportions while $w^c$ determines the known class of the input $\rvx$. See \cref{algorithm:color} for further details. Note that training multiple novelty heads ($[w^\alpha_1, w^\alpha_2, ..., w^\alpha_L$]) simultaneously is more computationally efficient than training a separate model for each candidate $\hat{\alpha}$ iteratively. 
We hypothesize that having a shared representation for both tasks enables the model to learn a simpler representation that fits both tasks, yielding favorable generalization bounds as suggested by theory on multitask learning \citep{baxter_model_2000, maurer_benet_16}.

Let us gather the components for our overall method, \ours{}, that we evaluate in the next section. \Cref{fig:color_architecture} illustrates the architecture of \ours{}.
It is trained by combining the constrained learning rule with $\ell_{\log}$, the binary cross-entropy loss, aiming to detect novel examples in $\datatarget$,\footnote{label $0$ in the loss corresponds here to a sample belonging to $\datasource$.} $\lambda_{\hat{\alpha}}$ is a Lagrange multiplier corresponding to the head $h_{\hat{\alpha}}=\phi\circ w^\alpha (\rvx)$, $\ell_{\sigma}(\cdot)$ is a sigmoid function which serves as an approximation to the indicator function, and the supervised cross entropy loss is $l_{ce}$ for $h_c(\rvx_i) = w^c \circ \phi (x)$.
\begin{align}\label{eq:eq04lagrangian}
    \mathcal{L}(h) = &N_\gS^{-1} \sum_{i \in \datasource} \biggl( l_{ce}(h_c(\rvx_i), y_i)  + \sum_{\hat{\alpha}\in{\boldsymbol{\alpha}}}l_{\text{log}}( h_{\hat{\alpha}}(\rvx_i), 0)\biggr) +
      N_\gT^{-1}  \sum_{\substack{i \in S_{\gT}, \\ \hat{\alpha}\in{\boldsymbol{\alpha}}}}\lambda_{\hat{\alpha}} \cdot \biggl(l_{\sigma}( h_{\hat{\alpha}}(\rvx_i)) - \hat{\alpha}\biggr)
      \vspace{-3pt}
\end{align}
\vspace{-10pt}

In our experiments, the representation $\phi(\rvx)$ is either trained from scratch or built on a pretrained backbone with added fully connected layers. For small datasets (e.g., CIFAR100), we train $\phi(x)$ as ResNet18 \citep{he_deep_2016} from scratch. For larger ones (e.g., SUN397), we use pretrained $\phi(x)$ such as ResNet50 (ImageNet1K\_V1) \citep{Russakovsky2015imagenet} or a CLIP-pretrained \vit{} encoder \citep{radford2021clip}. For text datasets (e.g., Amazon Reviews), we use RoBERTa \citep{liu_roberta_2019} for embeddings. See \cref{sec:app_color} for details.

\begin{algorithm}
\small
\caption{\ours: Constrained Learning for Open-set Recognition}\label{alg:cap}
\begin{algorithmic}[1]
\Require Labelled $\datasource$ and unlabelled $S_\mathcal{T}$ datasets, hypothesis class $\mathcal{H}$, target FPR $\beta > 0$ and $L$ potential novel class sizes $\boldsymbol{\alpha} \in [0,1]^L$.\\
Split $\datasource, S_\mathcal{T}$ into train $T_\mathcal{S}, T_\mathcal{T}$ and validation sets $V_\mathcal{S}, V_\mathcal{T}$ respectively.\\
Train either the entire model ($\phi$, $w^c$ \& $w^\alpha$) or use pretrained $\phi$ and just train the last fully connected layer of $\phi$ along with $w^c$ \& $w^\alpha_{\hat{\alpha}}$ to minimize the eq. \ref{eq:eq04lagrangian}.
\State Let $\hat{\beta}_{\mathrm{emp}}(w^\alpha_{\hat{\alpha}}) = \frac{1}{|V_\mathcal{S}|}\sum_{\rvx \in V_\mathcal{S}} w^\alpha_{\hat{\alpha}}\circ\phi(\rvx)$, and $\hat{\alpha}(w^\alpha_{\hat{\alpha}}) = \frac{1}{|V_\mathcal{T}|}\sum_{\rvx \in V_\mathcal{T}} w^\alpha_{\hat{\alpha}}\circ\phi(\rvx)$ $\forall \alpha\in{\boldsymbol{\alpha}}$
\State \Return \big[$w^c$, argmax$_{w^\alpha_{\hat{\alpha}}:\alpha \in \boldsymbol{\alpha}, \hat{\beta}_{\mathrm{emp}}(w^\alpha_{\hat{\alpha}})<\beta} \hat{\alpha}(w^\alpha_{\hat{\alpha}})$\big]
\end{algorithmic}
\label{algorithm:color}
\end{algorithm}

%% file: sections/5_experiments.tex
\section{Experiments} \label{sec:experiments}
We are now in place to empirically evaluate \ours{} against variety of baselines across background shifts and novel classes that we create in real data. Here, the main questions we wish to answer are:
\begin{enumerate}[leftmargin=0.4cm]
    \item Does background shift within non-novel instances affect OSDA performance? 
    \item Can shared representations mitigate the impact of background shift by enhancing robustness in classifying samples from $S_{\gT,[k]}$ thereby improving overall OSDA performance?
    \item What is the effect of novel class ratio $\alpha$ particularly w.r.t. detecting novel classes? 
\end{enumerate}


\subsection{Experimental setting}
The experiments to examine these questions are devised over image and text datasets as follows. We randomly draw a class from the set of classes $\gY$ and assign that as the $k+1$-th novel class. Denoting the instances of known classes by $S_k= \{\rvx : y \in [k]\}\, \forall \, (\rvx,y) \in S$ and novel ones as $S_{k+1}= \{\rvx : y=k+1\} \, \forall \, (\rvx,y) \in S$, we create a background shift by further splitting $S_k$ into $S_{\gS}$ and $S_{\gT,[k]}$. We use semantic attributes that are annotated in the metadata of each dataset to create this shift between $S_{\gS}$ and $S_{\gT,[k]}$. The attributes used for each dataset are specified in section \ref{sec:datasets_exp} and elaborated further in section \ref{sec:app_dataset}. At training time the learner is provided with a labelled dataset $S_{\gS}$ and the unlabelled dataset $S_{\gT}=S_{\gT,[k]}\cup S_{k+1}$. The mixture proportion $\alpha = |S_{k+1}|/(|S_{k+1}| + |S_{\mathcal{T}, [k]}|)$ is set by adjusting the sizes of the selected novel classes in $S_{\gT,[k]}$.

\subsection{Datasets} \label{sec:datasets_exp}

Most of the large scale image classification models are trained on ImageNet dataset. Hence we use a similar large scale dataset, \textbf{SUN397} \citep{jianxiong2010sun397} having completely different categories than ImageNet. 
We exploit three-level hierarchy structure to create distribution shifts by varying proportions of latent subtypes ($y$) of known class labels $\mathcal{Y}$. This causes a background shift where $ \mathrm{Supp}(\Plabel{[k]}) \subseteq \mathrm{Supp}(\Psource)$. Indoor scenes (e.g., shopping/dining places, workplaces) serve as known (in-distribution) classes, while outdoor scenes are randomly selected as novel classes. 
We also include a \textbf{Amazon Product Reviews} text dataset \citep{ni_justifying_2019} to demonstrate the versatility of the method across diverse modalities.
Classes are different product categories (prime pantry, musical instruments, etc.),
and induce background/covariate shift in known classes based on positive vs. negative sentiment in the review. The novel class is an unknown product category.
Finally, we include small scale \textbf{CIFAR100} \citep{Krizhevsky09learningmultiplecifar} dataset to evaluate the adaptive methods (like \DDnew{}, \nnPUnew{}, \uPUnew{}, \BODA{}, \osr{}) particularly in scenarios when their feature extractors are trained from scratch. 
Similar to the SUN397 dataset, we leverage the inherent hierarchies of CIFAR100 classes to create a natural background shift by varying latent subtype proportions and include novel classes in the target dataset. 
\Cref{table:shift_settings} provides a summary of characteristics and further details on each dataset are in \cref{sec:app_dataset}. 

\begin{table*}[h!]
\centering
\caption{Overview of experiment settings. DS = distribution shift, prop. = proportions}
\label{table:shift_settings}
\resizebox{10cm}{!}{%

\begin{tabular}{|c || c | c | c | c }
\hline
Experiment setup & SUN397 & CIFAR100 & Amazon Reviews \\
\hline\hline
 DS factor &  varying subtypes prop.  & varying subtypes prop. &  sentiment \\
\hline
 no. of novel classes & 12 & 5 & 6 \\
\hline
 Novel class ratio ($\alpha$) & $0.07\pm0.03$ & $0.16\pm0.09$ & $0.07\pm0.02$ \\
 \hline
\end{tabular}
}
\end{table*}

\textbf{Evaluation metrics:}
We primarily use Area Under ROC Curve (\auroc{}) and Area Under Precision-Recall Curve (\auprc{}) to evaluate the novel category detection performance, while we use Open-Set Classification Rate (\oscr{}) to summarize overall \osr{} performance for all the methods \footnote{Section \ref{sec:auroc_vs_auprc} in appendix provides detailed argument for the preference of AUPRC over AUROC mainly when minority class proportions are very low.}.
\oscr{} \citep{dhamija2018oscr} measures the trade-off between correct classification rate of the known classes and false positive rate of the novel samples.
\subsection{Baseline methods}

We include adaptive methods from novelty detection that access both labelled source and unlabelled target data. These baselines include domain discriminator \DDnew{}, \citep{elkan2008learning, duplessis2014analysis}, \uPUnew{} \citep{duplessis2014analysis} and \nnPUnew{} \citep{kiryo2017positive}. These methods are modified for \osr{} through joint learning approach enabled by a simple architectural modification like figure \ref{fig:color_architecture} and aggregating the loss components from both the tasks. We also include another popular \osr{} baselines \BODA{} \citep{saito2018open} and PULSE \citep{garg22adaptation} that are agnostic to input data modality. Another adaptive baseline is \cite{wuyang2023anna}, however, this is specifically designed for vision data. Hence, we report its performance on SUN397 dataset in \cref{table:sun397_vitl14_w_vs_wo_shift} along with other vision-specific baselines.

We include an entropy-based method, ARPL \citep{Chen_2021} with Maximum Logit Score as proposed in \citep{vaze2022openset}. Such methods are non-adaptive as they do not access target data, yet are a popular choice for novelty detection.\footnote{We are using ARPL on SUN397. The paper that proposed ARPL also introduced an improvement with Confused Sampling (ARPL+CS) that uses GANs, but that proved challenging to apply in the SUN397 dataset} We also tested simple and popular baselines such as MSP \cite{hendrycks2016baseline} and results can be found in appendix tables \ref{table:MSP_sun397_low_ratio}, \ref{table:MSP_sun397_high_ratio}. We further include SHOT \citep{liang2021reallyneedaccesssource}, CAC \citep{Miller2021ClassAC} and
a zero-shot OOD detection method (ZOC) proposed in \citet{Esmaeilpour_Liu_Robertson_Shu_2022}. 
Further details about training and hyperparameters are in Appendix \ref{sec:training}.

\vspace{-5pt}

%% file: sections/5a_Results.tex
\begin{table*}[h!]
\centering
\caption{Performance comparison of adaptive methods for \osr{} under background shift 
demonstrating versatility across data modalities. Detailed results in appendix Tables~\ref{table:sun397_rn50_w_shift}, \ref{table:cifar100_w_shift}, and \ref{table:amazon_reviews_w_shift}.\protect\footnotemark}
\label{table:results_sumary}
\resizebox{\textwidth}{!}{%

\begin{tabular}{c || c || c | c || c | c || c | c }
\hline
\multirow{3}{*}{Metric} & \multirow{3}{*}{Methods} & \multicolumn{2}{c||}{SUN397 ($\alpha=0.07\pm0.03$)} & \multicolumn{2}{c||}{CIFAR100 ($\alpha=0.07\pm0.02$)} & \multicolumn{2}{c}{Amazon Reviews ($\alpha=0.16\pm0.09$)} \\
\cline{3-8}
 &  & \multicolumn{2}{c||}{ResNet50 pretrained on ImageNet1K} & \multicolumn{2}{c||}{ResNet18 (randomly initialized)} & \multicolumn{2}{c}{pretrained RoBERTa} \\
\cline{3-8}
 & & Summary & Wins & Summary & Wins & Summary & Wins \\
\hline\hline
\multirow{5}{*}{\auroc{}} & \DDnew{} & $0.91\pm0.05$ & $1/15$ &                                                  $0.70\pm0.11$ & $3/25$ & $0.72\pm0.09$ & $1/30$ \\
                       & \uPUnew{} & $0.76\pm0.14$ & $0/15$ & $0.67\pm0.10$ & $1/25$ & $0.76\pm0.08$ & $4/30$ \\
                       & \nnPUnew{} & $0.76\pm0.14$ & $0/15$ & $0.67\pm0.12$ & $1/25$ & $0.76\pm0.08$ & $4/30^*$ \\
                       & \BODA{} & $0.86\pm0.06$ & $0/15$ & $0.57\pm0.06$ & $0/25$ & $0.66\pm0.10$ & $2/30$ \\
                       & \arpl{} & $0.71\pm0.08$ & $0/15$ & $0.76\pm0.05$ & $8/25$ & $0.70\pm0.05$ & $1/30$ \\
                       & \cac{} & $0.80\pm0.05$ & $0/15$ & $0.69\pm0.04$ & $2/25$ & $0.70\pm0.07$ & $5/30$ \\
                       & \pulse{} & $0.73\pm0.07$ & $0/15$ & $0.72\pm0.04$ & $0/25$ & $0.63\pm0.08$ & $1/30$ \\
                       & \ours{} & $\boldsymbol{0.98\pm0.02}$ & $\boldsymbol{14/15}$ & $\boldsymbol{0.77\pm0.09}$ & $\boldsymbol{10/25}$ & $\boldsymbol{0.79\pm0.09}$ & $\boldsymbol{16/30}$ \\
\hline
\multirow{5}{*}{\auprc{}} & \DDnew{} & $0.54\pm0.22$ & $1/15$ &                                                  $0.24\pm0.13$ & $2/25$ & $0.43\pm0.18$ & $0/30$ \\
                       & \uPUnew{} & $0.21\pm0.22$ & $0/15$ & $0.18\pm0.12$ & $1/25$ & $0.51\pm0.17$ & $7/30$ \\
                       & \nnPUnew{} & $0.21\pm0.22$ & $0/15$ & $0.20\pm0.16$ & $2/25$ & $0.51\pm0.17$ & $7/30^*$ \\
                       & \BODA{} & $0.45\pm0.11$ & $0/15$ & $0.09\pm0.03$ & $0/25$ & $0.28\pm0.21$ & $1/30$ \\
                       & \arpl{} & $0.12\pm0.07$ & $0/15$ & $0.18\pm0.07$ & $4/25$ & $0.27\pm0.13$ & $0/30$ \\
                       & \cac{} & $0.18\pm0.07$ & $0/15$ & $0.15\pm0.05$ & $1/25$ & $0.3\pm0.16$ & $1/30$ \\
                       & \pulse{} & $0.15\pm0.05$ & $0/15$ & $0.17\pm0.05$ & $0/25$ & $0.21\pm0.13$ & $0/30$ \\
                       & \ours{} & $\boldsymbol{0.91\pm0.09}$ & $\boldsymbol{14/15}$ & $\boldsymbol{0.33\pm0.14}$ & $\boldsymbol{15/25}$ & $\boldsymbol{0.54\pm0.18}$ & $\boldsymbol{21/30}$ \\
\hline
\multirow{5}{*}{\oscr{}} & \DDnew{} & $0.68\pm0.05$ & $1/15$ &                                                   $0.51\pm0.10$ & $1/25$ & $0.50\pm0.06$ & $0/30$ \\
                       & \uPUnew{} & $0.40\pm0.11$ & $0/15$ & $0.49\pm0.09$ & $0/25$ & $0.54\pm0.05$ & $5/30$ \\
                       & \nnPUnew{} & $0.40\pm0.11$ & $0/15$ & $0.48\pm0.10$ & $1/25$ & $0.54\pm0.05$ & $5/30^*$ \\
                       & \BODA{} & $0.55\pm0.10$ & $0/15$ & $0.60\pm0.04$ & $3/25$ & $\boldsymbol{0.56\pm0.06}$ & $7/30$ \\
                       & \arpl{} & $0.60\pm0.06$ & $0/15$ & $0.61\pm0.04$ & $\boldsymbol{7/25}$ & $0.56\pm0.05$ & $4/30$ \\
                       & \cac{} & $0.68\pm0.05$ & $0/15$ & $0.55\pm0.04$ & $0/25$ & $0.49\pm0.07$ & $0/30$ \\
                       & \pulse{} & $0.65\pm0.06$ & $0/15$ & $\boldsymbol{0.62\pm0.03}$ & $6/25$ & $0.53\pm0.07$ & $4/30$ \\
                       & \ours{} & $\boldsymbol{0.81\pm0.04}$ & $\boldsymbol{14/15}$ & $0.59\pm0.08$ & $\boldsymbol{7/25}$ & $\boldsymbol{0.56\pm0.06}$ & $\boldsymbol{10/30}$ \\  
\end{tabular}
}
\textit{$^*$ When \nnPU{} behaves exactly like \uPU{}, both are awarded the win.}
\end{table*}

\subsection{Results}      

\begin{figure}
    \centering
    \begin{subfigure}[!]{0.66\linewidth}
        \centering
        \includegraphics[height=8cm,scale=0.2]{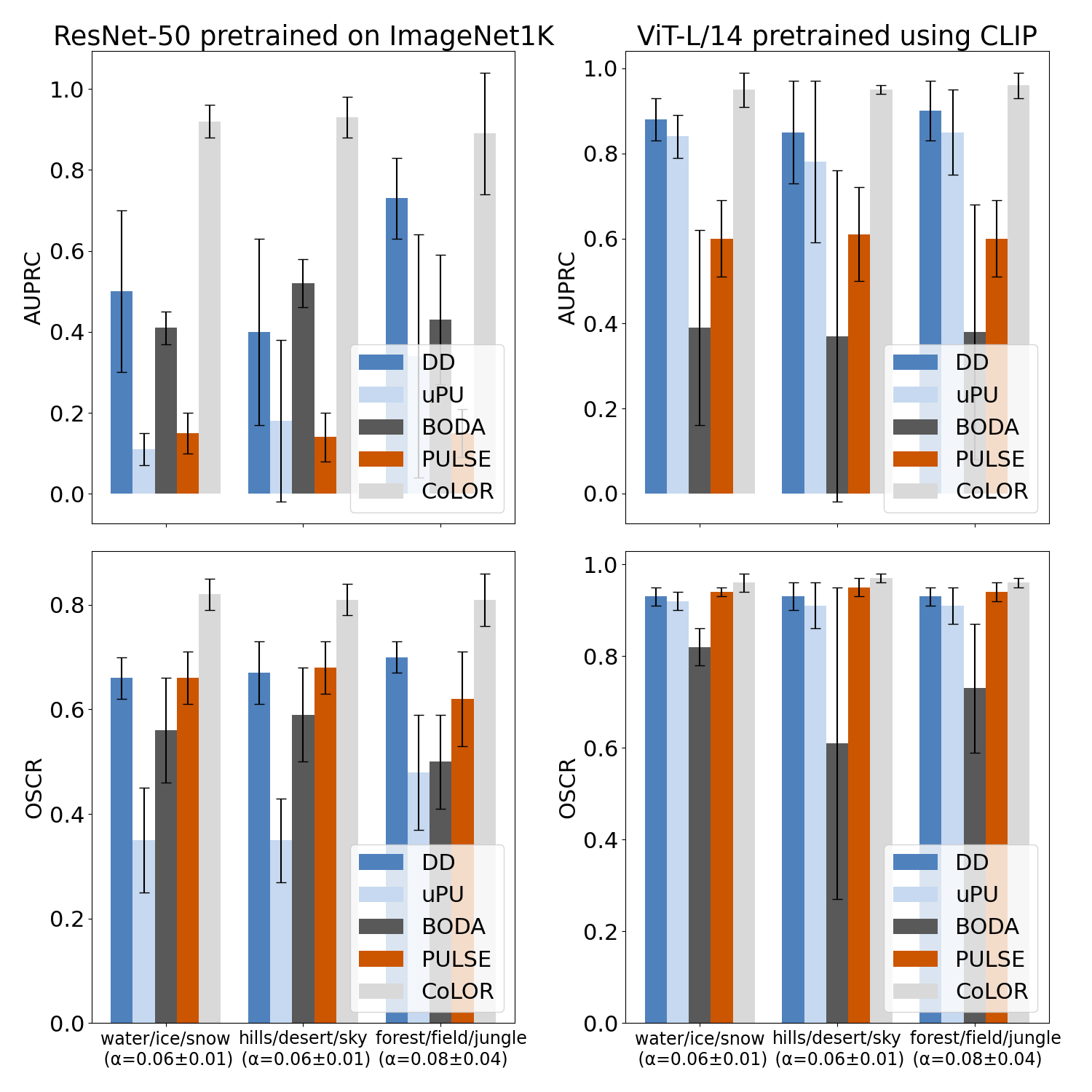}
        \caption{}
        \label{fig:sun397_bar_plot}
    \end{subfigure}
    \hfill 
    \begin{subfigure}[!]{0.33\linewidth}
        \centering
        \includegraphics[height=8cm,scale=0.2]{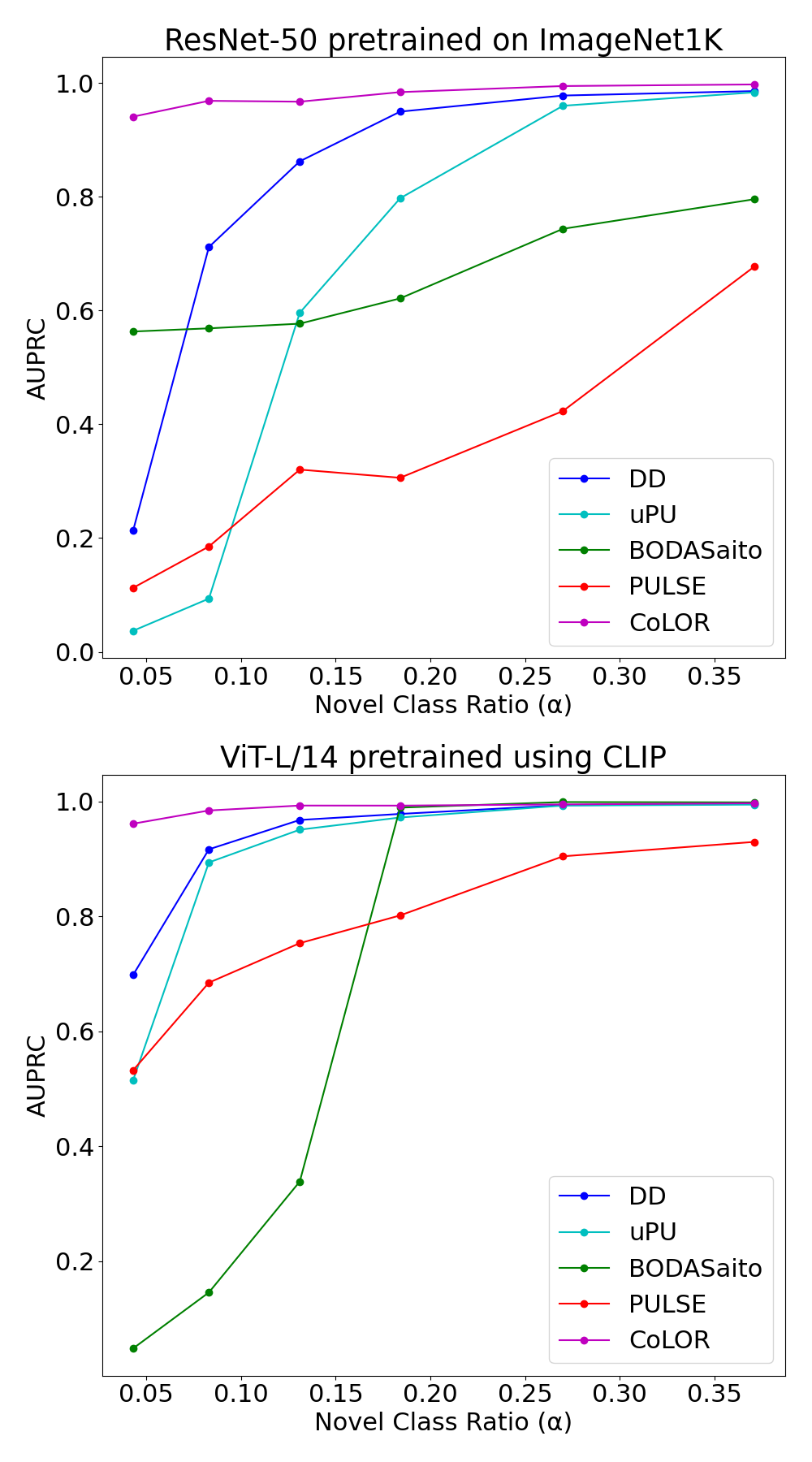}
        \caption{}
        \label{fig:sun397_curve_plot}
    \end{subfigure}
    \caption{(a) OSDA performance of top performing adaptive methods on SUN397 dataset with background shift using pretrained ResNet50 \& CLIP \vit{} backbone architectures. (b) Impact of novel class ratio ($\alpha$) on adaptive methods on SUN397 dataset under background shift. }
\end{figure}

Table \ref{table:results_sumary} summarizes the OSDA performance across all the datasets using all the metrics discussed previously. We observe that \ours{} generally outperforms all the baselines with a significant margin. The \oscr{} performance of \ours{} is comparable to \BODA{}, \arpl{} and \pulse{} for CIFAR100 and Amazon Reviews dataset however, they notably under perform in novelty detection (AUROC and AUPRC) for the same datasets. Additionally, we observe that these methods do not perform as well on the larger dataset of SUN397. We find that the \nnPUnew{} objective defaults to \uPUnew{} when using pretrained feature extractors (SUN397 and Amazon Reviews), as the empirical risk remains non-negative for both. This causes similar results of \uPUnew{} and \nnPUnew{} on SUN397 and CIFAR100 datasets. In the image datasets, we find that adaptive methods significantly outperform non-adaptive methods overall. We further strengthen our results by extending the study to other richer feature representations like pretrained CLIP \vit{}, particularly for SUN397 dataset. This is shown in \ref{fig:sun397_bar_plot} across different novel classes (X-axis) of SUN397 dataset. \ours{} still outperforms the baselines using either of the backbone architectures. We can see that CLIP \vit{} have overall better performance than ResNet50 (pretrained on ImageNet) in SUN397 dataset due to richer and more robust features. This mitigates the impact of distribution shift, which is more pronounced in SUN397 compared to other datasets. Further results are in appendix tables \ref{table:sun397_rn50_w_shift}, \ref{table:sun397_vitl14_w_shift}, \ref{table:cifar100_w_shift}, \ref{table:amazon_reviews_w_shift}.


\footnotetext{The summary statistics are derived by averaging (or aggregating wins) across all novel class identities and randomly generated data splits (between $\Psource$ \& $\Ptarget$) along with the corresponding standard deviations. Refer \ref{sec:training}}

%% file: sections/6_Discussion.tex
\vspace{-5pt}
\subsection{Discussion} \label{sec:discussion}

\begin{table*}
\caption{OSDA performance with and without background shift (DS) using ResNet50 backbone for SUN397. These results include additional vision-specific baselines like \shot{}, \zoc{}, \anna{} }
\centering
\resizebox{0.7\textwidth}{!}{
\begin{tabular}{c || c | c || c | c || c | c}
\hline
\multirow{2}{*}{Method} & \multicolumn{2}{c||}{\textbf{AUROC}} & \multicolumn{2}{c||}{\textbf{AUPRC}} & \multicolumn{2}{c}{\textbf{OSCR}} \\
\cline{2-7}
                        & \textbf{w/ BS} & \textbf{w/o BS} & \textbf{w/ BS} & \textbf{w/o BS} & \textbf{w/ BS} & \textbf{w/o BS} \\
\hline\hline
\DDnew{}   & $0.91\pm0.05$ & $\boldsymbol{1.00\pm0.00}$ & $0.54\pm0.22$ & $\boldsymbol{1.00\pm0.00}$ & $0.68\pm0.05$ & $\boldsymbol{0.99\pm0.00}$ \\
\uPUnew{}  & $0.76\pm0.14$ & $\boldsymbol{1.00\pm0.00}$ & $0.21\pm0.22$ & $\boldsymbol{1.00\pm0.00}$ & $0.40\pm0.11$ & $\boldsymbol{0.99\pm0.00}$ \\
\BODA{}   & $0.86\pm0.06$ & $0.88\pm0.04$ & $0.45\pm0.11$ & $0.24\pm0.15$ & $0.55\pm0.10$ & $0.93\pm0.02$ \\
\arpl{}   & $0.71\pm0.08$ & $0.84\pm0.04$ & $0.12\pm0.07$ & $0.20\pm0.07$ & $0.60\pm0.06$ & $0.80\pm0.04$ \\
\cac{}  & $0.80\pm0.05$ & $0.88\pm0.03$ & $0.18\pm0.07$ & $0.29\pm0.07$ & $0.68\pm0.05$ & $0.84\pm0.03$ \\
\pulse{}  & $0.73\pm0.07$ & $0.82\pm0.05$ & $0.15\pm0.05$ & $0.20\pm0.05$ & $0.65\pm0.06$ & $0.78\pm0.05$ \\
\ours{}   & $\boldsymbol{0.98\pm0.02}$ & $\boldsymbol{1.00\pm0.00}$ & $\boldsymbol{0.91\pm0.09}$ & $0.99\pm0.03$ & $\boldsymbol{0.81\pm0.04}$ & $0.93\pm0.01$ \\
\hline
\shot{}   & $0.71\pm0.16$ & $0.84\pm0.04$ & $0.19\pm0.13$ & $0.20\pm0.07$ & $0.22\pm0.07$ & $0.80\pm0.04$ \\
\zoc{}*  & $0.82\pm0.07$ & $0.82\pm0.08$ & $0.23\pm0.06$ & $0.22\pm0.07$ & $0.48\pm0.12$ & $0.49\pm0.05$ \\
\anna{}   & $0.93\pm0.05$ & $0.95\pm0.07$ & $0.73\pm0.16$ & $0.90\pm0.08$ & $0.60\pm0.07$ & $0.82\pm0.07$ \\
\hline
\hline
\end{tabular}
}

\textit{*\zoc{} uses CLIP ViT-B/32 image encoder with BERT text encoder.}
\label{table:sun397_vitl14_w_vs_wo_shift}
\end{table*}

\begin{table*}[h!]
\centering
\caption{Top-1 accuracy over $S_{\gT}$ for closed-set classification on SUN397 data.}
\centering
\resizebox{0.4\textwidth}{!}{%
\begin{tabular}{c || c | c }
\hline
Method & \resnet{} & \vit{} from CLIP \\
\hline\hline
\source{} & $0.72\pm0.03$ & $0.97\pm0.01$ \\
\DDnew{} & $0.75\pm0.04$ & $0.97\pm0.01$ \\
\BODA{} & $0.70\pm0.07$ & $0.75\pm0.13$ \\
\cac{} & $\boldsymbol{0.85\pm0.02}$ & $\boldsymbol{0.98\pm0.01}$ \\
\pulse{} & $0.82\pm0.03$ & $0.97\pm0.01$ \\
\ours{} & $0.83\pm0.03$ & $\boldsymbol{0.98\pm0.01}$ \\
\hline
\end{tabular}
}
\label{table:sun397_target_acc}
\hspace{0.5cm} 

\end{table*}

Based on our observations from the experiments, we address the three critical questions below: \\
\textbf{Background shift within non-novel instances causes a significant drop in the \osr{} performance of all the methods that involve fine-tuning using the source data.}\\
From tables \ref{table:sun397_vitl14_w_vs_wo_shift}, \ref{table:sun397_vitl14_w_shift} and \ref{table:sun397_vitl14_wo_shift} we observe that all the methods that involve training a closed-set classifier suffer from background shift. Hence, background shift not only harms closed-set performance but also that of open-set recognition. If we had a better closed set classifier, in the sense that it was more robust to the background shift, then we probably would’ve ended up with better \osr{} performance too (see \cref{sec:app_color_with_da}). \ours{} seems to primarily help when standard techniques like long training, label smoothing, and other techniques used in \citet{vaze2022openset} are insufficient to train a model that’s robust to shift. \zoc{} performs zero-shot open-set classification using pretrained CLIP and text decoder models without accessing source and target data. Hence, it is not impacted by background shift as expected. We have provided further results on CIFAR100 dataset in \ref{table:cifar100_w_shift} \& \ref{table:cifar100_wo_shift} although CIFAR100 is a much simpler dataset and hence we see near perfect scores by \zoc{} due to the use of large CLIP ViT-B/32 model pretrained on huge image caption datasets. We see in tables \ref{table:sun397_vitl14_w_shift} \& \ref{table:sun397_vitl14_wo_shift} that non-adaptive methods do not perform well when scaled to larger SUN397 dataset.\\
\textbf{Shared representations obtained by joint learning of closed set classification and novel category detection can help mitigate the impact of background shift within non-novel instances.}\\
Table \ref{table:sun397_target_acc} compares the Top-1 accuracy for known classes under background shift between $S_{\gS}$ and $S_{\gT}$. The Source-only method is trained on labelled $S_{\gS}$ for classifying known classes in $S_{\gT}$, while \DDnew{} combines closed-set classification with domain discrimination (distinguishing samples from $S_{\gS}$ and $S_{\gT}$). We can see that \DDnew{} has better Top-1 accuracy of known classes than the Source-only method on $S_{\gT}$ notably for ResNet50 model. \ours{} further improves the performance by employing constrained learning instead of domain discrimination in the joint learning objective to detect novel classes. Performance differences remain incremental for CLIP ViT-L/14 because its pretrained representations are already rich and robust.\\
\textbf{As the novel class ratio $\alpha$ decreases, the performance of existing methods to detect unknown/novel classes significantly decreases.}\\
Figure \ref{fig:sun397_curve_plot} illustrates the effect of novel class ratio on the novel class detection performance of existing \osr{} methods and \ours{}. A low $\alpha$ significantly deteriorates the open-set classification performance, especially for methods not designed to handle such cases. Existing benchmarks focus on large novel class sizes. We believe that experimenting with smaller sizes would be a valuable step toward creating more realistic benchmarks.

%% file: sections/7_Conclusion.tex
\section{Future Work} \label{sec:conclusion}

Future research could focus on refining the search criteria for the optimal model head by selecting an appropriate $\beta$ threshold.
Additionally, a robust theoretical underpinning is needed to explain why shared representations improve the classification of known classes under distribution shifts. Integrating test-time or model-free domain adaptation methods, such as those by \citet{wang_tent_2020}, \citet{zhang_model_free_2023}, or \citet{saito2018open}, could further enhance \osr{} performance. It would be valuable to compare our approach with large foundation models. Although curating such a benchmark would require access to publicly available pretraining datasets like LAION, which would require significant effort and present unforeseen challenges.


%% file: sections/acknowledgments.tex
\section*{Acknowledgments}
We acknowledge the guidance of Dr. Rama Chellappa in this project.
Furthermore, we acknowledge the support from the DARPA TIAMAT program under Grant No. HR00112490422. 
The views and conclusions contained in this document are those of the authors and should not be interpreted as representing the official policies, either expressed or implied, of the U.S. Government.

This work was also supported by the Gordon and Betty Moore Foundation through the grant titled 
“Safety Monitoring of Deployed Clinical AI via the Framework for Real-time Auditing of Individual Predictions (FRAP)” (GBMF ID \#12128).

We additionally acknowledge support from the NSF award titled 
“FW-HTF: Human-Machine Teaming for Medical Decision Making” (FAIN 1840088), which concluded in September 2024.

%% file: sections/notations.tex
\subsection{Notations and Keywords}\label{sec:notations}

\begin{table}[H]
\centering
\caption{Notations for CoLOR learning rules and theoretical analysis.}
\label{tab:notation_method}
\small
\begin{tabular}{ll}
\toprule
\textbf{Notation} & \textbf{Meaning} \\
\midrule

$k, [k]$ & Number of known (source) classes, known class labels $[1,...k]$ \\

$k+1$ & Novel / unknown class label \\

$x \in \gX$ & Input instance \\

$y \in \gY$ & Class label \\

$\datasource$ & Labeled source dataset $\{(x_i,y_i)\}_{i=1}^{N_S}$ \\

$\datatarget$ & Unlabeled target dataset $\{x_i\}_{i=1}^{N_T}$ \\

$N_S$ & Number of source samples \\

$N_T$ & Number of target samples \\

$N_{T,0}, N_{T,1}$ & Target known vs.\ novel samples (binary case) \\

$\gX, \gY$ & Input space and label space, $\gY=[k+1]$ \\

$P_{\gS} = P_{\gS}(x,y)$ & Source joint distribution over $(x,y)$ \\

$P_{\gT} = P_{\gT}(x,y)$ & Target joint distribution over $(x,y)$ \\

$P_{T,[k]}(x,y)$ & Target distribution restricted to known classes \\

$P_{T,k+1}(x)$ & Target distribution of the novel class \\

$\alpha := P_T(y=k+1)$ & Novel class proportion in target \\

\midrule

$\gH$ & Hypothesis class of open-set clasification models \\

$h:\gX \to \mathbb{R}^{k+1}$ & Model scoring all $k+1$ classes \\

$h_{\text{novel}}$ & Novelty detection score derived from $h$ \\

$h_c$ & Classification head for known classes \\

$h_{\DD}$ & ERM classifier distinguishing $P_S$ vs.\ $P_T$ \\

$h_{\ours}$ & Constrained learning solution (\cref{eq:learningrule}) \\

$w_{\DD}$ & Max-margin domain discriminator solution (\cref{eq:DD_color_objective}) \\

$w_{\ours}$ & Max-margin constrained separator (\cref{eq:DD_color_objective}) \\

\midrule

$\ell_{01}$ & Binary $0$-$1$ loss \\

$\ell_{\log}$ & Logistic / binary cross-entropy loss \\

$\ell_{\text{ce}}$ & Supervised cross-entropy loss \\

\midrule

$t \in \mathbb{R}$ & Decision threshold for novelty \\

$\beta(h) = \beta(h;t) = \mathbb{E}_{\rvx\sim\datasource}\left[\mathbf{1}\left[ h_{\mathrm{novel}}(\rvx; t) > t \right]\right]$ & False Positive Rate (FPR) on source \\

$\alpha(h) = \alpha(h;t) = \mathbb{E}_{\rvx\sim\datatarget}\left[\mathbf{1}\left[ h_{\mathrm{novel}}(\rvx; t) > t \right]\right]$ & Recall of novelty detection on target \\

$\hat{\alpha}$ & Target recall hyperparameter \\

$\hat{\beta}$ & Empirical FPR estimate \\

$\hat{\alpha}_{emp}(h)$ & Empirical estimate of target recall $\alpha(h)$ over $\datatarget$ \\

$\hat{\beta}_{emp}(h)$ & Empirical estimate of source FPR $\beta(h)$ over $\datasource$ \\

$\lambda_{\hat{\alpha}}$ & Lagrange multiplier for constraint \\

$\boldsymbol{\alpha} \in [0,1]^L$ & Grid of candidate novel proportions \\

$L$ & Number of novelty heads in CoLOR \\

\midrule

$\phi(x)$ & Backbone embedding model shared by $h_c$ and $h_{novel}$ heads \\

$w^c$ & Linear head for known-class classification \\

$w^\alpha_i$ & Novelty head for candidate $\hat{\alpha}_i$ \\

\midrule

$d$ & Feature dimension \\

$\mub=\mu$ & Background-shift direction \\

$\etab=\eta$ & Novel-class direction \\

$r_\mu=\|\mu\|$ & Magnitude of background shift \\

$r_\eta=\|\eta\|$ & Novel-class signal strength \\

$\sigma=1/\sqrt{d}$ & Noise scale \\

$\mathcal{N}(\cdot,\cdot)$ & Gaussian distribution \\

$\delta \in (0,1)$ & Failure probability in \cref{thm:th1} \\

\bottomrule
\end{tabular}
\end{table}

\begin{table}[H]
\centering
\caption{Keywords for OSDA problem setup, distributions, and evaluation metrics.}
\label{tab:notation_problem}
\small
\begin{tabular}{ll}
\toprule
\textbf{Keyword} & \textbf{Meaning} \\
\midrule




















$\mathrm{Supp}(\cdot)$ & Support of a distribution \\

Background shift & Shift where $\mathrm{Supp}(P_{T,[k]}) \subseteq \mathrm{Supp}(P_S)$ \\

Assumption 1 & Novel class separable from known classes \\

Linear-Gaussian theory (\cref{def:overparam_setting}) & Simplified overparameterized PU model \\







\midrule

$\auroc$ & Area Under ROC Curve \\

$\auprc$ & Area Under Precision-Recall Curve \\

$\oscr$ & Open-Set Classification Rate \\

Theorem 1 & Conditions in linear gaussian problem setup when \\
& $\auroc(w_{\DD})<0.5$ but $\auroc(w_{\ours})>0.9$ \\
\bottomrule
\end{tabular}
\end{table}

%% file: sections-rewrite/subappendix.tex
\subsection{Theoretical analysis}\label{sec:propositions}

Restating the statistical learning rule.
\begin{equation*}\label{eq:DD_color_objective}
\begin{aligned}
\rvw_{\text{\text{DD}}} & = \arg\min_{\rvw}\,\|\rvw\|& \\
\text{s.t.}\; & \rvw^\top\rvx \le -1  &\forall \rvx\in \datasource \\
              & \rvw^\top\rvx \ge 1   &\forall \rvx\in \datatarget
\end{aligned}
\qquad
\begin{aligned}
\rvw_{\mathrm{\text{color}}} & = \arg\min_{\rvw}\,\|\rvw\|& \\
\text{s.t.}\;       & \rvw^\top\rvx \le 0  &\forall \rvx\in \datasource \\
                    & \rvw^\top\rvx \ge 1 &\forall \rvx\in \datatarget
\end{aligned}
\end{equation*}

\begin{proof}[Proof sketch]
\label{sec:proof_sketch}
The proof consists of four parts. First, we show that $\mathrm{AU-ROC}(\rvw) = Q(\frac{\rvw^\top(\mub-\etab)}{\sigma\sqrt{\|\rvw\|^2 - (\rvw^\top(\mub+\etab))^2}})$, where $Q(\cdot)$ is the inverse Gaussian tail function. This means that whenever $\rvw_{\mathrm{\text{DD}}}^\top(\mub-\etab) \leq 0$, the domain discriminator algorithm will have $\mathrm{AU-ROC}(\rvw_{\mathrm{\text{DD}}}) \leq 0.5$. To derive ranges of problem parameters where this happens, we analyze the convex dual of the maximum-margin problem with an added constraint that $\rvw_{\mathrm{\text{DD}}}^\top(\mub-\etab) > 0$
and obtain a lower bound $\underline{ \gamma}$ on the value of the primal problem (i.e. a lower bound on the norm of a ``good'' solution). Then we guess a solution of the form $\rvw' = \mub + \sum_{i\in{U}}{\xi_i}$, which does not depend on $\etab$ at all and has AU-ROC smaller than $0.5$. Intuitively, $U$ is constructed by including the novel class examples and use the noise vectors of these examples to fit them and satisfy the constraints of the problem. Then, by showing that $\| \rvw'\| < \underline{\gamma} $ we prove that $\mathrm{AU-ROC}(\rvw_{\mathrm{\text{DD}}}) \leq 0.5$. This is summarized in \cref{proposition:DD_bad}.
Next in \cref{proposition:CoLOR good}, we follow a similar procedure with $\rvw_{\mathrm{\text{color}}}$ of analyzing a problem with an added constraint $\rvw^\top(\mub - \etab) < \tau$ for some $\tau > 0$. Again using weak duality and guessing a solution, we prove that $\rvw_{\mathrm{\text{color}}}$ will satisfy the constraint, this time showing that it obtains a high AU-ROC. Finally, by finding an intersection of parameters ranges where both bounds on the AU-ROC hold, we prove the main claim in \cref{thm:th1_app} (restatement of \cref{thm:th1}).
\end{proof}

\begin{lemma}\label{lemma:q_func}
    For a Gaussian two-feature PU problem, let $\hat{\etab} = \etab / \|\etab\|$ and $\hat{\mub}$ accordingly. We have that the AU-ROC of a model $\rvw\in{\bbR^d}$ is $Q(\frac{\langle \rvw, \etab - \mub \rangle}{\sigma\sqrt{(2\|\rvw\|^2 - \langle\rvw, \hat{\etab} \rangle^2 - \langle\rvw, \hat{\mub} \rangle ^2)}})$.
\end{lemma}
\begin{proof}
Let $\rvx_{\mathrm{novel}}\sim P_{\gT, 1}, \rvx_{\gT, 0}\sim P_{\gT, 0}$, we have
\begin{align*}
\mathrm{AU-ROC}(\rvw) &= P(\langle \rvw, \rvx_{\mathrm{novel}} \rangle < \langle \rvw, \rvx_{\gT, 0}\rangle ) \\
&= P(\langle \rvw, -\etab + \xi_{\mathrm{novel}} \rangle < \langle \rvw, -\mub + \xi_{\mathrm{\gT, 0}}\rangle ) \\
&= P(\langle \rvw, -\etab + \mub \rangle  < \langle \rvw, -\xi_{\mathrm{novel}} + \xi_{\mathrm{\gT, 0}}\rangle ),
\end{align*}
where $\xi_{\mathrm{novel}}\sim\cN(0, \sigma^2(I_d-\hat{\mub}\hat{\mub}^\top))$ and $\xi_{\gT, 0}\sim\cN(0, \sigma^2(I_d-\hat{\etab}\hat{\etab}^\top))$. Then we have that $\langle \rvw, \xi_{\mathrm{novel}} - \xi_{\gT,0} \rangle \sim \cN\left(0, \sigma^2\left( 2\| \rvw \|^2 -\langle \rvw, \hat{\etab} \rangle^2 -\langle \rvw, \hat{\mub} \rangle^2 \right)\right)$. Dividing both sides of the argument above by a factor $\sigma\sqrt{2\|w\|^2 -\langle \rvw, \hat{\etab} \rangle^2 -\langle \rvw, \hat{\mub} \rangle^2}$ we get that for a variable $\xi\sim \cN(0, 1)$
\begin{align*}
\mathrm{AU-ROC}(\rvw) &= P\left(\frac{\langle \rvw, \mub - \etab \rangle}{\sigma \sqrt{2\|\rvw\|^2 - \langle \rvw, \hat{\etab}\rangle ^2 - \langle w, \hat{\mub}\rangle ^2} } < \xi \right) \\
&= Q\left(\frac{\langle \rvw, \mub - \etab \rangle}{\sigma \sqrt{2\|\rvw\|^2 - \langle \rvw, \hat{\etab}\rangle ^2 - \langle \rvw, \hat{\mub}\rangle ^2} }\right) 
\end{align*}

Note that when $\langle \rvw, \mub - \etab \rangle > 0$ then the AU-ROC is lower than $0.5$, and when $\langle \rvw, \mub - \etab \rangle < 0$ then $Q\left(\frac{\langle \rvw, \mub - \etab \rangle}{\sigma \sqrt{2\|\rvw\|^2 - \langle \rvw, \hat{\etab}\rangle ^2 - \langle \rvw, \hat{\mub}\rangle ^2} }\right) \geq Q\left(\frac{\langle \rvw, \mub - \etab \rangle}{\sqrt{2}\sigma\|\rvw\| }\right)$. Hence the RHS is a lower bound on the AU-ROC.

\end{proof}
\begin{proposition}\label{proposition:DD_bad}
    For $\sigma=1/\sqrt{d}$, let $\delta \in{(0,1)}$ be a failure probability, and $C_d, c_d$ constants, and consider the set of problems where 
    \begin{align*}
    C_d\log(\frac{c_d}{\delta}) &\leq \frac{\sqrt{d}}{4\sqrt{N_{\gS} + N_{\gT}}}, \\
    r_\mu &\geq \frac{1}{2\sqrt{N_{\gS} + N_{\gT}}} , \\
    r_\eta &\leq \frac{r_\mu}{4+8r_\mu\sqrt{N_{\gT,1}}}-\frac{1}{\sqrt{N_{\gS} + N_{\gT}}}
    \end{align*}
    There exists $C_d, c_d$ such that for any $\delta$, if we fix $\mub, \etab, r_{\mu}, r_{\eta}, N_{\gS}, N_{\gT,0}, N_{\gT, 1}$ that are in the above set, and drawing training data $(S_\gS, S_\gT)$ as described in problem definition. Let $\rvw_{\mathrm{\text{DD}}}$ as defined in \cref{eq:DD_color_objective}, then with probability at least $1-\delta$ we have that $AU-ROC(\rvw_{\mathrm{\text{DD}}}) \leq 0.5$.
\end{proposition}
\begin{proof}
\textbf{Lower bound on norm of good solutions under DD:}\\
\label{sec:lowerbound_on_good_DD}
Lets take a Lagrangian
\begin{align*}
\cL(\rvw, \lambda, \nu) = \|\rvw\|^2 + \lambda_{\gS}^\top(\mathbf{1}_{\gS} - \boldsymbol{X}_{\gS}\rvw) + \lambda_{\gT}^\top(\mathbf{1}_{\gT} + \boldsymbol{X}_{\gT}\rvw) + \nu\cdot(\rvw^\top(\mub - \etab))
\end{align*}
and zero the gradient
\begin{align*}
\nabla_{\rvw}\cL(\rvw, \lambda, \nu) = 2\rvw - \boldsymbol{X}^\top_{\gS}\lambda_{\gS} + 
\boldsymbol{X}^\top_{\gT}\lambda_{\gT} + \nu\cdot(\mub - \etab) = 0 \\
\Leftrightarrow \rvw = \frac{1}{2}\left(\nu\cdot(\etab - \mub) + \boldsymbol{X}^\top_{\gS}\lambda_{\gS} - \boldsymbol{X}^\top_{\gT}\lambda_{\gT}\right)
\end{align*}
Plug back in to Lagrangian:
\begin{align*}
\cL(\lambda, \nu) = \lambda_{\gS}^{\top}\mathbf{1}_{\gS} + \lambda_{\gT}^{\top}\mathbf{1}_{\gT} - \frac{1}{4}\|\nu\cdot(\etab - \mub) + \boldsymbol{Z}^\top\lambda\|^2
\end{align*}
Here we used the notation $\boldsymbol{Z}$ to denote a matrix whose $i$-th row equals $\rvx_iy_i$, the set of indices of rows where $\rvx_i\in{S_{\gS}}$ as $\mathcal{I}_{\gS}$, and similarly for $\mathcal{I}_{\gT, 0}, \mathcal{I}_{\gT, 1}$.

Let $\alpha > 0$, to be determined later, $\nu = \frac{(|U| - k)2\alpha}{|U|}$, and define the set $U = \mathcal{I}_{\gT,1} \cup \mathcal{I}_{\mathrm{Rand}}$ where $\mathcal{I}_{\mathrm{Rand}}$ has $\frac{|U|}{2\alpha}\nu$ indices randomly drawn from $\mathcal{I}_{\gS}$ and half from $\mathcal{I}_{\gT, 0}$. Hence, $|U|=N$ Set $\lambda = \frac{\alpha}{|U|}\mathbf{1}_U$ and plug back into the Lagrangian. We have that
\begin{align*}
\boldsymbol{Z}^\top\lambda = \nu\mub + \frac{\alpha}{|U|}\left[ k\etab + \sum_{i\in{U}}{\xi_i y_i} \right].
\end{align*}
Hence denoting $\bar{\xi}_U = \frac{1}{|U|}\sum_{i\in{U}}{\xi_iy_i}$ as the average added noise vector (multiplied by the label) over set $U$, we have
\begin{align*}
\mathcal{L}(\alpha, \nu) = \alpha -\frac{1}{4}\| \etab(\nu + k\frac{\alpha}{|U|}) + \alpha\bar{\xi}_U \|^2
\end{align*}
Plugging in the value we already set for $\nu$,
\begin{align*}
\mathcal{L}(\alpha) = \alpha -\frac{1}{4}\| \etab(\frac{\alpha(2|U| - k)}{|U|}) + \alpha\bar{\xi}_U \|^2
\end{align*}
Taking the maximum over $\alpha$, and denoting the optimal solution to the primal optimization problem by $\rvw^*$, we get with probability at least $1-2\exp(-c t^2)$
\begin{align*}
\|\rvw^*\|^2 \geq \mathcal{L}(\alpha^*) &= 4\| \etab(\frac{2|U| - k}{|U|}) + \bar{\xi}_U \|^{-2} \\
&= 4\left( r^2_\eta(2- k/|U|)^2 + \|\bar{\xi}_U\|^2 + 2(2-k/|U|)\etab^\top\bar{\xi}_U \right)^{-1} \\
&\geq 4\left( r^2_\eta(2- k/|U|)^2 + |U|^{-2}\sigma^2(d + 2t\sqrt{d}) + 2(2-k/|U|)r_\eta\sigma |U|^{-1}t \right)^{-1}
\end{align*}

\textbf{Upper bound on norm of bad solutions under DD:\\}
\label{sec:upperbound_on_bad_DD}
Now we will guess a solution $\rvw_{\mathrm{\text{DD}}}$, show it satisfies the constraints of the problem and, under some conditions, achieves lower norm than the bound we obtained in the last section. 
\begin{lemma}
Let $t>0$ then with probability at least $1-2\exp\left[-ct^2k^{-2}\right]-2\exp\left[-ct^2/2\right]-3\exp\left[-ct^2/4\right]-2\exp\left[-ct^2\right]$, the model
\begin{align*}
\rvw_{\mathrm{\text{DD}}} = 2\left( r^2_{\mu} - \sigma r_\mu t\right)^{-1} \cdot \mub - 4\sigma^{-2}kd^{-1}\cdot \bar{\xi}_{\gT,1},
\end{align*}
satisfies $y_i\rvw_{\mathrm{\text{DD}}}^\top\rvx_i \geq 1$ for all $i\in{U}$ and $\rvw_{\mathrm{\text{DD}}}^\top(\mub-\etab) \geq 0$.
\end{lemma}
\begin{proof}
For $i\notin{\mathcal{I}_{\gT,1}}$, 
\begin{align*}
y_i\cdot \rvw^\top_{\mathrm{\text{DD}}}\rvx_i &= \left[2\left( r^2_{\mu} - \sigma r_\mu t \right)^{-1}\cdot \mub - 4\sigma^{-2}kd^{-1}\bar{\xi}_{\gT,1}\right]^\top\left[ \mub + y_i\xi_i \right] \\
&= 2(r^2_{\mu} - \sigma r_{\mu} t)^{-1}(r^2_\mu + y_i\mub^\top\xi_i) - 4\sigma^{-2}kd^{-1}y_i\xi_i^{\top}\bar{\xi}_{\gT,1}
\end{align*}
From \cref{eq:bernstein} we have that $|y_i\xi_i^\top\bar{\xi}_{\gT,1}| < \sigma^2 t\sqrt{\frac{d}{k}}$ w.p. at least $1 - 2\exp \left[-ct^{2}\right]$ for $t\leq d\sqrt{|U|}\sigma^2$. Hence, we can say w.p. at least $1 - 2\exp \left[-ct^{2}\right]$,
\begin{align*}
y_i\cdot \rvw^\top_{\mathrm{\text{DD}}}\rvx_i &\geq 2(r^2_\mu-\sigma r_\mu t)^{-1}(r^2_\mu + y_i\mub^\top\xi_i) - 4\sigma^{-2}kd^{-1} \cdot \sigma^2 t\sqrt{\frac{d}{k}} \\
&= 2(r^2_\mu-\sigma r_\mu t)^{-1}(r^2_\mu + y_i\mub^\top\xi_i) - 4t\sqrt{\frac{k}{d}} \\
&\geq
2(r^2_\mu-\sigma r_\mu t)^{-1}(r^2_\mu + y_i\mub^\top\xi_i) - 1 \text{ \hspace{4mm} for $k\leq d/(16t^2)$ and $t>1$}
\end{align*}
We also have that $|y_i\mub^\top\xi_i| < t\sigma r_\mu$ w.p. at least $1-2\exp(-ct^2/2)$. Putting these together
\begin{align*}
y_i\cdot \rvw^\top_{\mathrm{\text{DD}}}\rvx_i &\geq 2(r^2_\mu-\sigma r_\mu t)^{-1}(r^2_\mu - \sigma r_{\mu} t) - 1 = 1.
\end{align*}

Now for $i\in{\mathcal{I}_{\gT,1}}$, we repeat these steps, but this time use \cref{eq:i_in_U_innerproduct} to bound $|\xi^\top_i\bar{\xi}_{\gT,1}| \geq \frac{\sigma^2d}{2k}$ with probability at least $1-3\exp(-ct^2)$ for $t\leq\frac{\sqrt{d}}{2(2 + \sqrt{|U|-1})}$ and $|\etab^\top \bar{\xi}_{\gT,1}|<t\sigma r_\eta k^{-1}$ w.p. at least $1 - 2\exp\left(-ct^2/2\right)$. We get again that
\begin{align*}
y_i\cdot \rvw^\top_{\mathrm{\text{DD}}}\rvx_i &= 2(r^2_{\mu} - \sigma r_{\mu} t)^{-1}(-y_i \mub^\top \etab + y_i\cdot \mub^\top \xi_i) - 4\sigma^{-2}kd^{-1}(-y_i\etab^\top\bar{\xi}_{\gT,1} + y_i\cdot \xi_i^\top\bar{\xi}_{\gT,1})\\
&= 4\sigma^{-2}kd^{-1}(-\etab^\top\bar{\xi}_{\gT,1} + \xi_i^\top\bar{\xi}_{\gT,1}) \\
&\geq 4\sigma^{-2}kd^{-1}(-\frac{t\sigma r_\eta}{k} + \frac{\sigma^2d}{2k})\\
\end{align*}
For $r_\eta \leq \frac{\sigma d}{4t}$
\begin{align*}
y_i\cdot \rvw^\top_{\mathrm{\text{DD}}}\rvx_i & \geq 4\sigma^{-2}kd^{-1}(\frac{\sigma^2d}{4k})  = 1
\end{align*}

Finally, let us calculate $\rvw_{\mathrm{\text{DD}}}^\top(\mub-\etab)$,
\begin{align*}
\rvw_{\mathrm{\text{DD}}}^\top(\mub-\etab) &= 2(r_{\mu}^2 -\sigma r_{\mu}t)^{-1}r^2_{\mu} + 4\sigma^{-2}kd^{-1}\etab^\top\bar{\xi}_{\gT,1} \\
&\geq 2(1-\sigma r^{-1}_{\mu}t)^{-1} + 4\sigma^{-2}kd^{-1}\sigma r_\eta t k^{-1}\\
&= 2(1-\sigma r^{-1}_{\mu}t)^{-1} + 4\sigma^{-1}d^{-1}r_\eta t\\
&\geq 0 \text{ \hspace{4mm} for $r_\mu\geq\sigma t$ or more tighter condition is $r_\eta \geq \frac{\sigma d}{2t}(\sigma r_\mu^{-1}t - 1)^{-1}$.}
\end{align*}


The first inequality holds w.p. at least $1-2\exp\{-c t^{2}\}$.
Hence, for $\frac{\sigma d}{2t}(\sigma r_\mu^{-1}t - 1)^{-1} \leq r_\eta\leq \frac{\sigma d}{4t}$, we have w.p. at least $1-2\exp\{-c t^{2}\}$ that
\begin{align*}
    \rvw_{\mathrm{\text{DD}}}^\top(\mub-\etab) > 0 
\end{align*}

Overall, taking the union bounds so that these inequalities hold over the entire dataset, we get that the solution $\rvw_{\mathrm{\text{DD}}}$ satisfies all the constraints with probability $1-3|U|\exp(-ct^2/4) - 2\exp(-ct^2/2) - 2\exp(-ct^2)$.
\end{proof}

Requirements for parameters:
\begin{itemize}
\item $t\leq d\sqrt{|U|}\sigma^2$ 
\item $r_\mu \geq \sigma t$
\item $t\leq\frac{\sqrt{d}}{2(2 + \sqrt{|U|-1})}$
\item $\frac{\sigma d}{2t}(\sigma r_\mu^{-1}t - 1)^{-1} \leq r_\eta\leq \frac{\sigma d}{4t}$ (Lower bound is negative, so can be ignored)
\end{itemize}

\textbf{Finding ranges where AU-ROC is smaller than 0.5:\\}
Compare $\|\rvw_{\mathrm{\text{DD}}}\|$ and lower bounds on $\|\rvw^*\|$ obtained in earlier parts and find regions where we must have $\|\rvw_{\mathrm{\text{DD}}}\| < \|\rvw^*\|$

From Section~\ref{sec:upperbound_on_bad_DD}

We have that
\[
\|\rvw_{\mathrm{\text{DD}}}\| = \left\| \frac{2}{r^2_{\mu} - \sigma r_{\mu}t} \cdot \mub - \frac{4k}{\sigma^{2}d} \bar{\xi}_{\gT,1} \right\|.
\]
Hence from triangle inequality,
\[
\|\rvw_{\mathrm{\text{DD}}}\| \leq \frac{2}{r^2_{\mu} - \sigma r_{\mu}t_1} \cdot \|\mub\| + \frac{4k}{\sigma^{2}d} \cdot \|\bar{\xi}_{\gT,1}\|.
\]
We know that $\|\mub\| = r_{\mu}$ and $\bar{\xi}_{\gT,1} \sim \cN(0,\frac{\sigma^2}{k}I_d)$. Hence, using the concentration bound in \ref{eq:noise_chi_squared_gaussian}, we have with probability at least $1-\exp{\left(-\frac{t^2}{4}\right)}$ that
\[
\|\bar{\xi}_{\gT,1}\| \leq \sigma\sqrt{\frac{d + 2t\sqrt{d}}{k}}
\]
Hence with probability at least $1-\exp{\left(-\frac{t^2}{4}\right)}$,
\begin{align*}
\|\rvw_{\mathrm{\text{DD}}}\| &\leq \frac{2}{r_\mu - \sigma t} + \frac{4k}{\sigma d}\sqrt{\frac{d + 2t\sqrt{d}}{k}}\\
&\leq \frac{2}{r_\mu - \sigma t} + 8\sigma^{-1}\sqrt{\frac{k}{d}} \text{ \hspace{4mm} for $t\leq \frac{3}{2}\sqrt{d}$}
\end{align*}


We also have from section \ref{sec:lowerbound_on_good_DD}, 
\[\|\rvw^*\| \geq 2 \| \etab(\frac{2|U| - k}{|U|}) + \bar{\xi}_U \|^{-1}\]
\[\|\rvw^*\| \geq 2 (\|\etab\|(\frac{2|U| - k}{|U|}) + \|\bar{\xi}_U\|)^{-1}\]

Similarly, using concentration bound in \ref{eq:noise_chi_squared_gaussian} for $\|\bar{\xi}_U\|$, we get with probability at least $1-\exp{\left(-\frac{t ^2}{4}\right)}$, 
\begin{align*}
\|\rvw^*\| \geq \frac{2}{r_\eta(2-k/|U|) + \sigma \sqrt{\frac{d+2t\sqrt{d}}{|U|}}}
\end{align*}
For $t\leq \frac{3}{2}\sqrt{d} \implies \sqrt{d+2t\sqrt{d}}\leq 2\sqrt{d}$, hence
\begin{align*}
\|\rvw^*\| \geq \frac{2}{r_\eta(2-k/|U|) + 2\sigma \sqrt{\frac{d}{|U|}}}
\end{align*}
We know that $k\geq 0$,
Hence, the following inequality will guarantee \ref{eq:DD_to_prove} with probability at least $1-\exp{\left(-\frac{t ^2}{4}\right)}$,
\begin{align}\label{eq:w*_lower_bound}
\|\rvw^*\| \geq \frac{2}{2r_\eta + 2\sigma \sqrt{\frac{d}{|U|}}}
\end{align}

By union bounds w.p. at least $1-2\exp{\left(-\frac{t ^2}{4}\right)}$, we want conditions for
\begin{align*}
\frac{2}{2r_\eta + 2\sigma \sqrt{\frac{d}{|U|}}} \geq \frac{2}{r_\mu - \sigma t} + 8\sigma^{-1}\sqrt{\frac{k}{d}} \\
\frac{1}{2r_\eta + 2\sigma \sqrt{\frac{d}{|U|}}} \geq \frac{1}{r_\mu - \sigma t} + \frac{4\sqrt{k}}{\sigma\sqrt{d}} \\
\end{align*}
\begin{align}\label{eq:DD_to_prove}
2r_\eta + 2\sigma \sqrt{\frac{d}{|U|}} \leq \left(\frac{1}{r_\mu - \sigma t} + \frac{4\sqrt{k}}{\sigma\sqrt{d}}\right)^{-1}
\end{align}
For $r_\mu\geq 2\sigma t$, $\left(\frac{2}{r_\mu} + \frac{4\sqrt{k}}{\sigma\sqrt{d}}\right)^{-1} \leq \left(\frac{1}{r_\mu - \sigma t} + \frac{2\sqrt{k}}{\sigma\sqrt{d}}\right)^{-1}$. Hence, eq. \ref{eq:DD_to_prove} is satisfied if the following is guaranteed,
\begin{align*}
2r_\eta + 2\sigma \sqrt{\frac{d}{|U|}} \leq \left(\frac{2}{r_\mu} + \frac{4\sqrt{k}}{\sigma\sqrt{d}}\right)^{-1}\\
r_\eta \leq \frac{\sigma\sqrt{d}}{4} \left(\frac{r_\mu}{\sigma\sqrt{d} + 2r_\mu \sqrt{k}} - \frac{4}{\sqrt{|U|}}\right)
\end{align*}
Or more tighter condition that will guarantee eq. \ref{eq:DD_to_prove} is the following using $r_\mu \geq 2 \sigma t$
\begin{align*}
r_\eta \leq \frac{\sigma\sqrt{d}}{4} \left(\frac{2t}{\sqrt{d} + 4t \sqrt{k}} - \frac{4}{\sqrt{|U|}}\right)
\end{align*}
This means that the above parameter requirements are sufficient to guarantee that a domain discriminator is a suboptimal solution to the novelty detection problem under background shift with probability at least $1-2\exp{\left(-\frac{t ^2}{4}\right)}$.

Parameter requirements to guarantee eq. \ref{eq:DD_to_prove} w.p. at least $1-2\exp{\left(-\frac{t ^2}{4}\right)}$:
\begin{itemize}
\item $t\leq d\sqrt{|U|}\sigma^2$ 
\item $t\leq \frac{3}{2}\sqrt{d}$  \hspace*{1cm}  (redundant see next condition below.)
\item $t\leq\frac{\sqrt{d}}{2(2 + \sqrt{|U|-1})}$
\item $r_\mu \geq 2\sigma t$
\item $\frac{\sigma d}{2t}(\sigma r_\mu^{-1}t - 1)^{-1} \leq r_\eta\leq \frac{\sigma d}{4t}$ \hspace*{1cm} (Negative lower bound, so can be ignoredy)
\item $r_\eta \leq \frac{\sigma\sqrt{d}}{4} \left(\frac{r_\mu}{\sigma\sqrt{d} + 2r_\mu \sqrt{k}} - \frac{4}{\sqrt{|U|}}\right)$
\end{itemize}

\end{proof}

\begin{proposition}\label{proposition:CoLOR good}
    For $\sigma=1/\sqrt{d}$, there exist some constants $C_d$ and $c_d$, and for any failure probability $0\leq\delta\leq1$, if
    \begin{align*}
    C_d\log(\frac{c_d}{\delta}) &\leq \frac{\sqrt{d}}{4\sqrt{N_{\gS} + N_{\gT}}}, \\
    \frac{1}{2\sqrt{N_{\gT}}} &\leq r_\mu \leq \sqrt{(N_{\gS}+N_{\gT})\cdot N_{\gT,0}}, \\
    r_\eta &\geq \frac{1}{2\sqrt{N_{\gT}}}, \\
    \frac{3}{r_\eta}+\frac{3-5\tau}{r_\mu} &\leq\frac{3}{8}(1-\tau)\sqrt{N_{\gT}}-2
    \end{align*}
    then with probability at least $1-\delta$ over the drawing of $(S_\gS, S_\gT)$ as described in problem, $AU-ROC(\hat{\rvw}_{\mathrm{\text{color}}}) \geq 0.9$ as defined in first part of \cref{eq:DD_color_objective}
    is at least $Q\left(\frac{-\tau\sqrt{d}}{ \sqrt{2}\|\hat{\rvw}_{\mathrm{\text{color}}}\|}\right) \geq Q\left(\frac{-\tau\sqrt{d}}{ \sqrt{2}(\frac{3}{r_\eta}+\frac{3-5\tau}{r_\mu}+2)}\right) \geq Q\left(\frac{-\tau\sqrt{d}}{\sqrt{2}(1-\tau)}\right)$. 
\end{proposition}
\begin{proof}
\textbf{Lower bound on norm of bad solutions under \ours{}:\\}
This time the Lagrangian is
\begin{align*}
\cL(\rvw, \lambda, \nu) = \|\rvw\|^2  -\lambda_{\gS}^\top( \boldsymbol{X}_{\gS}\rvw) + \lambda_{\gT}^\top(\mathbf{1}_{\gT} + \boldsymbol{X}_{\gT}\rvw) - \nu\cdot(\rvw^\top(\mub - \etab) + \tau)
\end{align*}
and zero the gradient
\begin{align*}
\nabla_{\rvw}\cL(\rvw, \lambda, \nu) = 2\rvw - \boldsymbol{X}^\top_{\gS}\lambda_{\gS} + 
\boldsymbol{X}^\top_{\gT}\lambda_{\gT} - \nu\cdot(\mub - \etab) = 0 \\
\Leftrightarrow \rvw = \frac{1}{2}\left(\nu\cdot(\mub - \etab) + \boldsymbol{X}^\top_{\gS}\lambda_{\gS} - \boldsymbol{X}^\top_{\gT}\lambda_{\gT}\right)
\end{align*}
Plug back in to Lagrangian:
\begin{align*}
\cL(\lambda, \nu) = - \nu\cdot \tau + \lambda_{\gT}^{\top}\mathbf{1}_{\gT} - \frac{1}{4}\|\nu\cdot(\mub - \etab) + \boldsymbol{Z}^\top\lambda\|^2
\end{align*}
Consider setting $\lambda = \frac{\nu}{|U|}\mathbf{1}_{S_{\gT}}$, where $U$ has randomly drawn $|U|$ ($|U|=N_\gT$) examples from $S_{\gT}$. Also set $\nu = \frac{(1-\tau)}{\| \bar{\xi}_{S_{\gT}} \|^2}$. We have
\begin{align*}
\boldsymbol{Z}^\top\lambda &= -\nu\cdot \mub + \nu\bar{\xi}_{S_{\gT,0}}+\nu\cdot \etab + \nu\bar{\xi}_{S_{\gT,1}}\\
&= \nu(\etab-\mub) + \nu\bar{\xi}_{S_{\gT}}
\end{align*}
and then the lagrangian
\begin{align*}
\cL(\lambda, \nu) &= \nu\cdot(1-\tau) - \frac{1}{4}\| \nu\bar{\xi}_{S_{\gT}}\|^2 \\
&= \nu\cdot(1-\tau) - \frac{1}{4}\nu^2\left( \|\bar{\xi}_{S_{\gT}}\|^2\right)\| \\
\cL(\lambda)&= \frac{3}{4}\frac{(1-\tau)^2}{\| \bar{\xi}_{S_{\gT}} \|^2}
\end{align*}

\begin{align*}
    \|\hat{\rvw}_{\mathrm{color}}\|^2 \geq \cL(\lambda^*) = \frac{3}{4}\frac{(1-\tau)^2}{\| \bar{\xi}_{S_{\gT}} \|^2}
\end{align*}

\textbf{Getting feasible solution for our optimization problem:}
We will guess a solution of the form:
\begin{align*}
\rvw_{\mathrm{color}} = \alpha \etab + (\frac{\alpha r^2_{\eta}}{r^2_{\mu}} - \frac{5\tau}{r^2_{\mu}})\mub - \beta\bar{\xi}_{S_{\gT, 0}}\\
\rvw_{\mathrm{color}} = 2\frac{\etab}{r^{2}_\eta} + (\frac{2}{r^2_{\mu}} - \frac{5\tau}{r^2_{\mu}})\mub - \frac{1}{\sigma^2} \cdot\frac{|U|-k}{d}\bar{\xi}_{S_{\gT, 0}}\\
\rvw_{\mathrm{color}} = 2\frac{\etab}{r^{2}_\eta} + (\frac{2}{r^2_{\mu}} - \frac{5\tau}{r^2_{\mu}})\mub - (|U|-k)\bar{\xi}_{S_{\gT, 0}}
\end{align*}
To have $\rvw_{\mathrm{color}}^\top \rvx_i y_i > 1$ for $i\in{\mathcal{I}_{\gT, 1}}$, we need
\begin{align*}
\rvw_{\mathrm{color}}^\top (\etab - \xi_i) = \alpha r^2_\eta + \beta\xi_i^\top\bar{\xi}_{S_{\gT,0}} > 1
\end{align*}

We have
\begin{align*}
\rvw_{\mathrm{color}}^\top(-y_i\etab + y_i\xi_i) &= \rvw_{\mathrm{color}}^\top (\etab - \xi_i)\\
&= \alpha r^2_\eta - \alpha\etab^\top\xi_i + \beta\xi_i^\top\bar{\xi}_{\gT, 0}
\end{align*}
We know that $|\etab^\top\xi_i|<\sigma t r_\eta$ w.p. at least $1 - 2\exp\left(-ct^2/2\right)$ and $|\xi_i\bar{\xi}_{\gT,0}|\leq\sigma^2t\sqrt{\frac{d}{|U| -k}}$ w.p. at least $1 - 2\exp(-ct^2)$ for $t\leq d\sqrt{|U|}\sigma^2$. Hence, w.p. at least $1 - 2\exp\left(-ct^2/2\right) - 2\exp(-ct^2)$, we can say that,
\begin{align*}
    \rvw_{\mathrm{color}}^\top (\etab - \xi_i) \geq \alpha r^2_\eta - \alpha\sigma t r_\eta  - \beta\sigma^2t\sqrt{\frac{d}{|U| -k}}
\end{align*}
Required condition:
\begin{align*}
    \alpha r^2_\eta - \alpha\sigma t r_\eta  - \beta\sigma^2t\sqrt{\frac{d}{|U| -k}} \geq 1
\end{align*}

Possibly set $\alpha = 3r^{-2}_\eta$ and $\beta=\sigma^{-2}\frac{\sqrt{|U|-k}}{d}$.\\
\begin{align*}
    3 - 3\sigma t r^{-1}_\eta - \frac{t}{\sqrt{d}} \geq 1\\
    \frac{t}{\sqrt{d}} \leq 2 - 3\sigma t r^{-1}_\eta
\end{align*}
For $2 \sigma t \leq r_\eta$,
\begin{align*}
    \frac{t}{\sqrt{d}} \leq \frac{1}{2}
\end{align*}
Parameter requirements:
\begin{itemize}
    \item $\sigma t \leq \frac{r_\eta}{2}$
    \item $t\leq \frac{\sqrt{d}}{2}$
\end{itemize}

To also have $\rvw_{\mathrm{color}}^\top \rvx_i y_i > 1$ for $i\in{\mathcal{I}_{\gT, 0}}$ we need
\begin{align*}
\rvw_{\mathrm{color}}^\top (\mub + \xi_i) = (\frac{\alpha r^2_{\eta}}{r^2_{\mu}}-\frac{5\tau}{r^2_{\mu}})r^2_\mu + \beta\xi_i^\top\bar{\xi}_{S_{\gT, 0}} > 1
\end{align*}

We have
\begin{align*}
\rvw_{\mathrm{color}}^\top(-y_i\mub + y_i\xi_i) &= \rvw_{\mathrm{color}}^\top(\mub - \xi_i) \\
&= \left(\frac{\alpha r^2_\eta}{r^2_\mu} - \frac{5\tau}{r^2_\mu} \right) (\mub^\top\mub-\mub^\top\xi_i) - \beta(\mub^\top\bar{\xi}_{\gT,0} - \xi_i^\top\bar{\xi}_{\gT, 0})
\end{align*}
We know that $|\mub^\top\xi_i|\leq\sigma t r_\mu$ w.p. at least $1 - 2\exp\left(-ct^2/2\right)$ and $|\xi_i^\top\bar{\xi}_{\gT,0}|\geq\frac{\sigma^2d}{2(|U|-k)}$ w.p. at least $1 - 2\exp(-ct^2)$ for $t\leq \sigma^2$ and $t\leq\frac{\sqrt{d}}{2(2 + \sqrt{|U|-1})}$. We also know that $|\mub^\top \bar{\xi}_{\gT,0}|\leq \frac{\sigma tr_\mu}{|U|-k}$ w.p. at least $1-2\exp(-ct^2/2)$ Hence, w.p. at least $1 - 4\exp\left(-ct^2/2\right) - 2\exp(-ct^2)$, we can say that,
\begin{align*}
    \rvw_{\mathrm{color}}^\top(-y_i\mub + y_i\xi_i) &\geq \left(\frac{\alpha r^2_\eta}{r^2_\mu} - \frac{5\tau}{r^2_\mu}\right)(r^2_\mu - \sigma t r_\mu) + \beta\left(-\frac{\sigma tr_\mu}{|U|-k} + \frac{\sigma^2d}{2(|U|-k)}\right)\\
    &\geq \frac{r^2_\mu}{2}\left(\frac{\alpha r^2_\eta}{r^2_\mu} - \frac{5\tau}{r^2_\mu}\right) + \beta\left(\frac{\sigma^2d - r^2_\mu}{2(|U|-k)}\right) \text{ \hspace{4mm} for $r_\mu \geq 2\sigma t$}
\end{align*}
Required condition:
\begin{align*}
    \frac{1}{2}\left( \alpha r^2_\eta - 5\tau \right) + \beta\left(\frac{\sigma^2d - 2\sigma t r_\mu}{2(|U|-k)}\right)\geq1\\
\end{align*}

Setting $\alpha=3r^{-2}_\eta$ we get,
\begin{align*}
    \frac{1}{2}\left( 3 - 5\tau \right) + \beta\left(\frac{\sigma^2d - 2\sigma t r_\mu}{2(|U|-k)}\right)\geq1\\
    \beta\left(\frac{\sigma^2d - 2\sigma t r_\mu}{2(|U|-k)}\right) \geq \frac{5\tau-1}{2}\\
    \frac{\beta\sigma t r_\mu}{|U|-k} \leq \frac{\beta\sigma^2d}{2(|U|-k)} + \frac{1-5\tau}{2}
\end{align*}
Setting $\beta=\sigma^{-2}\frac{\sqrt{|U|-k}}{d}$ we get,
\begin{align*}
    \frac{r_\mu t}{\sigma d} \leq \frac{1}{2}+\frac{1-5\tau}{2}\sqrt{|U|-k}
\end{align*}
Parameter requirements:
\begin{itemize}
    \item $2\sigma t\leq r_\mu$
    \item $r_\mu t \leq \frac{\sigma d}{2}(1+(1-5\tau)\sqrt{|U|-k}) $
\end{itemize}

Finally, we also require $\rvw_{\mathrm{color}}^\top \rvx_i y_i > 0$ for $i\in{\mathcal{I}_{\gS, 0}}$, we need 

We have
\begin{align*}
\rvw_{\mathrm{color}}^\top(y_i\mub + y_i\xi_i) &= \rvw_{\mathrm{color}}^\top(\mub + \xi_i) \\
&= \left(\frac{\alpha r^2_\eta}{r^2_\mu} - \frac{5\tau}{r^2_\mu} \right) (\mub^\top\mub+\mub^\top\xi_i) - \beta(\mub^\top\bar{\xi}_{\gT,0} + \xi_i^\top\bar{\xi}_{\gT, 0})
\end{align*}
We know that $|\mub^\top\xi_i|\leq\sigma t r_\mu$ w.p. at least $1 - 2\exp\left(-ct^2/2\right)$ and $|\xi_i^\top\bar{\xi}_{\gT,0}|\leq\sigma^2t\sqrt{\frac{d}{|U|-k}}$ w.p. at least $1 - 2\exp(-ct^2)$ for $t\leq \sigma^2$ and $t\leq\frac{\sqrt{d}}{2(2 + \sqrt{|U|-1})}$. We also know that $|\mub^\top \bar{\xi}_{\gT,0}|\leq \frac{\sigma tr_\mu}{|U|-k}$ w.p. at least $1-2\exp(-ct^2/2)$ Hence, w.p. at least $1 - 4\exp\left(-ct^2/2\right) - 2\exp(-ct^2)$, we can say that,
\begin{align*}
    \rvw_{\mathrm{color}}^\top(-y_i\mub + y_i\xi_i) &\geq \left(\frac{\alpha r^2_\eta}{r^2_\mu} - \frac{5\tau}{r^2_\mu}\right)(r^2_\mu - \sigma t r_\mu) - \beta\left(\frac{\sigma tr_\mu}{|U|-k} + \sigma^2t\sqrt{\frac{d}{|U|-k}}\right)\\
    &\geq \frac{r^2_\mu}{2}\left(\frac{\alpha r^2_\eta}{r^2_\mu} - \frac{5\tau}{r^2_\mu}\right) - \beta\left(\frac{\sigma t r_\mu}{|U|-k} + \sigma^2 t\sqrt{\frac{d}{|U|-k}}\right) \text{ \hspace{4mm} for $r_\mu \geq 2\sigma t$}
\end{align*}

Required condition:
\begin{align*}
    \frac{1}{2}\left( \alpha r^2_\eta - 5\tau \right) - \beta\left(\frac{\sigma t r_\mu}{|U|-k} + \sigma^2 t\sqrt{\frac{d}{|U|-k}}\right) \geq 0
\end{align*}
Using $\alpha$ and $\beta$,
\begin{align*}
    \frac{3 - 5\tau}{2}  - \left(\frac{t r_\mu}{\sigma d\sqrt{|U|-k}} + \frac{t}{\sqrt{d}}\right) \geq 0
\end{align*}
The minimum value of the LHS above would be $\frac{3-5\tau}{2}  - \left(\frac{t r_\mu}{\sigma d\sqrt{|U|-k}} + \frac{1}{2}\right)$ We want this term to be greater than $0$. Hence, 
\begin{align*}
    \frac{1-5\tau}{2}  - \frac{t r_\mu}{\sigma d \sqrt{|U|-k}} \geq 0\\
    \frac{t r_\mu}{\sigma d}\leq \frac{1-5\tau}{2}\sqrt{|U|-k}
\end{align*}
Substituting $\alpha=3r^{-2}_\eta$ and $\beta=\sigma^{-2}\frac{\sqrt{|U|-k}}{d}$, we get,
\begin{align*}
    \frac{t r_\mu}{\sigma d} \leq \frac{1-5\tau}{2}\sqrt{|U|-k}
\end{align*}

Parameter requirements:
\begin{itemize}
    \item $t\leq \frac{\sqrt{d}}{2}$
    \item $r_\mu t \leq \frac{1}{2}\sigma d (1-5\tau) \sqrt{|U|-k}$
\end{itemize}


We want $\rvw_{\mathrm{color}}^\top(\mub-\etab)\leq -\tau$ i.e. $\rvw_{\mathrm{color}}^\top(\etab-\mub)\geq \tau$\\
We have
\begin{align*}
\rvw_{\mathrm{color}}^\top(\etab-\mub) &= \alpha \etab^\top\etab - \left(\frac{\alpha r^2_\eta}{r^2_\mu} - \frac{5\tau}{r^2_\mu} \right) \mub^\top\mub + \beta\cdot\mub^\top\bar{\xi}_{\gT,0}
\end{align*}
We know that $|\mub^\top \bar{\xi}_{\gT,0}|\leq \frac{\sigma tr_\mu}{|U|-k}$ w.p. at least $1-2\exp(-ct^2/2)$ Hence, w.p. at least $1 - 2\exp\left(-ct^2/2\right)$, we can say that,
\begin{align*}
    \rvw_{\mathrm{color}}^\top(\etab-\mub) &\geq \alpha r^2_\eta - \left(\frac{\alpha r^2_\eta}{r^2_\mu} - \frac{5\tau}{r^2_\mu}\right) r^2_\mu - \beta\frac{\sigma t r_\mu}{|U|-k}\\
\end{align*}

Required condition:
\begin{align*}
    5\tau - \beta\frac{\sigma t r_\mu}{|U|-k} \geq \tau\\
    \beta\frac{\sigma t r_\mu}{|U|-k}\leq 4\tau\\
    \frac{t r_\mu}{\sigma d\sqrt{|U|-k}} \leq 4\tau
\end{align*}
Parameter requirements:
\begin{itemize}
    \item $\frac{t r_\mu}{\sigma d} \leq 4\tau\sqrt{|U|-k}$
\end{itemize}

Bound on the norm of $\rvw_{\mathrm{color}}$:
\begin{align*}
    \rvw_{\mathrm{color}} = 3\frac{\etab}{r^{2}_\eta} + \left(\frac{3}{r^2_{\mu}} - \frac{5\tau}{r^2_{\mu}}\right)\mub - \sigma^{-2}\frac{\sqrt{|U|-k}}{d}\bar{\xi}_{S_{\gT, 0}}\\
    \|\rvw_{\mathrm{color}}\| \leq \frac{3}{r_\eta} + \left(\frac{3}{r_{\mu}} - \frac{5\tau}{r_{\mu}}\right) + \sigma^{-1}\frac{\sqrt{|U|-k}}{d} \sqrt{\frac{d+2t\sqrt{d}}{|U|-k}}\\
    \|\rvw_{\mathrm{color}}\| \leq \frac{3}{r_\eta} + \frac{3}{r_{\mu}} - \frac{5\tau}{r_{\mu}} + 2 \\
    \|\rvw_{\mathrm{color}}\| \leq \frac{3}{r_\eta} + \frac{3-5\tau}{r_{\mu}} + 2\\
    \text{\hspace{4mm} for $t\leq\frac{3}{2}\sqrt{d}$}
\end{align*}

We also want $\|\rvw_{\mathrm{color}}\|\leq \|\hat{\rvw}_{\mathrm{color}}\|$. For this, it is sufficient to ensure the following w.p. at least $1-2\exp(-ct^2)$:
\begin{align*}
    \frac{3}{r_\eta} + \frac{3-5\tau}{r_{\mu}} + 2
    &\leq \frac{3}{4}\frac{1-\tau}{\sqrt{(|U|)^{-1}\sigma^2(d+2t\sqrt{d})}}\\ 
    &\leq \frac{3}{8}\frac{1-\tau}{\sqrt{(|U|)^{-1}\sigma^2d}}  \text{ \hspace{4mm} for $t\leq \frac{3}{2}\sqrt{d}$}\\
    &\leq \frac{3}{8}(1-\tau)\sqrt{|U|}  \text{\hspace{4mm} for $\sigma =  \frac{1}{\sqrt{d}}$}\\
    \frac{3}{r_\eta} + \frac{3-5\tau}{r_{\mu}} &\leq \frac{3}{8}(1-\tau)\sqrt{|U|} - 2\\
\end{align*}
which satisfies the constraint in proposition 2.

Parameter requirements if $\alpha=2r^{-2}_\eta$ and $\beta=\sigma^{-2}\frac{|U|-k}{d}$:
\begin{itemize}
    \item $t\leq\frac{\sqrt{d}}{2(2+\sqrt{|U|-1})}$ 
    \item $t\leq\frac{\sqrt{d}}{2}$
    \item $2\sigma t \leq r_\eta$
    \item $2\sigma t\leq r_\mu$
    \item $r_\mu t \leq \min\left(\frac{1}{2}\sigma d (1-5\tau) \sqrt{|U|-k}, 4 \sigma d \tau \sqrt{|U|-k} \right)$
    \item $\frac{3}{r_\eta} + \frac{3-5\tau}{r_\mu}\leq\frac{3}{8}(1-\tau)\sqrt{|U|}-2$
\end{itemize}

Set $\sigma = \frac{1}{\sqrt{d}}$ and $t=\frac{\sqrt{d}}{4\sqrt{|U|}}$\\
Hence, $\frac{1}{2\sqrt{|U|}}\leq r_\eta$\\
$\frac{1}{2\sqrt{|U|}}\leq r_\mu$\\
$r_\mu\leq \min \{2\sqrt{|U|}(1-5\tau), 8\sqrt{|U|}\tau\}$\\

Now we want, 
\[Q\left(\frac{-\tau\sqrt{d}}{\sqrt{2}(1-\tau)}\right) \geq 0.9\]
This means we want,
\begin{align*}
    \frac{-\tau\sqrt{d}}{\sqrt{2}(1-\tau)} \leq -1.283 \\
    \frac{\tau\sqrt{d}}{\sqrt{2}(1-\tau)} \geq 1.283 \\
    \tau(\sqrt{\frac{d}{2}} + 1.283) \geq 1.283 \\
    \tau \geq \frac{1.283}{\sqrt{\frac{d}{2}}+1.283}
\end{align*}
From parameter requirements, we can also say,
\begin{align}\label{eq:tau_range}
    \frac{1.283}{\sqrt{\frac{d}{2}}+1.283} \leq \tau \leq 0.2 
\end{align}

For $d>300$, we can easily say pick $\tau=0.1$. This gives us $r_\mu t \leq \min\left(\frac{1}{4}\sigma d \sqrt{|U|-k}, 0.4\sigma d \sqrt{|U|-k}\right) = \frac{1}{4}\sigma d \sqrt{|U|-k}$
This gives us final parameter ranges as follows,
\begin{itemize}
    \item $t\leq\frac{\sqrt{d}}{2(2+\sqrt{|U|-1})}$ 
    \item $t\leq\frac{\sqrt{d}}{2}$
    \item $2\sigma t \leq r_\eta$
    \item $2\sigma t\leq r_\mu$
    \item $r_\mu t \leq \frac{1}{4}\sigma d \sqrt{|U|-k}$
    \item $\frac{3}{r_\eta} + \frac{3-5\tau}{r_\mu}\leq\frac{3}{8}(1-\tau)\sqrt{|U|}-2$
\end{itemize}
\end{proof}

\begin{theorem}\label{thm:th1_app}
(Restating \cref{thm:th1})\\
Consider a Linear-Gaussian PU-learning problem with parameters $\mub, \etab, d > 300$, dataset sizes $N_{\gT} > 10$, $N_{\gS} \geq N_{\gT}$, $N=N_{\gS}+N_{\gT}$, $\alpha \in{[\frac{1}{N_\gT}, \frac{N}{1024 N_\gT})}$ 
and let $\delta\in{(0,1)}$.
For all problems where
\begin{align}
\min\{r_{\eta}, r_{\mu}\} &\geq \tfrac{16}{\sqrt{N_\gT}}, \label{eq:cond1}\\
r_{\mu} &\leq \tfrac{1}{2}\sqrt{N_{\gT, 0}}, \label{eq:cond2}\\
\tfrac{r_{\eta}}{r_{\mu}} &\leq \tfrac{1}{8(1+\sqrt{N_{\gT}N_{\gT,1}})}, \label{eq:cond3}\\
c_1\log\!\left(\tfrac{c_2}{\delta}\right) &\leq \min{\left(\sqrt{N}, \sqrt{\frac{d}{N}}\right)} \label{eq:cond4}
\end{align}
it holds with probability at least $1-\delta$ that $\mathrm{AU-ROC}(\rvw_{\mathrm{\text{DD}}}) < 0.5$ and $\mathrm{AU-ROC}(\rvw_{\mathrm{\text{color}}}) > 0.9$.
\end{theorem}
\begin{proof}\label{theorem:proof}
Let us denote $p=\frac{N}{N_\gT}$. Then based on the stated assumptions on $N_\gS, N_{\gT}, N_{\gT,0}, N_{\gT,1}$ we can say,\\
\begin{align}\label{eq:nn_t1}
    \sqrt{N\cdot N_{\gT,1}} &= \sqrt{pN_{\gT}\cdot N_{\gT,1}} \nonumber\\
    &\geq\sqrt{pN_{\gT,0}\cdot N_{\gT,1}} \nonumber\\
    &=(\sqrt{p}-1)\sqrt{N_{\gT,0}\cdot N_{\gT,1}}+\sqrt{N_{\gT,0}\cdot N_{\gT,1}}\nonumber\\
    &\geq 1+\sqrt{N_{\gT,0} \cdot N_{\gT,1}}
\end{align}
The last transition uses $(\sqrt{p}-1)\sqrt{N_{\gT,0}\cdot N_{\gT,1}}\geq 1$, which holds because $N_{\gS}\geq N_{\gT}$ gives $p\geq 2$, while $N_{\gT}>10$ and $N_{\gT,1}\geq 1$ give $N_{\gT,0}\geq 10$, so that $(\sqrt{2}-1)\sqrt{10}>1$.


Using \cref{eq:cond2} we can write,
$1+\sqrt{N_{\gT,0} \cdot N_{\gT,1}}\geq 1+2r_\mu\sqrt{N_{\gT,1}}$

From \cref{eq:cond1} we know that $r_\mu\sqrt{N}\geq 16$. Hence, \\
\begin{align}\label{eq:r_mu_sqrt_n}
    r_\mu\sqrt{N}&\geq 8+\frac{1}{2}r_\mu\sqrt{N} \nonumber\\
    &\geq 8+\frac{1}{2}r_\mu\sqrt{\frac{p}{\alpha}\cdot N_{\gT, 1}} \nonumber\\
    &\geq 8+16r_\mu\sqrt{N_{\gT,1}}=8(1+2r_\mu\sqrt{N_{\gT,1}})
\end{align}

Here, the first transition is from \cref{eq:cond1} while second and third from the range of $\alpha = \frac{pN_{\gT,1}}{N}\leq \frac{p}{1024}$.

From the \cref{proposition:DD_bad} and \cref{proposition:CoLOR good}, we have the following ranges of parameter ranges for $t=c_1\log(\frac{c_2}{\delta})$ where the \auroc($\rvw_{\mathrm{\text{DD}}}$)<0.5 and \auroc($\rvw_{\mathrm{\text{color}}}$)>0.9 with probability at least $1-\delta$:
\begin{align}
    \sqrt{N}&\geq c_1\log(\frac{c_2}{\delta})\\
    \sqrt{d}&\geq c_1\sqrt{N}\log(\frac{c_2}{\delta})\\
    \frac{1}{2\sqrt{N_\gT}}&\leq \min(r_\eta,r_\mu) \text{\hspace{2mm} which is guaranteed by \cref{eq:cond1}} \\
    r_\mu\ &\leq\sqrt{N\cdot N_{\gT,0}} \text{\hspace{2mm} which is guaranteed by \cref{eq:cond2}}\\
    \frac{r_\eta}{r_\mu} &\leq \frac{1}{4}\left(\frac{1}{1+2r_\mu\sqrt{N_{\gT,1}}}-\frac{4}{r_\mu\sqrt{N}}\right)\\
    \frac{3}{r_\eta}+\frac{3-5\tau}{r_\mu}&\leq \frac{27}{80}\sqrt{N_\gT}-2
\end{align}

Then putting this together, for parameters that satisfy the \cref{eq:cond4} we have,
\begin{align*}
    \frac{r_\eta}{r_\mu}&\leq\frac{1}{8}\left(1+\sqrt{N_{\gT}\cdot N_{\gT,1}}\right)^{-1} \text{ \hspace*{4mm} from \cref{eq:cond3}}\\
    &\leq\frac{1}{8}(1+2r_\mu \sqrt{N_{\gT,1}})^{-1}  \text{ \hspace*{4mm}  from $N_{\gT,0}\leq N_\gT$ and \cref{eq:cond2}}\\
    &\leq\frac{1}{4}\left(\frac{1}{1+2r_\mu\sqrt{N_{\gT,1}}} - \frac{4}{r_\mu\sqrt{N}}\right) \text{ \hspace*{4mm}  from \cref{eq:r_mu_sqrt_n}}\\
\end{align*}
which satisfies the constraint in \cref{proposition:DD_bad}. By \cref{eq:nn_t1}, the simpler condition $\frac{r_\eta}{r_\mu}\leq\frac{1}{8}(N\cdot N_{\gT,1})^{-1/2}$ is also sufficient.

The remaining bounds $\frac{1}{2\sqrt{N_\gT}}\leq r_\mu\leq\sqrt{N\cdot N_{\gT,0}}$ are guaranteed by \cref{eq:cond1} and \cref{eq:cond2} respectively.
\begin{align*}
    \frac{3}{r_\eta} + \frac{3-5\tau}{r_\mu}&\leq\frac{3}{8}(1-\tau)\sqrt{N_\gT}-2\\
    \frac{3}{r_\eta} + \frac{2.5}{r_\mu}&\leq\frac{27}{80}\sqrt{N_\gT}-2 \text{\hspace*{4mm} using $\tau=0.1$}\\
    3 + \frac{2.5r_\eta}{r_\mu}&\leq\frac{27}{80}r_\eta(\sqrt{N_\gT}-2) \\
    \frac{r_\eta}{r_\mu}\leq \frac{10}{25}\left(\frac{27}{80}r_\eta(\sqrt{N_\gT}-2) - 3\right)
\end{align*}
By \cref{eq:cond2} and \cref{eq:cond3}, $r_\eta\leq\frac{r_\mu}{8(1+\sqrt{N_{\gT}N_{\gT,1}})}\leq\frac{\sqrt{N_{\gT,0}}}{16\sqrt{N_{\gT}N_{\gT,1}}}\leq\frac{1}{16\sqrt{N_{\gT,1}}}\leq\frac{1}{16}$, which with \cref{eq:cond1} gives $N_\gT\geq 65536$ and hence $r_\eta(\sqrt{N_\gT}-2)\geq16\left(1-\frac{2}{\sqrt{N_\gT}}\right)>15$. Since \cref{eq:cond3} gives $\frac{r_\eta}{r_\mu}\leq\frac{1}{8}$,
\begin{align*}
    \frac{r_\eta}{r_\mu}&\leq \frac{1}{8}\\
    &\leq \frac{10}{25}\left(\frac{27}{80}\cdot 15 - 3\right)\\
    &\leq \frac{10}{25}\left(\frac{27}{80}r_\eta(\sqrt{N_\gT}-2) - 3 \right)
\end{align*}
and the constraint above is satisfied.
\end{proof}

\begin{theorem}\label{thm:th1_app2}
(Restating \cref{thm:th1})\\
Consider a Linear-Gaussian PU-learning problem with parameters $\mub, \etab, d > 300$, dataset sizes $N_{\gT} > 10$, $N = N_{\gS} + N_{\gT}$, $\alpha \in (0, \frac{N}{1024 N_\gT})$ 
and let $\delta \in (0,1)$.
For all problems where
\begin{align}
\tfrac{16}{\sqrt{N_\gT}} &\le r_{\eta} \le \tfrac{3}{8}\left(\sqrt{N_{\gT,1}}\right)^{-1},\\
\tfrac{1}{\sqrt{N_{\gT,1}}} &\le r_{\mu} \le \tfrac{1}{2}\sqrt{N_{\gT,0}},\\
c_1\log\!\left(\tfrac{c_2}{\delta}\right) &\le \min\!\left(\sqrt{N}, \sqrt{\frac{d}{N}}\right)
\end{align}
it holds with probability at least $1-\delta$ that $\mathrm{AU\!-\!ROC}(\rvw_{\mathrm{DD}}) < 0.5$ and $\mathrm{AU\!-\!ROC}(\rvw_{\mathrm{color}}) > 0.9$.
\end{theorem}

\begin{theorem}\label{thm:th1_app3}
(Restating \cref{thm:th1})\\
Consider a Linear-Gaussian PU-learning problem with parameters $\mub, \etab, d > 300$, dataset sizes $N_{\gT} > 10$, $N = N_{\gS} + N_{\gT}$, $\alpha \in (0, \frac{N}{1024 N_\gT})$ 
and let $\delta \in (0,1)$.
For all problems where
\begin{align}
\tfrac{16}{\sqrt{N_\gT}} &\le \min{\left(r_{\eta},r_{\mu}\right)},\\
r_{\mu} &\le \tfrac{1}{2}\tfrac{\sqrt{N_{\gT,0}}}{N_{\gT,1}},\\
\tfrac{r_\eta}{r_{\mu}} &\le \sqrt{\frac{N_{\gT,1}}{N_{\gT,0}}}\\
c_1\log\!\left(\tfrac{c_2}{\delta}\right) &\le \min\!\left(\sqrt{N}, \sqrt{\frac{d}{N}}\right)
\end{align}
it holds with probability at least $1-\delta$ that $\mathrm{AU\!-\!ROC}(\rvw_{\mathrm{DD}}) < 0.5$ and $\mathrm{AU\!-\!ROC}(\rvw_{\mathrm{color}}) > 0.9$.
\end{theorem}

\subsection{Concentration statements}
If $\xi_i\sim{\cN (0, \sigma^2\mathbf{I}_{d})}$ for all $i\in{[N]}$, then
\begin{itemize}
    \item Bound on norm of vector \citep[Eq.~3.5]{ledoux2013probability}
    \begin{align} \label{eq:gauss_norm_bound}
    P(\|\xi_i\| \leq t) > 1-4\exp\left\{ -\frac{t^2}{8d\sigma^2} \right\}.
    \end{align}
    \item For $U\subseteq [N]$ consider $\{\xi_j\}_{j\in{U}}$ and $\xi_i, i\notin{U}$,
    \begin{align} \label{eq:noise_inner_prod_bound1}
    P\left( \Big| \langle \xi_i, \frac{1}{|U|} \sum_{j\in{U}} \xi_j \rangle \Big| < |U|^{-1/2}t \right) > 1-2\exp\left[ -c \min\{ \frac{t^2}{\sigma^4 d}, \frac{t}{4\sigma^2} \}  \right]
    \end{align}
    This is a derivation from Bernstein's inequality, see e.g. Lemma~3 in \citet{puli2023don}.
    \item For $i\in{U}$ we have,
    \begin{align} \label{eq:noise_chi_squared_gaussian}
    P\left(\| \xi_i \|^2 > \sigma^2(d+ 2t\sqrt{d})\right) &\leq \exp\left(-\frac{t^2}{4}\right)
    \end{align}
    \item If $Z$ is a chi-square variable with d-degrees of freedom, then it holds for any $v>0$
    \begin{align*}
    P(Z \leq d - 2\sqrt{d v}) \leq \exp{\left(-v\right)}
    \end{align*}
    It follows that
    \begin{align*}
    P(Z \geq d - 2\sqrt{d v}) \geq 1 - \exp{\left(-v\right)}
    \end{align*}
    Hence, the lower bound on $\|\xi_i\|^2 \geq \sigma^2(d-2\sqrt{dv})$ w.p. at least $1 - \exp(-v)$.\\
    Upon setting $t=\sqrt{v}$, we get 
    \begin{align}\label{eq:laurentmassart}
        P\left(\|\xi_i\|^2 \geq \sigma^2(d-2t\sqrt{d})\right) \geq 1 - \exp\left(-ct^2\right)
    \end{align}
    \item Let $u\in{[d]}$, then
    \[
    \| \xi_{i,u} \cdot \sum_{j\in{U\setminus i}} \xi_{j, u} \|_{\psi_1} \leq \|\xi_{i,u}\|_{\psi_2} \|\sum_{j\in{U\setminus i}}\xi_{j,u}\|_{\psi_2} \leq \sqrt{|U|} \sigma^2.
    \]
    Recall Bernstein \citep{vershynin2018high} says that for independent sub-exponential variables $X_1,\ldots, X_N$, for every $t>0$,
    \begin{align*}
    P(\left| \sum_{i=1}^{N} X_i \right| \geq t) \leq 2\exp\left[ -c \min\left\{ \frac{t^2}{\sum_{i=1}^{N}{\|X_i\|^2_{\psi_1}}}, \frac{t}{\max_i \| X_i \|_{\psi_1}} \right\} \right]
    \end{align*}
    Applying this inequality to our case (where $N=d$ and the $X_i$s are the products of $\xi$ variables on the LHS above),
    \[
    P(|\sum_{j\in{U\setminus i}} \xi_i^\top \xi_{j} | \geq t) \leq 2\exp\left[ -c \min \left\{ \frac{t^2}{d |U| \sigma^4}, \frac{t}{\sqrt{|U|}\sigma^2} \right\} \right].
    \]
    Now \textbf{assume that $t < d\sqrt{|U|}\sigma^2$}, then
    \begin{align*}
    P(|\sum_{j\in{U\setminus i}} \xi_i^\top \xi_{j} | \geq t) &\leq 2\exp\left[ -c \min \left\{ \frac{t^2}{d |U| \sigma^4}, \frac{t}{\sqrt{|U|}\sigma^2} \right\} \right] \\
    &= 2\exp\left[ -c \frac{t^2}{d|U|\sigma^4} \right].
    \end{align*}
    Replacing variables $t$ with $t \sqrt{d|U|\sigma^4}$, and divding both sides of the inequality by $|U|$, we end up with
    \begin{align}\label{eq:bernstein}
    P\left(\Big| |U|^{-1}\sum_{j\in{U\setminus i}} \xi_i^\top \xi_{j} \Big| \geq     \sigma^2 t \sqrt{d |U|^{-1}} \right) \leq 2\exp\left[ -ct^2 \right]
    \end{align}
    \item We know that for $ \Big| \langle \xi_i, \frac{1}{|U|}\sum_{j\in{U}} \xi_j \rangle \Big|  = \frac{1}{|U|} \|\xi_i\|^2 + \frac{1}{|U|} \Big| \langle \xi_i, \frac{1}{|U|}\sum_{j\in{U}, i\neq j} \xi_j \rangle \Big|$
    \begin{align*}
    \langle \xi_i^\top \bar{\xi}_{U} \rangle = |U|^{-1}\|\xi_i\|^2 + \langle\xi_i, |U|^{-1}\sum_{j\in U, j\neq i} \xi_j\rangle\\
    \left|\langle \xi_i^\top \bar{\xi}_{U} \rangle \right| \geq \left||U|^{-1}\|\xi_i\|^2 \right| - \left| \langle\xi_i, |U|^{-1}\sum_{j\in U, j\neq i} \xi_j\rangle \right|
    \end{align*}
    Taking union bound of inequalities \ref{eq:laurentmassart} and \ref{eq:bernstein}, we get
    \begin{align*}
        P\left(\left| \langle \xi_i^\top \bar{\xi}_{U} \rangle \right| \geq |U|^{-1}\sigma^2(d-2t\sqrt{d}) - |U|^{-1}\sigma^2t\sqrt{d(|U|-1)} \right) \geq 1 - 3\exp(-ct^2)\\
        P\left(\left| \langle \xi_i^\top \bar{\xi}_{U} \rangle \right| \geq |U|^{-1}\sigma^2\left(d-2t\sqrt{d} - t\sqrt{d(|U|-1)}\right) \right) \geq 1 - 3\exp(-ct^2)\\
    \end{align*}
    If we assume $t\leq\frac{\sqrt{d}}{2(2 + \sqrt{|U|-1})}$, we get
    \begin{align}\label{eq:i_in_U_innerproduct}
    P\left(\left| \langle \xi_i^\top \bar{\xi}_{U} \rangle \right| \geq \frac{\sigma^2d}{2|U|} \right) \geq 1 - 3\exp(-ct^2)
    \end{align}
\end{itemize}

%% file: sections/appendix.tex
 \subsection{Proof of Lemma \ref{lem:impossibility_lemma}} \label{sec:proof} 
Let us recall the strong positivity assumption stated in the main paper, which appears in \citet{garg22adaptation}.
\begin{assumption}[Strong positivity]
There exists $X_{sep}\subseteq \gX$ such that $\Plabel{1}(X_{sep}) = 0$ and the matrix $[\Psource(\rvx \mid y)]_{\rvx\in{X_{sep}m y\in{[k]}}}$ is full rank and diagonal.
\end{assumption}
We restate and prove the claim that \problem{} is not learnable under this assumption, once the label shift assumption is removed.
\begin{lemma*}
Let $\gA$ be an algorithm for \problem. There are distributions $\Psource, \Plabel{[k]}$ and $\Plabel{k+1}$ such that the problem satisfies strong positivity, and $\exists h^*\in{\gH}$ for which $R^{l_{01}}_{\gT}(h^*)=0$, while $\E_{S_{\gS}, S_{\gT}}\left[ R^{l_{01}}_{\gT}(\gA(S_{\gS}, S_{\gT})) \right] \geq 0.5$.
\label{lem:impossibility_under_strong_positivity}
\end{lemma*}
\begin{proof}
Define the following distributions over $4$ states
\begin{align*}
[\Psource(x \mid y)]_{x\in{\gX}, y\in{[k]}} = \begin{bmatrix}
1-\varepsilon & 0 & 0 & \varepsilon \\
0 & 1- \varepsilon & \varepsilon & 0
\end{bmatrix}, \\
\Psource(Y) = [\frac{1-2\varepsilon}{1-\varepsilon}, \frac{\varepsilon}{1-\varepsilon}, 0]
\end{align*} for some $\varepsilon > 0$, and two other distribution over $\gX$, $Q(X) = [0, 0, 1, 0]$ and $D(X) = [0, 0, 0, 1]$. Consider $2$ \problem{} problems where $k=2$ and $\alpha=0.5$:
\begin{itemize}
    \item One where $\Plabel{[k]} = Q$ and $\Plabel{k+1} = D$, which means that $\Plabel{[k]}(X\mid Y=y) = [0, 0, 1, 0]$ and we set $\Plabel{[k]}(Y) = [\frac{1-2\varepsilon}{1-\varepsilon}, \frac{\varepsilon}{1-\varepsilon}, 0]$ although we can set it to any arbitrary distribution.
    \item For the second problem $\Plabel{[k]} = D$ and $\Plabel{k+1} = Q$, which entails similarly to the first case that $\Plabel{[k]}(X \mid Y=y) = [0, 0, 0, 1]$ for $y\in{[k]}$, while we keep $\Psource$ and the rest of the details as they are in the first problem.
\end{itemize}
It is clear that under the hypothesis class $\gH$ of all binary classifiers on $\gX$, it holds that $R^{l_{01}}_{\gT}(h^*)=0$. Now we will show that both problems satisfy strong positivity (note that they also satisfy that $\mathrm{Supp}(\Plabel{[k]}) \subseteq \mathrm{Supp}(\Psource)$), and also $\Psource(X, Y)$ and $\Ptarget (X)$ are the same for both problems.

Once this is shown, we can conclude our result, since any observed dataset that is an input to $\gA$ is equally likely in both problems. However, any hypothesis $h$ that achieves $R^{l_{01}}_{\gT}(h)=\delta$ on the first problem, achieves risk $1-\delta$ on the other problem since $\Plabel{[k]}$ and $\Plabel{k+1}$ switch roles between the two problems.

To show that the problems satisfy strong positivity, consider $X_{sep}$ as the first and second states. We have that for both problem $\Plabel{k+1}(X_{sep} = 0$ since both $D(X_{sep}) = 0$ and $Q(X_{sep})=0$, while
\begin{align*}
    [\Psource(x \mid y)]_{x\in{X_{sep}}, y\in{[k]}} = \begin{bmatrix}
    1-\varepsilon & 0 \\
    0 & 1-\varepsilon
    \end{bmatrix},
\end{align*}
which is a full rank and diagonal matrix. Hence the strong positivity condition is satisfied. 
We defined the same $\Psource(X, Y)$ for both problems, so it is left to show that $\Ptarget(X)$ also equals for them.
This is also straightforward as for both problem $\Ptarget(X) = 0.5\cdot Q + 0.5\cdot D$, which concludes the proof.
\end{proof}

\subsection{Additional details on experimental setting}\label{sec:add_setup}
\subsubsection{AUROC vs AUPRC scores for novel class detection }\label{sec:auroc_vs_auprc}
It is a common argument in the machine learning literature that AUPRC scores are more suitable for evaluating
methodologies in class-imbalanced scenarios. However, this stance is nuanced, as some research, such as \cite{mcdermott2024closer} suggests favoring AUROC over AUPRC in certain imbalanced conditions. \cite{mcdermott2024closer} further notes that
AUPRC inherently emphasizes the performance on samples with higher scores. Given our experimental focus on assessing
models’ ability to assign higher scores for novelty detection, AUPRC emerges as the most relevant metric for our analysis, especially when the proportions of positive class (novel class in our case) is very low.
\subsubsection{Discussion}\label{sec:app_discussion}
Building on the results and observations from the previous section, we proceed to further analyze and understand the aspects of OSDA under conditions of background shift. The curve plots in \ref{fig:sun397_curve_plot} compare the novelty detection performance of the methods shown , it is evident that \ours{} outperforms other methods particularly when the novel class ratio $\alpha$ is less than $0.2$. However, as $\alpha$ is increased beyond $0.3$, the other baselines rapidly catch up to the AUPRC performance of \ours{}.    
From Table \ref{table:sun397_target_acc} we observe that using constrained learning to acquire shared representations benefits the classification performance on known classes. The source-only method mentioned in \ref{table:sun397_target_acc} serves as a baseline, trained exclusively on the 
$\datasource$ and evaluated on the $\datatarget$ without employing any strategies to mitigate shift effects. Furthermore, Table \ref{table:sun397_vitl14_w_vs_wo_shift} provides insights on the impact of distribution shift on the overall OSDA performance of all the methods using \vit{} visual encoder pretraiend using CLIP.  

We specify our empirical test of measuring separability in \ref{sec:app_dataset}. Based on this test of separability, we observe that novel classes in Amazon Reviews dataset are not perfectly separable and hence violate the assumption of separability to some extent. However, we see in our result in tables \ref{table:results_sumary}, \ref{table:amazon_reviews_w_shift} and \ref{table:abs_amazon_reviews_w_shift_auprc} that \ours{} still outperforms other baselines. Hence, we realize that \ours{} is robust to certain violations of the separability assumption.
\subsubsection{Implementation of \ours{}}

It is important to note that each model head $h_{\hat{\alpha}}$ independently solve constrained problem in \ref{eq:eq04lagrangian} using primal-dual optimization while remaining $k$ heads focus on classifying samples (from $\datasource$) from $k$ known categories using ground truths $\mathcal{Y}_\mathcal{S}$.  

\subsection{Combining \ours{} with Domain Adaptation methods}\label{sec:app_color_with_da}
\osr{} performance depends on the robustness of the closed-set classifier. If the classifier were more robust to the specific shift, OSDA performance would improve. Since CoLOR operates on top of the closed-set classifier, any method that improves robustness to a specific shift (e.g., domain adaptation, shift-robust training) could be used together with \ours{}. \ours{} would ensure that the domain shift does not impact the performance of novelty detection while the robust closed-set classifier would aid in improving the performance over known classes. This flexibility enables \ours{} to be used with any architecture and make them robust against distribution shift during test time.  

\subsection{Extending constrained learning objective to OSDA}\label{sec:app_color}
In practice, constraints such as $\beta(h) \leq \beta$ are enforced using a differentiable approximation of a step function, e.g. via a sigmoid $\sum_{\rvx\in{\datasource}} \sigma(h(\rvx)) \leq \beta$, and objectives are optimized using the logistic loss.
We either train an entire model (encoder + classifier) from scratch or just add two fully-connected (FC) layers on top of a pretrained encoder and only train these additional FC layers. The first FC layer provides the shared representation acquired through learning from the related tasks while second FC layer uses this shared representation to classify the known classes and detect novel identities.
\subsubsection{Dataset}\label{sec:app_dataset}
One important factor to consider while creating shifts is that $\datasource$ and $S_{\mathcal{T},k+1}$ should be distinguishable. As dataset separability is a difficult quantity to measure we train a classifier (oracle) for each dataset to distinguish novel groups from samples belonging to known categories. We then use the learnability of the oracle as a criterion to ensure the separability of novel classes. This means that we calculate the AUROC and AUPRC scores of the oracle for the task of supervised novelty detection. Higher AUROC and AUPRC would correspond to higher separability. We consider an AUROC and AUPRC higher than 0.98 as ideal to ensure separability between novel class identities and known classes. It is difficult to ensure such high separability for all the datasets, particularly for Amazon Reviews dataset which does not perfectly satisfy the separability. Yet we observe that \ours{} is robust to background shift in such settings outperforming all the baselines as observed in Tables \ref{table:results_sumary}, \ref{table:amazon_reviews_w_shift} and 
\ref{table:abs_amazon_reviews_w_shift_auprc}. \\
We conduct 5 repetitions of an experiment for each dataset and for every identity of the novel class. Each repetition uses a unique random seed, representing a distinct background shift setting. These settings are generated by randomly varying the subtype proportions of the known categories within the SUN397 dataset.

\textbf{SUN397:} It consists of images of scenes/places from various locations. The dataset is provided with 3 levels of hierarchy where Level-1 is grouped as indoor, natural outdoor and man-made outdoor scenes. We use indoor classes as in-distribution classes while we choose novel classes from natural outdoor scenes. Level-2 hierarchy has shopping, workplace, homes/hotels, etc. under indoor category while outdoor natural contains classes like water/ice/snow, mountains/hills desert/sky, forest, etc. Each of these level-2 categories have subcategories (level-3 classes) that form these level-2 groupings. We randomly select 8 level-3 subtypes (like bakery shop or banquet hall) per level-2 category (shopping/dining places) and vary the subtype proportions to create background shift between source and target. Furthermore, novel classes are randomly selected from the level-2 categories of outdoor natural group. \\      
\textbf{CIFAR100:} It consists of 60,000 32x32 colour images in 20 primary classes (superclasses) each of which have 5 subcategories composing a total of 100 classes \cite{krizhevsky2009learning}. The training set has 50,000 images (i.e. 500 images per subcategory and 2500 images per primary class) while the test set has 10000 images (i.e. 100 images per subcategory and 500 images per primary class). We retain these splits for our experiments. We use 4 primary categories as aquatic mammals, flowers, fishes and birds. The subcategories of the aquatic mammals are beaver, dolphin, otter, seal are whale while that of flowers are orchid, poppy, rose, sunflower and tulip. Similarly fishes and birds have 5 subcategories each. We vary the marginal distribution of these subcategories to create a subpopulation shift leading to a background shift between source and target data while maintaining no label shift w.r.t. primary categories i.e. aquatic mammals and flowers. The novel category is randomly selected from the remaining unseen categories.\\
\textbf{Amazon Reviews} The dataset is heavily skewed with respect to sentiments and product categories. Hence, we select 6 product categories having similar orders of magnitude of the number of reviews namely 'Digital Music', 'Industrial \& Scientific', 'Luxury Beauty', 'Musical Intstruments', 'Prime Pantry' and 'Software'. To prevent further skewness in the dataset due to sentiments, we restrict the sample size per sentiment per category to 500 reviews in the training set and 125 reviews per sentiment per product category. We induce a background shift based on sentiments. Reviews with rating strictly below 3.0 (out of 5.0) are considered negative sentiments and those with rating strictly above 3.0 are considered positive sentiments whereas reviews having a rating of exactly 3.0 are discarded from the dataset. The minimum rating is 1.0 while the maximum is 5.0. 
\subsubsection{Hyper-Parameters and Training}\label{sec:training}
We consistently set the FPR threshold $\beta = 0.01$ without optimizing it at all based on validation dataset. For each novel class identity, we repeat the experiments for 5 different randomly generated splits between $\Psource$ and $\Ptarget$ adhering to the definition and assumptions of background shift. For CIFAR100, we use ResNet18 backbone followed by a linear layer for classification and train the whole model from scratch for all the methods. For Amazon Reviews dataset, we use pretrained RoBERTa features followed by 2 linear layers for classification. For Amazon Reviews, due to computational limitations we resort to linear probing rather than finetuning the whole model and only finetune the last two linear layers. The primary hyperparameters we tune for stable convergence are learning rate, L2 weight penalty scalar to prevent overfitting, logit multiplier values that act as temperature controllers for softmax/sigmoid scores and gradient clipping to avoid exploding gradients. These hyper parameters are tuned based on a sample training and validation sets but are kept constant throughout the dataset and baseline across different seed values and novelty cases. Furthermore, the methods {\oursold} \& {\ours} require additional hyperparameters like dual learning rate and lagrange multipliers. Each output node is associated with Lagrange multipliers, which address a distinct primal-dual optimization problem owing to diverse target recall constraints. These multipliers are initialized to 1.0, while the dual learning rate is meticulously calibrated for CIFAR100 \& Amazon Reviews datasets individually to ensure stable learning dynamics conducive to minimizing both the objective surrogate loss function for FPR  and recall inequality constraints. Note that for \zoc{}, We used cosine similarity between image embeddings and the known class text embeddings to obtain the closed-set class predictions. Table \ref{table:hyperparams} displays all the hyperparameters used for each of the baselines and datasets.

\subsubsection{Impact of search grid range and density on the performance}\label{sec:alpha_impact_on_performance}
For all the methods, we keep the search grid of candidate target recall values consistent through all the experiments i.e. $\boldsymbol{\alpha} = [0.02, 0.05, 0.10, 0.15, 0.20, 0.25, 0.30, 0.35, 0.40, 0.45]$.\\
\textbf{Choice of the search grid range:}\\
$\boldsymbol{\alpha} = [0.02, 0.45]$ is not claimed to be universal, but chosen to study the most informative regime. We observed that when the size of the novel class (true size, not the one estimated by the model) is below $1\%$, all baselines failed. Accordingly, we set the lower bound of our search grid on that scale, but we show in experiments that including lower values does not change performance by much.
As for the upper bound, we observe that our method is relatively robust when the true novel class size is small (i.e. the model wouldn’t choose a larger alpha if it had the possibility). We cap the upper bound of the gris at $0.45$ because when the novel class constitutes a large fraction of the target data, novelty detection becomes relatively easier. In such regimes, even simple inspection of random data points or sampling strategies can easily detect novel instances, enabling labeling and the use of alternative solutions. Our primary focus is therefore on the more challenging and practically relevant regime where novel classes are relatively rare (typically between 0.02 and 0.2).
We further confirm this through our ablation study in \cref{tab:alpha_range_ablation}. We note that this lower bound is not fundamental and may depend on factors such as dataset characteristics, feature representations, and model architecture. In practice, we recommend determining an appropriate $\boldsymbol{\alpha}$ range via simulation: known classes can be temporarily treated as unknown, and their detectability can be evaluated under varying assumed novel class ratios. This procedure provides empirical guidance for both the bottom value and density of the grid range.\\
\textbf{Choice of search grid density:}
We added experiments on the SUN397 dataset using a ResNet50 backbone pretrained on ImageNet-1K (Table 19 and Figure 5 in the revised draft). These results show that CoLOR is performant for a fixed $\boldsymbol{\alpha}$ range with only minor performance variations across varying grid densities. This indicates that the method does not require finely tuned grids to achieve strong performance, and that the chosen density is sufficient to cover the relevant operating regime.  
\begin{table*}[h!]
\centering
\caption{Hyperparameters: lr = learning rate, dlr = dual learning rate (\ours{}), L2 penalty = L2 weight penalty scaler, lm = logit multiplier, clip = gradient clipping value. The two values separated by "$/$" in learning rate column of SUN397 dataset correspond to the linear probing of ResNet50 (pretrained on ImageNet) and ViT (pretrained using CLIP) backbones respectively. }
\label{table:hyperparams}
\resizebox{\textwidth}{!}{%

\begin{tabular}{c || c | c | c | c | c || c | c | c | c | c || c | c | c | c | c } 
 \hline
 Method & \multicolumn{5}{c||}{CIFAR100} & \multicolumn{5}{c||}{Amazon Reviews} & \multicolumn{5}{c}{SUN397} \\
 \cline{2-16}
 & lr & dlr & L2 penalty & lm & clip & lr & dlr & L2 penalty & lm & clip & lr & dlr & L2 penalty & lm & clip \\
 \hline\hline
 \DDnew{} & $1e-2$ & $-$ & $3e-5$ & $1.2$ & $5.0$ & $1e-2$ & $-$ & $1e-4$ & $1.0$ & $1.0$ & $1e-2/1e-1$ & $-$ & $3e-5$ & $1.2$ & $5.0$ \\ 
 \uPUnew{} & $1e-3$ & $-$ & $3e-7$ & $1.2$ & $5.0$ & $1e-3$ & $-$ & $1e-4$ & $1.0$ & $1.0$ & $1e-3/1e-1$ & $-$ & $3e-7$ & $1.2$ & $5.0$ \\
 \nnPUnew{} & $1e-3$ & $-$ & $3e-7$ & $1.2$ & $5.0$ & $1e-3$ & $-$ & $1e-4$ & $1.0$ & $1.0$ & $1e-3/1e-1$ & $-$ & $3e-7$ & $1.2$ & $5.0$ \\
 \BODA{} & $1e-3$ & $-$ & $3e-3$ & $1.2$ & $5.0$ & $1e-3$ & $-$ & $1e-4$ & $1.0$ & $1.0$ & $1e-3$ & $-$ & $3e-3$ & $1.2$ & $5.0$  \\
 \arpl{} & $1e-2$ & $-$ & $3e-5$ & $1.0$ & $100.0$ & $-$ & $-$ & $-$ & $-$ & $-$ & $1e-2$ & $-$ & $3e-5$ & $1.0$ & $100.0$  \\
 \pulse{} & $1e-3$ & $-$ & $3e-5$ & $1.2$ & $5.0$ & $1e-3$ & $-$ & $1e-4$ & $1.0$ & $1.0$ & $1e-3$ & $-$ & $3e-5$ & $1.2$ & $5.0$ \\
 \ours{} & $1e-3$ & $2e-2$ & $3e-7$ & $1.2$ & $5.0$ & $1e-3$ & $6e-2$ & $1e-4$ & $1.0$ & $1.0$ & $1e-3$ & $2e-2$ & $3e-7$ & $1.2$ & $5.0$ \\ [1ex] 
 \hline
\end{tabular}
}
\end{table*}

\subsection{Performance comparison based on average relative \& absolute AU-ROC and AU-PRC scores}\label{sec:auroc}
Refer to tables \ref{table:cifar100_w_shift}, \ref{table:amazon_reviews_w_shift}, \ref{table:sun397_rn50_w_shift}, \ref{table:sun397_vitl14_w_shift}, \ref{table:sun397_vitl14_wo_shift}, \ref{table:sun397_vitl14_w_vs_wo_shift}, \ref{table:sun397_target_acc} below.

\begin{table*}[h!]
\centering
\caption{SUN397 dataset with distribution shift due to varying proportions of subtypes of scenes/places. All the methods here use ResNet50 backbone pretrained on ImageNet1K\_V1 \cite{Russakovsky2015imagenet}. \auroc{} and \auprc{} represent the performance for the novel category detection task while \oscr{} measures overall performance of the methods on both known and unknown classes. $\alpha$ is the mixture proportion column for the respective novel classes.}
\label{table:sun397_rn50_w_shift}
\resizebox{\textwidth}{!}{%

\begin{tabular}{c || c || c | c | c || c }
\hline
Metric & Method & \multicolumn{3}{c||}{Novel Classes (natural outdoor scenes/places)} &  \\
\cline{3-5}
 & & $\alpha$ $=0.06\pm0.01$ & $\alpha$ $=0.06\pm0.01$ & $\alpha$ $=0.08\pm0.04$ & \\
 \cline{3-5}
 & & [water, ice, snow, etc.] & [mountains, hills, desert, sky, etc.] & [forest, field, jungle, etc.] & Summary \\
\hline\hline
\multirow{5}{*}{\auroc{}} & \DDnew{} & $0.90\pm0.06$ & $0.88\pm0.05$ &                          $0.94\pm0.02$ & $0.91\pm0.05$ \\
                       & \uPUnew{} & $0.71\pm0.13$ & $0.72\pm0.18$ & $0.84\pm0.08$ & $0.76\pm0.14$ \\
                       & \nnPUnew{} & $0.71\pm0.13$ & $0.72\pm0.18$ & $0.84\pm0.08$ & $0.76\pm0.14$ \\
                       & \BODA{} & $0.86\pm0.03$ & $0.90\pm0.01$ & $0.82\pm0.09$ & $0.86\pm0.06$ \\
                       & \shot{} & $0.77\pm0.08$ & $0.64\pm0.21$ & $0.70\pm0.18$ & $0.71\pm0.16$ \\
                       & \arpl{} & $0.73\pm0.03$ & $0.72\pm0.08$ & $0.68\pm0.12$ & $0.71\pm0.08$ \\
                       & \anna{} & $0.92\pm0.03$ & $0.96\pm0.02$ & $0.92\pm0.08$ & $0.93\pm0.05$ \\
                       & \cac{} & $0.79\pm0.05$ & $0.82\pm0.05$ & $0.78\pm0.04$ & $0.80\pm0.05$ \\
                       & \pulse{} & $0.75\pm0.05$ & $0.76\pm0.05$ & $0.69\pm0.10$ & $0.73\pm0.07$ \\
                       & \ours{} & $\boldsymbol{0.98\pm0.02}$ & $\boldsymbol{0.98\pm0.01}$ & $\boldsymbol{0.98\pm0.02}$ & $\boldsymbol{0.98\pm0.02}$ \\
\hline
\multirow{5}{*}{\auprc{}} & \DDnew{} & $0.50\pm0.20$ & $0.40\pm0.23$ &                          $0.73\pm0.10$ & $0.54\pm0.22$ \\
                       & \uPUnew{} & $0.11\pm0.04$ & $0.18\pm0.20$ & $0.34\pm0.30$ & $0.21\pm0.22$ \\
                       & \nnPUnew{} & $0.11\pm0.04$ & $0.18\pm0.20$ & $0.34\pm0.30$ & $0.21\pm0.22$ \\
                       & \BODA{} & $0.41\pm0.04$ & $0.52\pm0.06$ & $0.43\pm0.16$ & $0.45\pm0.11$ \\
                       & \shot{} & $0.19\pm0.08$ & $0.17\pm0.19$ & $0.21\pm0.14$ & $0.19\pm0.13$ \\
                       & \arpl{} & $0.11\pm0.02$ & $0.11\pm0.04$ & $0.15\pm0.11$ & $0.12\pm0.07$ \\
                       & \anna{} & $0.71\pm0.13$ & $0.77\pm0.13$ & $0.72\pm0.23$ & $0.73\pm0.16$ \\
                       & \cac{} & $0.17\pm0.05$ & $0.17\pm0.06$ & $0.20\pm0.10$ & $0.18\pm0.07$ \\
                       & \pulse{} & $0.15\pm0.05$ & $0.14\pm0.06$ & $0.15\pm0.06$ & $0.15\pm0.05$ \\
                       & \ours{} & $\boldsymbol{0.92\pm0.04}$ & $\boldsymbol{0.93\pm0.05}$ & $\boldsymbol{0.89\pm0.15}$ & $\boldsymbol{0.91\pm0.09}$ \\
\hline
\multirow{5}{*}{\oscr{}} & \DDnew{} & $0.66\pm0.04$ & $0.67\pm0.06$ & $0.70\pm0.03$                         & $0.68\pm0.05$ \\
                       & \uPUnew{} & $0.35\pm0.10$ & $0.35\pm0.08$ & $0.48\pm0.11$ & $0.40\pm0.11$ \\
                       & \nnPUnew{} & $0.35\pm0.10$ & $0.35\pm0.08$ & $0.48\pm0.11$ & $0.40\pm0.11$ \\
                       & \BODA{} & $0.56\pm0.10$ & $0.59\pm0.09$ & $0.50\pm0.09$ & $0.55\pm0.10$ \\
                       & \shot{} & $0.25\pm0.05$ & $0.18\pm0.07$ & $0.24\pm0.08$ & $0.22\pm0.07$ \\
                       & \arpl{} & $0.61\pm0.01$ & $0.61\pm0.06$ & $0.58\pm0.08$ & $0.60\pm0.06$ \\
                       & \anna{} & $0.58\pm0.07$ & $0.62\pm0.05$ & $0.59\pm0.09$ & $0.60\pm0.07$ \\
                       & \cac{} & $0.68\pm0.05$ & $0.70\pm0.07$ & $0.66\pm0.03$ & $0.68\pm0.05$ \\
                       & \pulse{} & $0.66\pm0.05$ & $0.68\pm0.05$ & $0.62\pm0.09$ & $0.65\pm0.06$ \\
                       & \ours{} & $\boldsymbol{0.82\pm0.03}$ & $\boldsymbol{0.81\pm0.03}$ & $\boldsymbol{0.81\pm0.05}$ & $\boldsymbol{0.81\pm0.04}$ \\         
\end{tabular}

}
\end{table*}

\begin{table*}[h!]
\centering
\caption{SUN397 dataset without distribution shift due to varying proportions of subtypes of scenes/places. All the methods here use ResNet50 backbone pretrained on ImageNet1K\_V1 \cite{Russakovsky2015imagenet}. \auroc{} and \auprc{} represent the performance for the novel category detection task while \oscr{} measures overall performance of the methods on both known and unknown classes. $\alpha$ is the mixture proportion column for the respective novel classes.}
\label{table:sun397_rn50_wo_shift}
\resizebox{\textwidth}{!}{%

\begin{tabular}{c || c || c | c | c || c }
\hline
Metric & Method & \multicolumn{3}{c||}{Novel Classes (natural outdoor scenes/places)} &  \\
\cline{3-5}
 & & $\alpha$ $=0.06\pm0.01$ & $\alpha$ $=0.06\pm0.01$ & $\alpha$ $=0.08\pm0.04$ & \\
 \cline{3-5}
 & & [water, ice, snow, etc.] & [mountains, hills, desert, sky, etc.] & [forest, field, jungle, etc.] & Summary \\
\hline\hline
\multirow{5}{*}{\auroc{}} & \DDnew{} & $\boldsymbol{1.00\pm0.00}$ &                                              $\boldsymbol{1.00\pm0.00}$ & $\boldsymbol{1.00\pm0.00}$ &                                 $\boldsymbol{1.00\pm0.00}$ \\
                       & \uPUnew{} & $0.99\pm0.00$ & $0.99\pm0.01$ & $1.00\pm0.00$ & $0.99\pm0.01$ \\
                       & \nnPUnew{} & $0.99\pm0.00$ & $0.99\pm0.01$ & $1.00\pm0.00$ & $0.99\pm0.01$ \\
                       & \BODA{} & $0.86\pm0.06$ & $0.91\pm0.01$ & $0.85\pm0.05$ & $0.87\pm0.05$ \\
                       & \shot{} & $0.63\pm0.08$ & $0.58\pm0.12$ & $0.61\pm0.14$ & $0.61\pm0.11$ \\
                       & \arpl{} & $0.86\pm0.03$ & $0.81\pm0.03$ & $0.84\pm0.05$ & $0.84\pm0.04$ \\
                       & \anna{} & $0.95\pm0.05$ & $0.98\pm0.02$ & $0.92\pm0.11$ & $0.95\pm0.07$ \\
                       & \cac{} & $0.89\pm0.05$ & $0.89\pm0.01$ & $0.87\pm0.03$ & $0.88\pm0.03$ \\
                       & \pulse{} & $0.82\pm0.08$ & $0.84\pm0.02$ & $0.79\pm0.04$ & $0.82\pm0.05$ \\
                       & \ours{} & $\boldsymbol{1.00\pm0.00}$ & $\boldsymbol{1.00\pm0.00}$ & $\boldsymbol{1.00\pm0.00}$ & $\boldsymbol{1.00\pm0.00}$ \\
\hline
\multirow{5}{*}{\auprc{}} & \DDnew{} & $\boldsymbol{1.00\pm0.00}$ &                                              $\boldsymbol{1.00\pm0.00}$ & $\boldsymbol{0.99\pm0.01}$ &                                 $\boldsymbol{1.00\pm0.00}$ \\
                       & \uPUnew{} & $0.95\pm0.05$ & $0.90\pm0.13$ & $0.98\pm0.02$ & $0.94\pm0.09$ \\
                       & \nnPUnew{} & $0.95\pm0.05$ & $0.90\pm0.13$ & $0.98\pm0.02$ & $0.94\pm0.09$ \\
                       & \BODA{} & $0.40\pm0.11$ & $0.52\pm0.02$ & $0.40\pm0.06$ & $0.44\pm0.09$ \\
                       & \shot{} & $0.09\pm0.02$ & $0.09\pm0.06$ & $0.11\pm0.06$ & $0.10\pm0.05$ \\
                       & \arpl{} & $0.20\pm0.04$ & $0.16\pm0.04$ & $0.25\pm0.09$ & $0.20\pm0.07$ \\
                       & \anna{} & $0.90\pm0.05$ & $0.94\pm0.03$ & $0.86\pm0.13$ & $0.90\pm0.08$ \\
                       & \cac{} & $0.29\pm0.08$ & $0.27\pm0.06$ & $0.30\pm0.08$ & $0.29\pm0.07$ \\
                       & \pulse{} & $0.21\pm0.07$ & $0.21\pm0.05$ & $0.19\pm0.03$ & $0.20\pm0.05$ \\
                       & \ours{} & $0.96\pm0.03$ & $\boldsymbol{1.00\pm0.00}$ & $\boldsymbol{1.00\pm0.00}$ & $0.99\pm0.02$ \\
\hline
\multirow{5}{*}{\oscr{}} & \DDnew{} & $0.86\pm0.00$ & $0.86\pm0.00$ & $0.86\pm0.00$                              & $0.86\pm0.00$ \\
                       & \uPUnew{} & $0.72\pm0.01$ & $0.72\pm0.01$ & $0.72\pm0.01$ & $0.72\pm0.01$ \\
                       & \nnPUnew{} & $0.72\pm0.01$ & $0.72\pm0.01$ & $0.72\pm0.01$ & $0.72\pm0.01$ \\
                       & \BODA{} & $0.71\pm0.03$ & $0.70\pm0.01$ & $0.61\pm0.05$ & $0.67\pm0.06$ \\
                       & \shot{} & $0.19\pm0.05$ & $0.16\pm0.04$ & $0.17\pm0.08$ & $0.18\pm0.06$ \\
                       & \arpl{} & $0.82\pm0.02$ & $0.78\pm0.03$ & $0.80\pm0.05$ & $0.80\pm0.04$ \\
                       & \anna{} & $0.82\pm0.05$ & $0.85\pm0.02$ & $0.79\pm0.10$ & $0.82\pm0.07$ \\
                       & \cac{} & $0.84\pm0.05$ & $0.85\pm0.01$ & $0.83\pm0.03$ & $0.84\pm0.03$ \\
                       & \pulse{} & $0.78\pm0.08$ & $0.81\pm0.02$ & $0.76\pm0.03$ & $0.78\pm0.05$ \\
                       & \ours{} & $\boldsymbol{0.92\pm0.03}$ & $\boldsymbol{0.93\pm0.01}$ & $\boldsymbol{0.92\pm0.02}$ & $\boldsymbol{0.92\pm0.02}$ \\         
\end{tabular}

}
\end{table*}

\begin{table*}[h!]
\centering
\caption{SUN397 dataset with distribution shift due to varying proportions of subtypes of scenes/places. All the principled methods (\DDnew{}, \uPU{}, \nnPU{}, \BODA{} \& \ours{}) here use pretrained CLIP \vit{} backbone from \cite{radford2021clip}. }
\label{table:sun397_vitl14_w_shift}
\resizebox{\textwidth}{!}{%

\begin{tabular}{c || c || c | c | c || c }
\hline
Metric & Method & \multicolumn{3}{c||}{Novel Classes (natural outdoor scenes/places)} &  \\
\cline{3-5}
 & & $\alpha$ $=0.06\pm0.01$ & $\alpha$ $=0.06\pm0.01$ & $\alpha$ $=0.08\pm0.04$ & \\
 \cline{3-5}
 & & [water, ice, snow, etc.] & [mountains, hills, desert, sky, etc.] & [forest, field, jungle, etc.] & Summary \\
\hline\hline
\multirow{5}{*}{\auroc{}} & \DDnew{} & $0.96\pm0.03$ & $0.96\pm0.03$ &                          $0.96\pm0.02$ & $0.96\pm0.02$ \\
                       & \uPUnew{} & $0.95\pm0.03$ & $0.94\pm0.05$ & $0.95\pm0.04$ & $0.95\pm0.04$ \\
                       & \nnPUnew{} & $0.95\pm0.03$ & $0.94\pm0.05$ & $0.95\pm0.04$ & $0.95\pm0.04$ \\
                       & \BODA{} & $0.91\pm0.06$ & $0.82\pm0.18$ & $0.79\pm0.19$ & $0.84\pm0.15$ \\
                       & \zoc{} & $0.84\pm0.02$ & $0.85\pm0.03$ & $0.78\pm0.12$ & $0.82\pm0.07$ \\
                       & \cac{} & $0.98\pm0.02$ & $0.99\pm0.01$ & $0.98\pm0.01$ & $0.98\pm0.01$ \\
                       & \pulse{} & $0.96\pm0.02$ & $0.96\pm0.01$ & $0.95\pm0.01$ & $0.96\pm0.01$ \\
                       & \ours{} & $\boldsymbol{0.99\pm0.02}$ & $\boldsymbol{0.99\pm0.01}$ & $\boldsymbol{0.99\pm0.01}$ & $\boldsymbol{0.99\pm0.01}$ \\
\hline
\multirow{5}{*}{\auprc{}} & \DDnew{} & $0.88\pm0.05$ & $0.85\pm0.12$ &                          $0.90\pm0.07$ & $0.87\pm0.08$ \\
                       & \uPUnew{} & $0.84\pm0.05$ & $0.78\pm0.19$ & $0.85\pm0.10$ & $0.82\pm0.12$ \\
                       & \nnPUnew{} & $0.84\pm0.05$ & $0.78\pm0.19$ & $0.85\pm0.10$ & $0.82\pm0.12$ \\
                       & \BODA{} & $0.39\pm0.23$ & $0.37\pm0.39$ & $0.38\pm0.30$ & $0.38\pm0.29$ \\
                       & \zoc{} & $0.21\pm0.05$ & $0.23\pm0.04$ & $0.26\pm0.08$ & $0.23\pm0.06$ \\
                       & \cac{} & $0.77\pm0.12$ & $0.82\pm0.10$ & $0.78\pm0.08$ & $0.79\pm0.10$ \\
                       & \pulse{} & $0.6\pm0.09$ & $0.61\pm0.11$ & $0.60\pm0.09$ & $0.60\pm0.09$ \\
                       & \ours{} & $\boldsymbol{0.95\pm0.04}$ & $\boldsymbol{0.95\pm0.01}$ & $\boldsymbol{0.96\pm0.03}$ & $\boldsymbol{0.95\pm0.03}$ \\
\hline
\multirow{5}{*}{\oscr{}} & \DDnew{} & $0.93\pm0.02$ & $0.93\pm0.03$ &                           $0.93\pm0.02$ & $0.93\pm0.02$ \\
                       & \uPUnew{} & $0.92\pm0.02$ & $0.91\pm0.05$ & $0.91\pm0.04$ & $0.92\pm0.03$ \\
                       & \nnPUnew{} & $0.92\pm0.02$ & $0.91\pm0.05$ & $0.91\pm0.04$ & $0.92\pm0.03$ \\
                       & \BODA{} & $0.82\pm0.04$ & $0.61\pm0.34$ & $0.73\pm0.14$ & $0.72\pm0.21$ \\
                       & \zoc{} & $0.53\pm0.03$ & $0.52\pm0.05$ & $0.47\pm0.08$ & $0.51\pm0.06$ \\
                       & \cac{} & $0.96\pm0.02$ & $0.96\pm0.01$ & $0.95\pm0.02$ & $0.96\pm0.02$ \\
                       & \pulse{} & $0.94\pm0.01$ & $0.95\pm0.02$ & $0.94\pm0.02$ & $0.94\pm0.02$ \\
                       & \ours{} & $\boldsymbol{0.96\pm0.02}$ & $\boldsymbol{0.97\pm0.01}$ & $\boldsymbol{0.96\pm0.01}$ & $\boldsymbol{0.96\pm0.01}$ \\         
\end{tabular}
}
\end{table*}

\begin{table*}[h!]
\centering
\caption{SUN397 dataset \textbf{without} any intended distribution shift. All the principled methods (\DDnew{}, \uPU{}, \nnPU{}, \BODA{} \& \ours{}) here use pretrained CLIP \vit{} backbone from \cite{radford2021clip}. }
\label{table:sun397_vitl14_wo_shift}
\resizebox{\textwidth}{!}{%

\begin{tabular}{c || c || c | c | c || c }
\hline
Metric & Method & \multicolumn{3}{c||}{Novel Classes (natural outdoor scenes/places)} &  \\
\cline{3-5}
 & & $\alpha$ $=0.06\pm0.00$ & $\alpha$ $=0.05\pm0.01$ & $\alpha$ $=0.07\pm0.03$ & \\
 \cline{3-5}
 & & [water, ice, snow, etc.] & [mountains, hills, desert, sky, etc.] & [forest, field, jungle, etc.] & Summary \\
\hline\hline
\multirow{5}{*}{\auroc{}} & \DDnew{} & $\boldsymbol{1.00\pm0.00}$ &                                      $\boldsymbol{1.00\pm0.00}$ & $\boldsymbol{1.00\pm0.00}$ &                         $\boldsymbol{1.00\pm0.00}$ \\
                       & \uPUnew{} & $\boldsymbol{1.00\pm0.00}$ & $\boldsymbol{1.00\pm0.00}$ & $\boldsymbol{1.00\pm0.00}$ & $\boldsymbol{1.00\pm0.00}$ \\
                       & \nnPUnew{} & $\boldsymbol{1.00\pm0.00}$ & $\boldsymbol{1.00\pm0.00}$ & $\boldsymbol{1.00\pm0.00}$ & $\boldsymbol{1.00\pm0.00}$ \\
                       & \BODA{} & $0.87\pm0.03$ & $0.88\pm0.04$ & $0.90\pm0.05$ & $0.88\pm0.04$ \\
                       & \arpl{} & $0.86\pm0.03$ & $0.81\pm0.03$ & $0.84\pm0.05$ & $0.84\pm0.04$ \\
                       & \zoc{} & $0.84\pm0.02$ & $0.86\pm0.02$ & $0.77\pm0.12$ & $0.82\pm0.08$ \\
                       & \cac{} & $0.99\pm0.01$ & $0.99\pm0.00$ & $0.99\pm0.01$ & $0.99\pm0.01$ \\
                       & \pulse{} & $0.98\pm0.01$ & $0.98\pm0.01$ & $0.97\pm0.02$ & $0.98\pm0.01$ \\
                       & \ours{} & $\boldsymbol{1.00\pm0.00}$ & $\boldsymbol{1.00\pm0.00}$ & $\boldsymbol{1.00\pm0.00}$ & $\boldsymbol{1.00\pm0.00}$ \\
\hline
\multirow{5}{*}{\auprc{}} & \DDnew{} & $\boldsymbol{1.00\pm0.00}$ &                                              $\boldsymbol{1.00\pm0.00}$ & $\boldsymbol{1.00\pm0.00}$ &                                 $\boldsymbol{1.00\pm0.00}$ \\
                       & \uPUnew{} & $\boldsymbol{1.00\pm0.00}$ & $\boldsymbol{1.00\pm0.00}$ & $\boldsymbol{1.00\pm0.00}$ & $\boldsymbol{1.00\pm0.00}$ \\
                       & \nnPUnew{} & $\boldsymbol{1.00\pm0.00}$ & $\boldsymbol{1.00\pm0.00}$ & $\boldsymbol{1.00\pm0.00}$ & $\boldsymbol{1.00\pm0.00}$ \\
                       & \BODA{} & $0.18\pm0.04$ & $0.20\pm0.06$ & $0.34\pm0.22$ & $0.24\pm0.15$ \\
                       & \arpl{} & $0.20\pm0.04$ & $0.16\pm0.04$ & $0.25\pm0.09$ & $0.20\pm0.07$ \\
                       & \zoc{} & $0.19\pm0.02$ & $0.21\pm0.04$ & $0.25\pm0.11$ & $0.22\pm0.07$ \\
                       & \cac{} & $0.87\pm0.10$ & $0.91\pm0.03$ & $0.86\pm0.08$ & $0.88\pm0.08$ \\
                       & \pulse{} & $0.76\pm0.08$ & $0.76\pm0.08$ & $0.74\pm0.11$ & $0.75\pm0.09$ \\
                       & \ours{} & $0.99\pm0.01$ & $0.99\pm0.01$ & $0.99\pm0.01$ & $0.99\pm0.01$ \\
\hline
\multirow{5}{*}{\oscr{}} & \DDnew{} & $\boldsymbol{0.99\pm0.00}$ &                                       $\boldsymbol{0.99\pm0.00}$ & $\boldsymbol{0.99\pm0.00}$ &                         $\boldsymbol{0.99\pm0.00}$ \\
                       & \uPUnew{} & $\boldsymbol{0.99\pm0.00}$ & $\boldsymbol{0.99\pm0.00}$ & $\boldsymbol{0.99\pm0.00}$ & $\boldsymbol{0.99\pm0.00}$ \\
                       & \nnPUnew{} & $\boldsymbol{0.99\pm0.00}$ & $\boldsymbol{0.99\pm0.00}$ & $\boldsymbol{0.99\pm0.00}$ & $\boldsymbol{0.99\pm0.00}$ \\
                       & \BODA{} & $0.94\pm0.01$ & $0.94\pm0.01$ & $0.92\pm0.02$ & $0.93\pm0.02$ \\
                       & \arpl{} & $0.82\pm0.02$ & $0.78\pm0.03$ & $0.80\pm0.05$ & $0.80\pm0.04$ \\
                       & \zoc{} & $0.50\pm0.01$ & $0.51\pm0.02$ & $0.45\pm0.07$ & $0.49\pm0.05$ \\
                       & \cac{} & $0.98\pm0.01$ & $0.99\pm0.00$ & $0.98\pm0.01$ & $0.98\pm0.01$ \\
                       & \pulse{} & $0.98\pm0.01$ & $0.98\pm0.01$ & $0.97\pm0.02$ & $0.97\pm0.02$ \\
                       & \ours{} & $0.98\pm0.00$ & $0.98\pm0.00$ & $0.98\pm0.00$ & $0.98\pm0.00$ \\         
\end{tabular}
}
\end{table*}

\begin{table*}[h!]
\centering
\caption{CIFAR100 dataset with background shift due to varying proportions of subtypes. (\auroc{}) and (\auprc{}) represent the performance for the task of novel class detection while Open-Set Classification Rate (\oscr{}) measures overall performance of the methods on both known and unknown classes. $\alpha$ is in the range of $0.05$ to $0.10$. \textit{Note that all the adaptive methods use ResNet18 and are trained from scratch.}}
\label{table:cifar100_w_shift}
\resizebox{\textwidth}{!}{%

\begin{tabular}{c || c || c | c | c | c | c || c }
\hline
Metric & Method & \multicolumn{5}{c||}{Novel Classes ($\alpha$$=0.07\pm0.02$)} &  \\
\cline{3-7}
 & & Baby & Man & Butterfly & Rocket & Streetcar & Summary \\
\hline\hline
\multirow{5}{*}{\auroc{}} & \DDnew{} & $0.76\pm0.03$ & $0.61\pm0.13$ & $0.64\pm0.10$                         & $0.78\pm0.09$ & $0.72\pm0.07$ & $0.70\pm0.11$ \\
                       & \uPUnew{} & $0.71\pm0.07$ & $0.62\pm0.09$ & $0.56\pm0.13$ & $0.74\pm0.06$ & $0.73\pm0.05$ & $0.67\pm0.10$ \\
                       & \nnPUnew{} & $0.69\pm0.09$ & $0.58\pm0.10$ & $0.60\pm0.09$ & $0.74\pm0.06$ & $0.72\pm0.16$ & $0.67\pm0.12$ \\
                       & \BODA{} & $0.57\pm0.07$ & $0.54\pm0.05$ & $0.59\pm0.03$ & $0.54\pm0.07$ & $0.61\pm0.05$ & $0.57\pm0.06$ \\
                       & \arpl{} & $0.79\pm0.05$ & $0.80\pm0.03$ & $0.73\pm0.05$ & $0.73\pm0.04$ & $0.73\pm0.05$ & $0.76\pm0.05$ \\
                       & \zoc{} & $\boldsymbol{0.97\pm0.01}$ & $\boldsymbol{0.98\pm0.01}$ & $\boldsymbol{0.82\pm0.03}$ & $\boldsymbol{0.98\pm0.01}$ & $\boldsymbol{1.00\pm0.00}$ & $\boldsymbol{0.95\pm0.07}$ \\
                       & \cac{} & $0.71\pm0.05$ & $0.72\pm0.04$ & $0.71\pm0.02$ & $0.64\pm0.03$ & $0.68\pm0.02$ & $0.69\pm0.04$ \\
                       & \pulse{} & $0.72\pm0.04$ & $0.74\pm0.03$ & $0.70\pm0.02$ & $0.70\pm0.02$ & $0.72\pm0.06$ & $0.72\pm0.04$ \\
                       & \ours{} & $0.77\pm0.06$ & $0.70\pm0.12$ & $0.68\pm0.04$ & $0.83\pm0.04$ & $0.85\pm0.03$ & $0.77\pm0.09$ \\
\hline
\multirow{5}{*}{\auprc{}} & \DDnew{} & $0.25\pm0.03$ & $0.15\pm0.10$ & $0.15\pm0.09$                         & $0.37\pm0.12$ & $0.27\pm0.17$ & $0.24\pm0.13$ \\
                       & \uPUnew{} & $0.22\pm0.18$ & $0.11\pm0.03$ & $0.11\pm0.05$ & $0.28\pm0.14$ & $0.18\pm0.10$ & $0.18\pm0.12$ \\
                       & \nnPUnew{} & $0.21\pm0.16$ & $0.09\pm0.03$ & $0.11\pm0.04$ & $0.32\pm0.09$ & $0.27\pm0.25$ & $0.20\pm0.16$ \\
                       & \BODA{} & $0.10\pm0.05$ & $0.08\pm0.02$ & $0.09\pm0.02$ & $0.08\pm0.02$ & $0.09\pm0.03$ & $0.09\pm0.03$ \\
                       & \arpl{} & $0.22\pm0.07$ & $0.23\pm0.06$ & $0.16\pm0.05$ & $0.15\pm0.05$ & $0.15\pm0.06$ & $0.18\pm0.07$ \\
                       & \zoc{} & $\boldsymbol{0.89\pm0.04}$ & $\boldsymbol{0.76\pm0.05}$ & $\boldsymbol{0.32\pm0.07}$ & $\boldsymbol{0.77\pm0.05}$ & $\boldsymbol{0.96\pm0.02}$ & $\boldsymbol{0.74\pm0.23}$ \\
                       & \cac{} & $0.18\pm0.06$ & $0.16\pm0.06$ & $0.17\pm0.04$ & $0.11\pm0.03$ & $0.14\pm0.07$ & $0.15\pm0.05$ \\
                       & \pulse{} & $0.18\pm0.06$ & $0.21\pm0.06$ & $0.15\pm0.05$ & $0.13\pm0.02$ & $0.16\pm0.04$ & $0.17\pm0.05$ \\
                       & \ours{} & $0.35\pm0.12$ & $0.20\pm0.11$ & $0.21\pm0.06$ & $0.43\pm0.03$ & $0.44\pm0.14$ & $0.33\pm0.14$ \\
\hline
\multirow{5}{*}{\oscr{}} & \DDnew{} & $0.55\pm0.04$ & $0.44\pm0.12$ & $0.44\pm0.11$                         & $0.57\pm0.10$ & $0.52\pm0.07$ & $0.51\pm0.10$ \\
                       & \uPUnew{} & $0.51\pm0.08$ & $0.45\pm0.08$ & $0.42\pm0.12$ & $0.53\pm0.05$ & $0.53\pm0.05$ & $0.49\pm0.09$ \\
                       & \nnPUnew{} & $0.50\pm0.09$ & $0.41\pm0.07$ & $0.44\pm0.07$ & $0.54\pm0.08$ & $0.52\pm0.16$ & $0.48\pm0.10$ \\
                       & \BODA{} & $0.61\pm0.04$ & $0.61\pm0.04$ & $0.61\pm0.02$ & $0.57\pm0.05$ & $0.62\pm0.04$ & $0.60\pm0.04$ \\
                       & \arpl{} & $0.63\pm0.04$ & $0.64\pm0.03$ & $0.60\pm0.04$ & $0.60\pm0.04$ & $0.59\pm0.05$ & $0.61\pm0.04$ \\
                       & \zoc{} & $\boldsymbol{0.82\pm0.02}$ & $\boldsymbol{0.83\pm0.02}$ & $\boldsymbol{0.70\pm0.02}$ & $\boldsymbol{0.82\pm0.02}$ & $\boldsymbol{0.82\pm0.01}$ & $\boldsymbol{0.84\pm0.01}$ \\
                       & \cac{} & $0.56\pm0.04$ & $0.57\pm0.04$ & $0.56\pm0.03$ & $0.51\pm0.03$ & $0.54\pm0.02$ & $0.55\pm0.04$ \\
                       & \pulse{} & $0.62\pm0.03$ & $0.64\pm0.04$ & $0.61\pm0.02$ & $0.60\pm0.02$ & $0.62\pm0.05$ & $0.62\pm0.03$ \\
                       & \ours{} & $0.58\pm0.06$ & $0.54\pm0.12$ & $0.52\pm0.05$ & $0.64\pm0.05$ & $0.65\pm0.02$ & $0.59\pm0.08$ \\         
\end{tabular}

}
\end{table*}

\begin{table*}[h!]
\centering
\caption{CIFAR100 dataset without background shift. (\auroc{}) and (\auprc{}) represent the performance for the task of novel class detection while Open-Set Classification Rate (\oscr{}) measures overall performance of the methods on both known and unknown classes. $\alpha$ is set to $0.05$. \textit{Note that all the adaptive methods use ResNet18 and are trained from scratch.}}
\label{table:cifar100_wo_shift}
\resizebox{\textwidth}{!}{%

\begin{tabular}{c || c || c | c | c | c | c || c }
\hline
Metric & Method & \multicolumn{5}{c||}{Novel Classes ($\alpha$$=0.05\pm0.00$)} &  \\
\cline{3-7}
 & & Baby & Man & Butterfly & Rocket & Streetcar & Summary \\
\hline\hline
\multirow{5}{*}{\auroc{}} & \DDnew{} & $0.82\pm0.04$ & $0.71\pm0.09$ &                          $0.71\pm0.11$ & $0.84\pm0.09$ & $0.72\pm0.15$ &                          $0.76\pm0.11$ \\
                       & \uPUnew{} & $0.75\pm0.10$ & $0.67\pm0.08$ & $0.64\pm0.12$ & $0.84\pm0.06$ & $0.74\pm0.11$ & $0.73\pm0.11$ \\
                       & \nnPUnew{} & $0.75\pm0.05$ & $0.62\pm0.05$ & $0.66\pm0.07$ & $0.82\pm0.04$ & $0.79\pm0.08$ & $0.73\pm0.10$ \\
                       & \BODA{} & $0.59\pm0.0$ & $0.58\pm0.02$ & $0.59\pm0.05$ & $0.59\pm0.02$ & $0.62\pm0.06$ & $0.60\pm0.04$ \\
                       & \arpl{} & $0.78\pm0.03$ & $0.82\pm0.03$ & $0.76\pm0.03$ & $0.73\pm0.04$ & $0.74\pm0.05$ & $0.77\pm0.05$ \\
                       & \zoc{} & $\boldsymbol{0.97\pm0.01}$ & $\boldsymbol{0.98\pm0.01}$ & $\boldsymbol{0.82\pm0.02}$ & $\boldsymbol{0.98\pm0.00}$ & $\boldsymbol{1.00\pm0.00}$ & $\boldsymbol{0.95\pm0.07}$ \\
                       & \cac{} & $0.71\pm0.01$ & $0.72\pm0.02$ & $0.69\pm0.03$ & $0.62\pm0.03$ & $0.71\pm0.01$ & $0.69\pm0.04$ \\
                       & \pulse{} & $0.73\pm0.02$ & $0.77\pm0.03$ & $0.72\pm0.01$ & $0.70\pm0.02$ & $0.72\pm0.05$ & $0.73\pm0.04$ \\
                       & \ours{} & $0.82\pm0.05$ & $0.76\pm0.04$ & $0.76\pm0.05$ & $0.82\pm0.03$ & $0.82\pm0.05$ & $0.80\pm0.05$ \\
\hline
\multirow{5}{*}{\auprc{}} & \DDnew{} & $0.39\pm0.10$ & $0.15\pm0.06$ &                          $0.20\pm0.09$ & $0.48\pm0.23$ & $0.25\pm0.16$ &                          $0.29\pm0.18$ \\
                       & \uPUnew{} & $0.20\pm0.11$ & $0.13\pm0.06$ & $0.11\pm0.05$ & $0.45\pm0.12$ & $0.21\pm0.13$ & $0.22\pm0.15$ \\
                       & \nnPUnew{} & $0.18\pm0.08$ & $0.08\pm0.02$ & $0.10\pm0.05$ & $0.39\pm0.12$ & $0.27\pm0.14$ & $0.21\pm0.14$ \\
                       & \BODA{} & $0.07\pm0.01$ & $0.06\pm0.00$ & $0.06\pm0.01$ & $0.07\pm0.01$ & $0.07\pm0.01$ & $0.07\pm0.01$ \\
                       & \arpl{} & $0.15\pm0.06$ & $0.20\pm0.05$ & $0.18\pm0.06$ & $0.11\pm0.03$ & $0.12\pm0.02$ & $0.15\pm0.05$ \\
                       & \zoc{} & $\boldsymbol{0.87\pm0.03}$ & $\boldsymbol{0.68\pm0.04}$ & $\boldsymbol{0.25\pm0.04}$ & $\boldsymbol{0.70\pm0.05}$ & $\boldsymbol{0.94\pm0.02}$ & $\boldsymbol{0.69\pm0.25}$ \\
                       & \cac{} & $0.12\pm0.02$ & $0.13\pm0.02$ & $0.09\pm0.01$ & $0.08\pm0.01$ & $0.12\pm0.02$ & $0.11\pm0.02$ \\
                       & \pulse{} & $0.12\pm0.01$ & $0.16\pm0.03$ & $0.12\pm0.02$ & $0.12\pm0.03$ & $0.11\pm0.02$ & $0.13\pm0.03$ \\
                       & \ours{} & $0.35\pm0.12$ & $0.17\pm0.05$ & $0.21\pm0.07$ & $0.39\pm0.05$ & $0.33\pm0.11$ & $0.29\pm0.12$ \\
\hline
\multirow{5}{*}{\oscr{}} & \DDnew{} & $0.63\pm0.05$ & $0.51\pm0.15$ &                           $0.52\pm0.17$ & $0.62\pm0.17$ & $0.53\pm0.18$ &                          $0.56\pm0.15$ \\
                       & \uPUnew{} & $0.54\pm0.12$ & $0.49\pm0.08$ & $0.45\pm0.11$ & $0.63\pm0.06$ & $0.53\pm0.13$ & $0.53\pm0.12$ \\
                       & \nnPUnew{} & $0.55\pm0.06$ & $0.44\pm0.06$ & $0.45\pm0.07$ & $0.58\pm0.06$ & $0.59\pm0.10$ & $0.52\pm0.09$ \\
                       & \BODA{} & $0.62\pm0.04$ & $0.63\pm0.03$ & $0.61\pm0.02$ & $0.56\pm0.03$ & $0.64\pm0.01$ & $0.61\pm0.04$ \\
                       & \arpl{} & $0.66\pm0.02$ & $0.68\pm0.02$ & $0.64\pm0.03$ & $0.62\pm0.03$ & $0.63\pm0.04$ & $0.65\pm0.04$ \\
                       & \zoc{} & $\boldsymbol{0.81\pm0.01}$ & $\boldsymbol{0.81\pm0.00}$ & $\boldsymbol{0.70\pm0.02}$ & $\boldsymbol{0.81\pm0.00}$ & $\boldsymbol{0.82\pm0.01}$ & $\boldsymbol{0.79\pm0.05}$ \\
                       & \cac{} & $0.55\pm0.01$ & $0.57\pm0.01$ & $0.53\pm0.02$ & $0.49\pm0.03$ & $0.55\pm0.02$ & $0.54\pm0.04$ \\ 
                       & \pulse{} & $0.63\pm0.02$ & $0.66\pm0.02$ & $0.63\pm0.02$ & $0.61\pm0.02$ & $0.63\pm0.04$ & $0.63\pm0.03$ \\ 
                       & \ours{} & $0.65\pm0.02$ & $0.60\pm0.03$ & $0.59\pm0.05$ & $0.63\pm0.04$ & $0.63\pm0.06$ & $0.62\pm0.04$ \\         
\end{tabular}

}
\end{table*}

\begin{table*}[h!]
\centering
\caption{Amazon Reviews dataset with sentiment based distribution shift. \auroc{} and \auprc{} represent the performance for the task of novel class detection while \oscr{} measures overall performance of the methods on both known and unknown classes. $\alpha$ is the mixture proportion column for the respective novel classes.}
\label{table:amazon_reviews_w_shift}
\resizebox{\textwidth}{!}{%

\begin{tabular}{c || c || c | c | c | c | c | c || c }
\hline
Metric & Method & \multicolumn{6}{c||}{Novel Classes} &  \\
\cline{3-8}
 & & $\alpha=0.32$ & $\alpha=0.11$ & $\alpha=0.07$ & $\alpha=0.07$ & $\alpha=0.22$ & $\alpha=0.18$ &  \\
\cline{3-8}
 & & Musical Instruments & Digital Music & Software & Luxury Beauty & Industrial \& Scientific & Prime Pantry & Summary \\
\hline\hline
\multirow{5}{*}{\auroc{}} & \DDnew{} & $0.75\pm0.05$ & $0.81\pm0.04$ & $0.80\pm0.04$                         & $0.73\pm0.04$ & $0.58\pm0.04$ & $0.63\pm0.06$ &                                   $0.72\pm0.09$\\
                       & \uPUnew{} & $0.79\pm0.04$ & $0.85\pm0.06$ & $0.79\pm0.03$ & $0.76\pm0.03$ & $0.67\pm0.04$ & $0.67\pm0.09$ & $0.76\pm0.08$ \\
                       & \nnPUnew{} & $0.79\pm0.04$ & $0.85\pm0.06$ & $0.79\pm0.03$ & $0.76\pm0.03$ & $0.67\pm0.04$ & $0.67\pm0.09$ & $0.76\pm0.08$ \\
                       & \BODA{} & $0.76\pm0.03$ & $0.61\pm0.13$ & $0.56\pm0.06$ & $0.67\pm0.07$ & $0.72\pm0.04$ & $0.65\pm0.09$ & $0.66\pm0.10$ \\
                       & \arpl{} & $0.66\pm0.04$ & $0.75\pm0.06$ & $0.70\pm0.05$ & $0.66\pm0.06$ & $0.72\pm0.02$ & $0.74\pm0.03$ & $0.70\pm0.05$ \\
                       & \cac & $0.70\pm0.03$ & $0.72\pm0.05$ & $0.70\pm0.03$ & $0.56\pm0.08$ & $\boldsymbol{0.74\pm0.05}$ & $0.74\pm0.02$ & $0.70\pm0.07$ \\
                       & \pulse{} & $0.60\pm0.03$ & $0.57\pm0.07$ & $0.65\pm0.06$ & $0.58\pm0.03$ & $0.70\pm0.02$ & $\boldsymbol{0.75\pm0.02}$ & $0.63\pm0.08$ \\
                       & \ours{} & $\boldsymbol{0.82\pm0.03}$ & $\boldsymbol{0.87\pm0.04}$ & $\boldsymbol{0.84\pm0.06}$ & $\boldsymbol{0.77\pm0.04}$ & $0.68\pm0.09$ & $0.72\pm0.11$ & $\boldsymbol{0.79\pm0.09}$ \\
\hline
\multirow{5}{*}{\auprc{}} & \DDnew{} & $0.68\pm0.06$ & $0.62\pm0.08$ & $0.32\pm0.12$                         & $0.27\pm0.09$ & $0.33\pm0.05$ & $0.38\pm0.07$ &                                   $0.43\pm0.18$ \\
                       & \uPUnew{} & $0.73\pm0.04$ & $0.68\pm0.10$ & $0.39\pm0.07$ & $0.37\pm0.10$ & $0.43\pm0.07$ & $0.44\pm0.14$ & $0.51\pm0.17$ \\
                       & \nnPUnew{} & $0.73\pm0.04$ & $0.68\pm0.10$ & $0.39\pm0.07$ & $0.37\pm0.10$ & $0.43\pm0.07$ & $0.44\pm0.14$ & $0.51\pm0.17$ \\
                       & \BODA{} & $0.67\pm0.07$ & $0.19\pm0.11$ & $0.07\pm0.01$ & $0.14\pm0.04$ & $0.38\pm0.06$ & $0.25\pm0.07$ & $0.28\pm0.21$ \\
                       & \arpl{} & $0.43\pm0.04$ & $0.27\pm0.13$ & $0.14\pm0.04$ & $0.12\pm0.03$ & $0.36\pm0.04$ & $0.33\pm0.03$ & $0.27\pm0.13$ \\
                       & \cac{} & $0.50\pm0.04$ & $0.24\pm0.06$ & $0.14\pm0.02$ & $0.08\pm0.01$ & $0.41\pm0.07$ & $0.38\pm0.03$ & $0.30\pm0.16$ \\
                       & \pulse{} & $0.38\pm0.02$ & $0.14\pm0.04$ & $0.10\pm0.02$ & $0.09\pm0.01$ & $0.33\pm0.03$ & $0.36\pm0.03$ & $0.21\pm0.13$ \\
                       & \ours{} & $\boldsymbol{0.76\pm0.04}$ & $\boldsymbol{0.69\pm0.07}$ & $\boldsymbol{0.49\pm0.12}$ & $\boldsymbol{0.35\pm0.14}$ & $\boldsymbol{0.44\pm0.12}$ & $\boldsymbol{0.50\pm0.15}$ & $\boldsymbol{0.54\pm0.18}$ \\
\hline
\multirow{5}{*}{\oscr{}} & \DDnew{} & $0.54\pm0.04$ & $0.54\pm0.05$ & $0.53\pm0.05$                         & $0.47\pm0.02$ & $0.43\pm0.03$ & $0.46\pm0.04$ &                                    $0.50\pm0.06$ \\
                       & \uPUnew{} & $0.59\pm0.03$ & $\boldsymbol{0.59\pm0.06}$ & $0.53\pm0.04$ & $0.49\pm0.02$ & $0.52\pm0.03$ & $0.51\pm0.03$ & $0.54\pm0.05$ \\
                       & \nnPUnew{} & $0.59\pm0.03$ & $\boldsymbol{0.59\pm0.06}$ & $0.53\pm0.04$ & $0.49\pm0.02$ & $0.52\pm0.03$ & $0.51\pm0.03$ & $0.54\pm0.05$ \\
                       & \BODA{} & $0.50\pm0.01$ & $0.53\pm0.05$ & $\boldsymbol{0.56\pm0.04}$ & $\boldsymbol{0.51\pm0.03}$ & $\boldsymbol{0.62\pm0.01}$ & $0.62\pm0.06$ & $\boldsymbol{0.56\pm0.06}$ \\
                       & \arpl{} & $0.55\pm0.04$ & $0.53\pm0.02$ & $0.54\pm0.04$ & $\boldsymbol{0.51\pm0.03}$ & $\boldsymbol{0.62\pm0.04}$ & $0.60\pm0.02$ & $0.56\pm0.05$ \\
                       & \cac{} & $0.50\pm0.02$ & $0.47\pm0.04$ & $0.46\pm0.03$ & $0.37\pm0.04$ & $0.57\pm0.05$ & $0.56\pm0.03$ & $0.49\pm0.07$ \\
                       & \pulse{} & $0.52\pm0.03$ & $0.47\pm0.04$ & $0.54\pm0.04$ & $0.49\pm0.03$ & $0.63\pm0.03$ & $\boldsymbol{0.64\pm0.06}$ & $0.53\pm0.07$ \\
                       & \ours{} & $\boldsymbol{0.61\pm0.04}$ & $\boldsymbol{0.59\pm0.05}$ & $\boldsymbol{0.56\pm0.07}$ & $0.50\pm0.05$ & $0.56\pm0.03$ & $0.55\pm0.07$ & $\boldsymbol{0.56\pm0.06}$ \\         
\end{tabular}
}
\end{table*}

\begin{figure*}[h!]
\begin{subfigure}{.33\textwidth}
  \centering
  \includegraphics[width=1.1\linewidth]{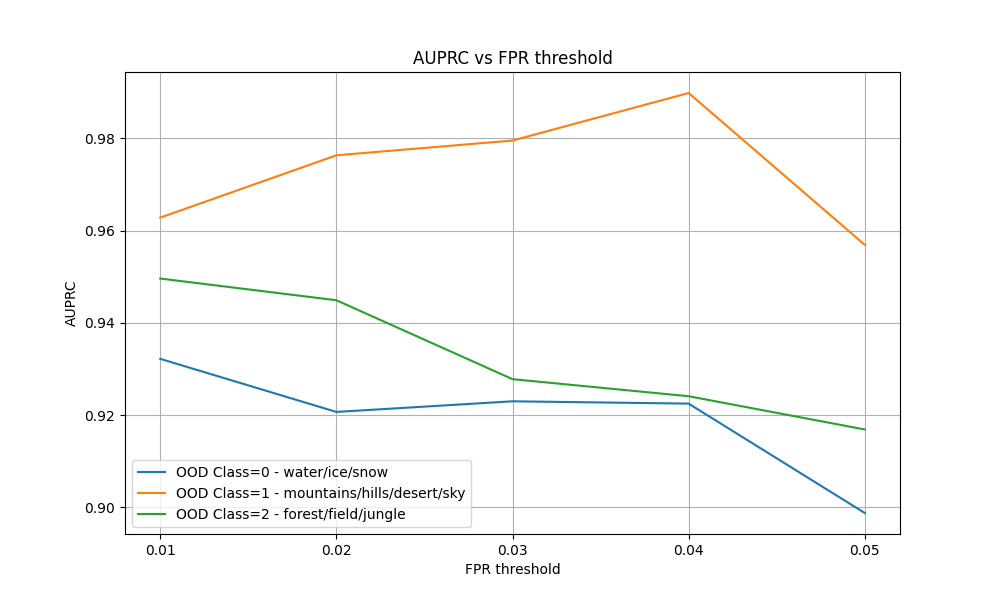}
  \caption{}
  \label{fig:auprc_vs_fpr}
\end{subfigure}%
\begin{subfigure}{.33\textwidth}
  \centering
  \includegraphics[width=1.1\linewidth]{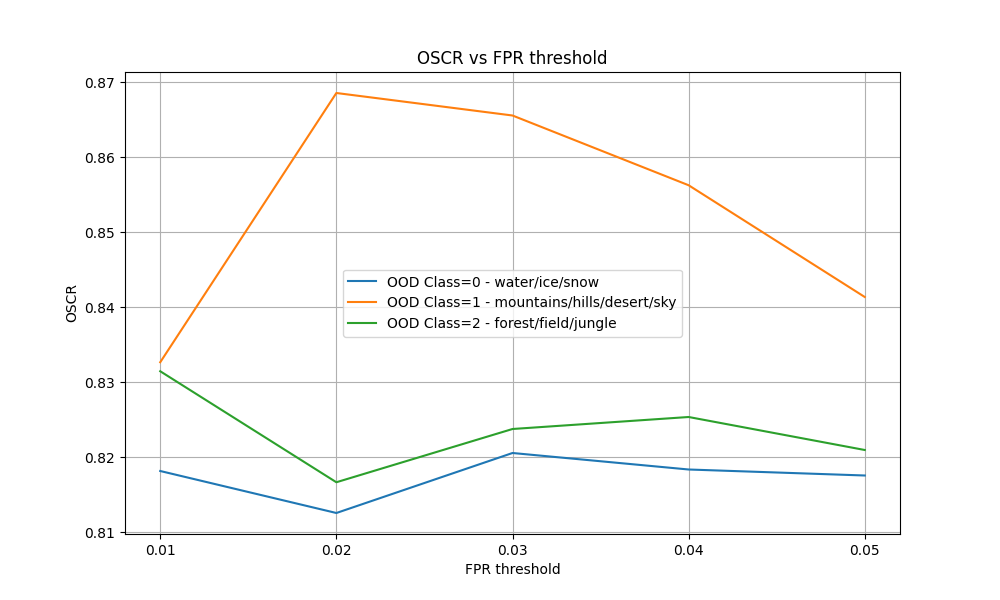}
  \caption{}
  \label{fig:oscr_vs_fpr}
\end{subfigure}%
\begin{subfigure}{.33\textwidth}
  \centering
  \includegraphics[width=1.1\linewidth]{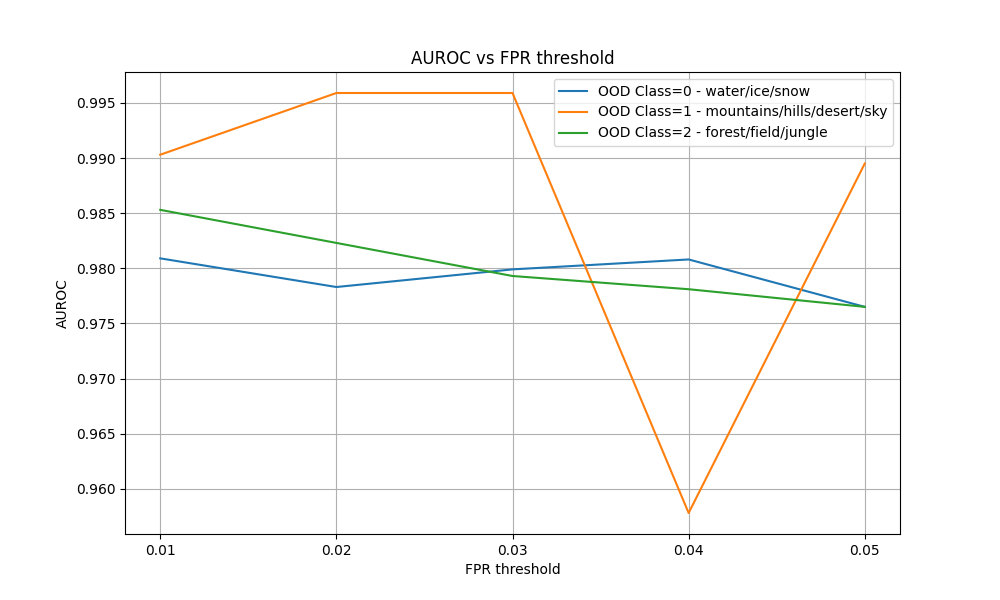}
  \caption{}
  \label{fig:auroc_vs_fpr}
\end{subfigure}
\caption{Effects of varying the FPR threshold on \ours{} method performance on SUN397 dataset.}
\label{fig:vary_fpr}
\end{figure*}

\begin{table*}[h!]
\centering
\caption{AUPRC scores with \& w/o joint learning for novelty detection on CIFAR100 dataset with distribution shift.\\
}
\label{table:abs_cifar100_w_shift_auprc}
\resizebox{\textwidth}{!}{%
\begin{tabular}{c | c || c | c | c | c | c | c | c | c } 
 \hline
 Novel Class & $\alpha$ & \multicolumn{8}{c}{Absolute AUPRC} \\
 \cline{3-10}
  &  & DD w/o joint learning & \DDnew{} & uPU w/o joint learning & \uPUnew{} & nnPU w/o joint learning & \nnPUnew{} & CoLoR w/o joint learning & \ours \\
 \hline\hline
 Baby & $0.23\pm0.05$ & $0.69\pm0.10$ & $0.77\pm0.04$ & $0.63\pm0.10$ & $0.64\pm0.10$ & $0.61\pm0.08$ & $0.61\pm0.11$ & $0.65\pm0.13$ & $\boldsymbol{0.80\pm0.08}$ \\ 
 Tulip & $0.29\pm0.12$ & $0.41\pm0.14$ & $0.39\pm0.14$ & $0.27\pm0.11$ & $0.25\pm0.04$ & $0.26\pm0.12$ & $0.29\pm0.10$ & $0.50\pm0.16$ & $\boldsymbol{0.53\pm0.16}$ \\
 Crocodile & $0.23\pm0.05$ & $0.58\pm0.08$ & $0.64\pm0.04$ & $0.47\pm0.09$ & $0.44\pm0.12$ & $0.47\pm0.11$ & $0.42\pm0.11$ & $0.61\pm0.15$ & $\boldsymbol{0.68\pm0.10}$ \\
 Dolphin & $0.29\pm0.12$ & $\boldsymbol{0.65\pm0.12}$ & $\boldsymbol{0.65\pm0.12}$ & $0.50\pm0.06$ & $0.49\pm0.08$ & $0.50\pm0.08$ & $0.45\pm0.11$ & $\boldsymbol{0.65\pm0.12}$ & $0.64\pm0.19$ \\
 Man & $0.23\pm0.05$ & $0.59\pm0.19$ & $0.64\pm0.13$ & $0.31\pm0.04$ & $0.41\pm0.14$ & $0.28\pm0.12$ & $0.46\pm0.13$ & $0.68\pm0.09$ & $\boldsymbol{0.79\pm0.07}$ \\ [1ex] 
 \hline
\end{tabular}
}
\end{table*}

\begin{table*}[h!]
\centering
\caption{AUPRC scores with \& w/o joint learning for novelty detection on Amazon Reviews dataset with distribution shift.
}
\label{table:abs_amazon_reviews_w_shift_auprc}
\resizebox{\textwidth}{!}{%
\begin{tabular}{c | c || c | c | c | c | c | c | c | c } 
 \hline
 Novel Class & $\alpha$ & \multicolumn{8}{c}{Absolute AUPRC} \\
 \cline{3-10}
  &  & DD w/o joint learning & \DDnew{} & uPU w/o joint learning & \uPUnew{} & nnPU w/o joint learning & \nnPUnew{} & CoLoR w/o joint learning & \ours \\
 \hline\hline
 Musical Instruments & $0.16$ & $0.37\pm0.09$ & $0.40\pm0.1$ & $0.34\pm0.09$ & $0.45\pm0.13$ & $0.34\pm0.09$ & $0.45\pm0.13$ & $0.48\pm0.11$ & $\boldsymbol{0.55\pm0.13}$ \\ 
 Digital Music & $0.05$ & $0.28\pm0.07$ & $0.29\pm0.09$ & $0.23\pm0.10$ & $0.24\pm0.12$ & $0.23\pm0.10$ & $0.24\pm0.12$ & $0.33\pm0.13$ & $\boldsymbol{0.35\pm0.15}$ \\
 Software & $0.06$ & $0.36\pm0.05$ & $0.37\pm0.11$ & $0.31\pm0.13$ & $0.36\pm0.08$ & $0.31\pm0.13$ & $0.36\pm0.08$ & $0.48\pm0.11$ & $\boldsymbol{0.49\pm0.12}$ \\
 Luxury Beauty & $0.07$ & $0.27\pm0.15$ & $0.25\pm0.08$ & $0.28\pm0.12$ & $0.35\pm0.11$ & $0.28\pm0.12$ & $0.35\pm0.11$ & $0.35\pm0.14$ & $\boldsymbol{0.36\pm0.13}$ \\
 Industrial \& Scientific & $0.12$ & $0.13\pm0.01$ & $0.14\pm0.03$ & $0.11\pm0.02$ & $0.16\pm0.05$ & $0.11\pm0.02$ & $0.16\pm0.05$ & $0.17\pm0.02$ & $\boldsymbol{0.21\pm0.08}$ \\ [1ex] 
 \hline
\end{tabular}
}
\end{table*}


\begin{table}[h!]
\centering
\caption{MSP/Entropy based OOD detection on SUN397 with low $\alpha=0.05\pm0.01$ on three types of novel categories}
\label{table:MSP_sun397_low_ratio}
\begin{tabular}{|c|c|c|c|c|c|c|}
\hline
\multirow{2}{*}{} & \multicolumn{2}{c|}{ [water, ice, snow, etc.]} & \multicolumn{2}{c|}{[mountains, hills, desert, sky, etc.]} & \multicolumn{2}{c|}{[forest, field, jungle, etc.]} \\ \cline{2-7} 
                  & AUROC        & AUPRC        & AUROC        & AUPRC        & AUROC        & AUPRC        \\ \hline
\multicolumn{7}{|c|}{With Shift}                                                                           \\ \hline
MSP               & 0.57 ± 0.05  & 0.05 ± 0.01  & 0.58 ± 0.06  & 0.05 ± 0.01  & 0.56 ± 0.05  & 0.07 ± 0.03  \\ \hline
Entropy           & 0.58 ± 0.07  & 0.08 ± 0.04  & 0.58 ± 0.08  & 0.07 ± 0.03  & 0.56 ± 0.07  & 0.10 ± 0.06  \\ \hline
\multicolumn{7}{|c|}{Without Shift}                                                                          \\ \hline
MSP               & 0.59 ± 0.05  & 0.04 ± 0.00  & 0.59 ± 0.05  & 0.05 ± 0.01  & 0.58 ± 0.07  & 0.06 ± 0.03  \\ \hline
Entropy           & 0.59 ± 0.07  & 0.08 ± 0.02  & 0.59 ± 0.08  & 0.08 ± 0.04  & 0.59 ± 0.10  & 0.09 ± 0.03  \\ \hline
\end{tabular}
\end{table}

\begin{table}[h!]
\centering
\caption{MSP/Entropy based OOD detection on SUN397 with high $\alpha=0.5\pm0.10$ on three types of novel categories}
\label{table:MSP_sun397_high_ratio}
\begin{tabular}{|c|c|c|c|c|c|c|}
\hline
\multirow{2}{*}{} & \multicolumn{2}{c|}{[water, ice, snow, etc.]} & \multicolumn{2}{c|}{[mountains, hills, desert, sky, etc.]} & \multicolumn{2}{c|}{[forest, field, jungle, etc.]} \\ \cline{2-7} 
                  & AUROC        & AUPRC        & AUROC        & AUPRC        & AUROC        & AUPRC        \\ \hline
\multicolumn{7}{|c|}{With Shift}                                                                           \\ \hline
MSP               & 0.52 ± 0.06  & 0.55 ± 0.05  & 0.61 ± 0.06  & 0.28 ± 0.04  & 0.60 ± 0.08  & 0.50 ± 0.09  \\ \hline
Entropy           & 0.52 ± 0.06  & 0.58 ± 0.05  & 0.63 ± 0.06  & 0.45 ± 0.07  & 0.62 ± 0.09  & 0.67 ± 0.04  \\ \hline
\multicolumn{7}{|c|}{Without Shift}                                                                          \\ \hline
MSP               & 0.59 ± 0.06  & 0.47 ± 0.04  & 0.61 ± 0.08  & 0.27 ± 0.04  & 0.58 ± 0.08  & 0.50 ± 0.05  \\ \hline
Entropy           & 0.59 ± 0.08  & 0.59 ± 0.07  & 0.63 ± 0.10  & 0.41 ± 0.10  & 0.59 ± 0.10  & 0.63 ± 0.09  \\ \hline
\end{tabular}
\end{table}

\begin{table}[h!]
\centering
\caption{Effect of background shift on Top-1 accuracies of remaining baselines. We do not expect any performance reduction of \zoc{} as it is a zero-shot method that does not utilize source data at all, rendering any distribution shift between $\Psource$ and $\Ptarget$ irrelevant. Such models are influenced only by target datasets that drift from their pretraining datasets.}
\resizebox{5cm}{!}{%
\begin{tabular}{c || c | c }
\hline
Methods & w/ DS & w/o DS \\
\hline\hline
\BODA{} (\vit{}) & $0.75\pm0.13$ & $0.94\pm0.00$ \\
\arpl{} (custom default) & $0.75\pm0.03$ & $0.91\pm0.01$ \\
\zoc{} (custom default) & $0.62\pm0.05$ &  $0.59\pm0.00$ \\
\anna{} (custom default) & $0.65\pm0.04$ &  $0.87\pm0.01$ \\
\hline
\end{tabular}
}
\label{table:sun397_remaining_top1}
\end{table}

\begin{table}[h!]
\centering
\caption{\ours{} performance for different target recall constraints for the SUN397 dataset with novel identity as outdoor scenes from the group [water, ice, snow, etc.] and novel class size $\alpha = 0.06\pm0.01$ using ResNet50 model. Note that out of $\|\boldsymbol{\alpha}\|$ novelty detection heads each corresponding to a candidate value $\hat{\alpha}\in\boldsymbol{\alpha}$, we select the head achieving highest recall in the validation set (argmax$_{w^\alpha_{\hat{\alpha}}:\alpha \in \boldsymbol{\alpha}, \hat{\beta}(w^\alpha_{\hat{\alpha}})<\beta} \hat{\alpha}(w^\alpha_{\hat{\alpha}})$). Consequently, the reported AUROC, AUPRC and OSCR performance correspond to that selected head. Scores of the selected novelty detection head for each seed are highlighted in bold below whereas the highest scores amongst all the novelty detection heads for each seed have been underlined. These heads having scores underlined were not selected besides having highest scores for a metric, because they simply did not satisfy the selection criteria, i.e. their validation recall values were not the highest for the corresponding seed among all the novelty detection heads.}
\resizebox{\textwidth}{!}{%
\begin{tabular}{c || c | c | c | c | c || c | c | c | c | c || c | c | c | c | c || c | c | c | c | c }
\hline
\textbf{Target Recall} & \multicolumn{5}{c||}{\textbf{Validation Recall across 5 seeds}} & \multicolumn{5}{c||}{\textbf{AUROC across 5 seeds}} & \multicolumn{5}{c||}{\textbf{AUPRC across 5 seeds}} & \multicolumn{5}{c}{\textbf{OSCR across 5 seeds}} \\ 
\cline{2-21}
 $\hat{\alpha}$ & 0 & 8 & 1057 & 103 & 573 & 0 & 8 & 1057 & 103 & 573 & 0 & 8 & 1057 & 103 & 573 & 0 & 8 & 1057 & 103 & 573 \\
\hline\hline
0.02 & $\boldsymbol{0.18}$ & $0.06$ & $0.10$ & $\boldsymbol{0.08}$ & $0.08$ & $\boldsymbol{0.98}$ & $\underline{0.97}$ & $\underline{0.95}$ & $\boldsymbol{0.99}$ & $0.94$ & $\boldsymbol{0.93}$ & $0.87$ & $\underline{0.87}$ & $\boldsymbol{0.94}$ & $0.43$ & $\boldsymbol{0.82}$ & $0.82$ & $0.75$ & $\boldsymbol{0.84}$ & $0.67$ \\
0.05 & $0.18$ & $0.06$ & $0.00$ & $0.08$ & $0.15$ & $\underline{0.99}$ & $0.96$ & $0.48$ & $\underline{1.00}$ & $\underline{0.99}$ & $\underline{0.94}$ & $0.82$ & $0.05$ & $\underline{0.95}$ & $0.95$ & $\underline{0.83}$ & $0.81$ & $0.25$ & $\underline{0.85}$ & $0.78$ \\
0.10 & $0.17$ & $0.02$ & $\boldsymbol{0.14}$ & $0.06$ & $0.09$ & $0.98$ & $0.82$ & $\boldsymbol{0.95}$ & $\underline{1.00}$ & $0.93$ & $\underline{0.93}$ & $0.13$ & $\boldsymbol{0.85}$ & $\underline{0.96}$ & $0.38$ & $0.81$ & $0.66$ & $\boldsymbol{0.76}$ & $\underline{0.85}$ & $0.65$ \\
0.15 & $0.16$ & $0.10$ & $0.01$ & $0.06$ & $0.04$ & $\underline{0.99}$ & $0.95$ & $0.42$ & $\underline{0.99}$ & $0.95$ & $\underline{0.95}$ & $0.80$ & $0.05$ & $0.92$ & $0.64$ & $\underline{0.84}$ & $0.80$ & $0.23$ & $\underline{0.84}$ & $0.65$ \\
0.20 & $0.12$ & $\boldsymbol{0.12}$ & $0.14$ & $0.05$ & $0.03$ & $0.97$ & $\boldsymbol{0.97}$ & $\underline{0.95}$ & $\underline{0.99}$ & $0.93$ & $0.79$ & $\boldsymbol{0.90}$ & $\underline{0.86}$ & $\underline{0.94}$ & $0.42$ & $0.76$ & $\boldsymbol{0.85}$ & $\underline{0.78}$ & $\underline{0.84}$ & $0.64$ \\
0.25 & $0.07$ & $0.00$ & $0.07$ & $0.03$ & $\boldsymbol{0.18}$ & $\underline{0.99}$ & $0.68$ & $0.47$ & $0.98$ & $\boldsymbol{0.99}$ & $0.83$ & $0.08$ & $0.05$ & $0.80$ & $\boldsymbol{0.97}$ & $0.77$ & $0.52$ & $0.27$ & $0.81$ & $\boldsymbol{0.82}$ \\
0.30 & $0.12$ & $0.01$ & $0.00$ & $0.07$ & $0.00$ & $\underline{0.99}$ & $0.81$ & $0.44$ & $\underline{1.00}$ & $0.86$ & $0.92$ & $0.13$ & $0.05$ & $\underline{0.95}$ & $0.29$ & $0.77$ & $0.65$ & $0.25$ & $\underline{0.85}$ & $0.61$ \\
0.35 & $0.11$ & $0.04$ & $0.02$ & $0.01$ & $0.04$ & $0.97$ & $0.81$ & $0.47$ & $0.84$ & $0.83$ & $0.55$ & $0.13$ & $0.05$ & $0.21$ & $0.25$ & $0.74$ & $0.65$ & $0.26$ & $0.69$ & $0.59$ \\
0.40 & $0.06$ & $0.00$ & $0.09$ & $0.01$ & $0.00$ & $0.94$ & $0.84$ & $0.43$ & $0.87$ & $0.78$ & $0.41$ & $0.16$ & $0.05$ & $0.25$ & $0.18$ & $0.73$ & $0.69$ & $0.23$ & $0.71$ & $0.57$ \\
0.45 & $0.04$ & $0.00$ & $0.00$ & $0.07$ & $0.00$ & $\underline{0.99}$ & $0.74$ & $0.54$ & $0.97$ & $0.79$ & $0.89$ & $0.11$ & $0.06$ & $0.82$ & $0.22$ & $0.77$ & $0.58$ & $0.34$ & $0.82$ & $0.57$ \\
\hline
\end{tabular}%
}
\label{table:varying_target_recall}
\end{table}

\begin{figure*}[t]
    \centering
    \caption{
    Effect of density of search grid $\boldsymbol{\alpha}$ on the performance of \ours{}. We consider the range [0.02, 0.45] to search the candidate $\alpha$ and vary the density (i.e. number of candidate $\alpha$s) in the same interval). We report AUROC, AUPRC, OSCR, and Top-1 Accuracy for different novel class groupings. We use \resnet{} features here that are pretrained on ImageNet1K.}
    \label{fig:alpha_density_ablation}

    \begin{subfigure}[t]{0.48\textwidth}
        \centering
        \includegraphics[width=\linewidth]{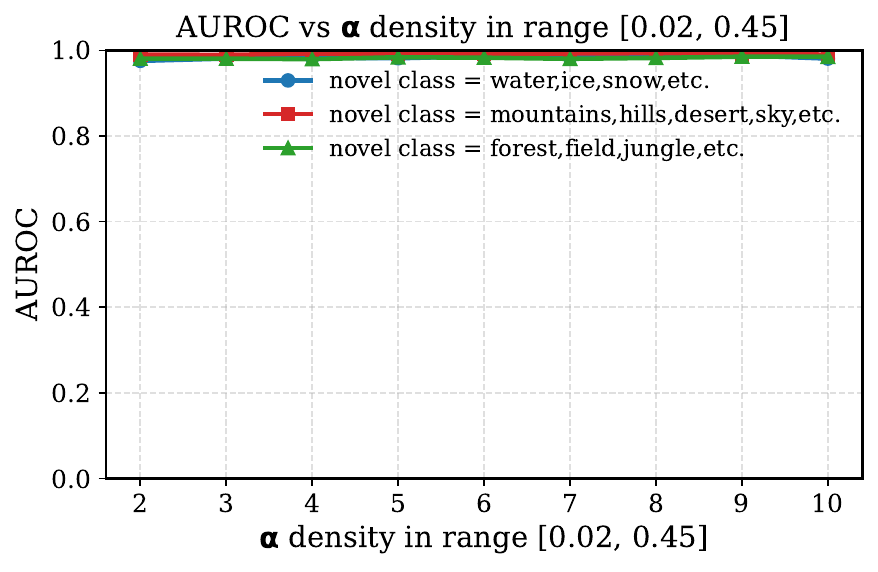}
        \label{fig:alpha_auroc}
    \end{subfigure}
    \hfill
    \begin{subfigure}[t]{0.48\textwidth}
        \centering
        \includegraphics[width=\linewidth]{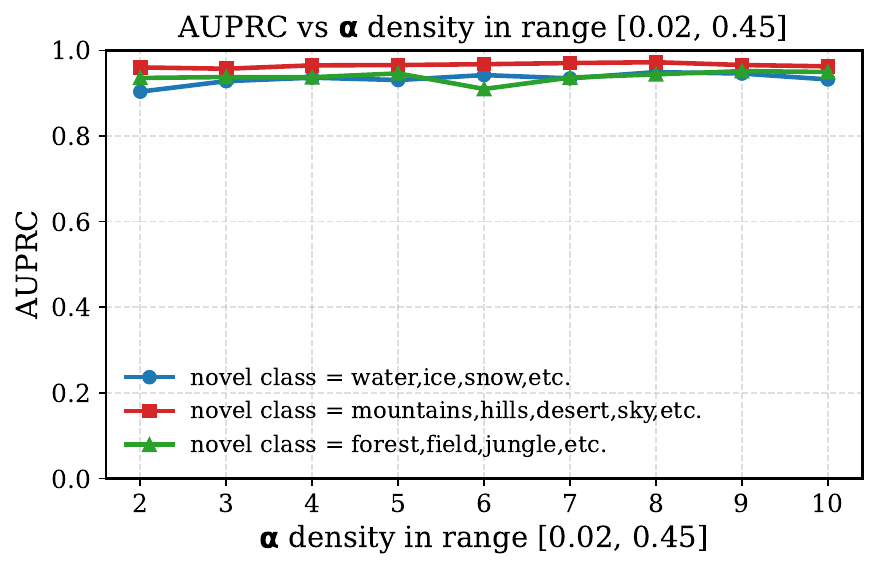}
        \label{fig:alpha_auprc}
    \end{subfigure}

    \vspace{0.5em}

    \begin{subfigure}[t]{0.48\textwidth}
        \centering
        \includegraphics[width=\linewidth]{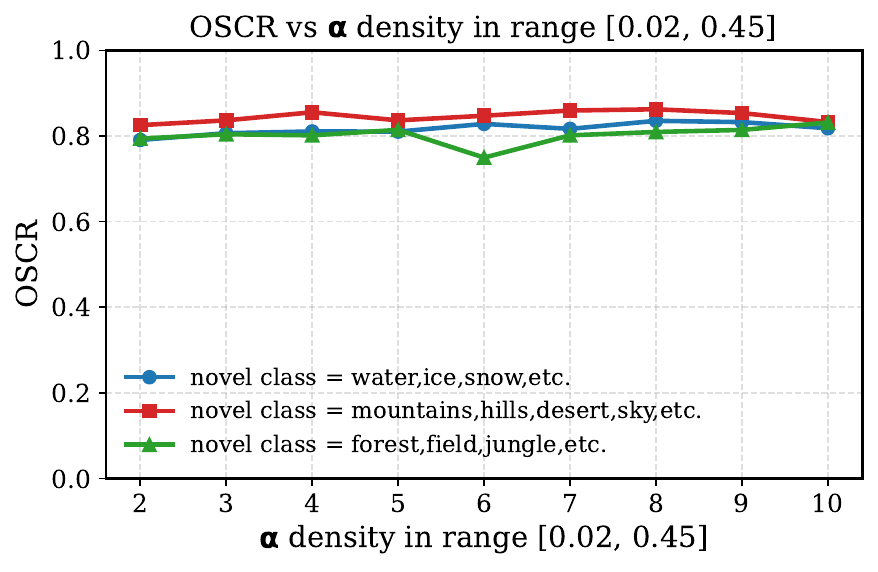}
        \label{fig:alpha_oscr}
    \end{subfigure}
    \hfill
    \begin{subfigure}[t]{0.48\textwidth}
        \centering
        \includegraphics[width=\linewidth]{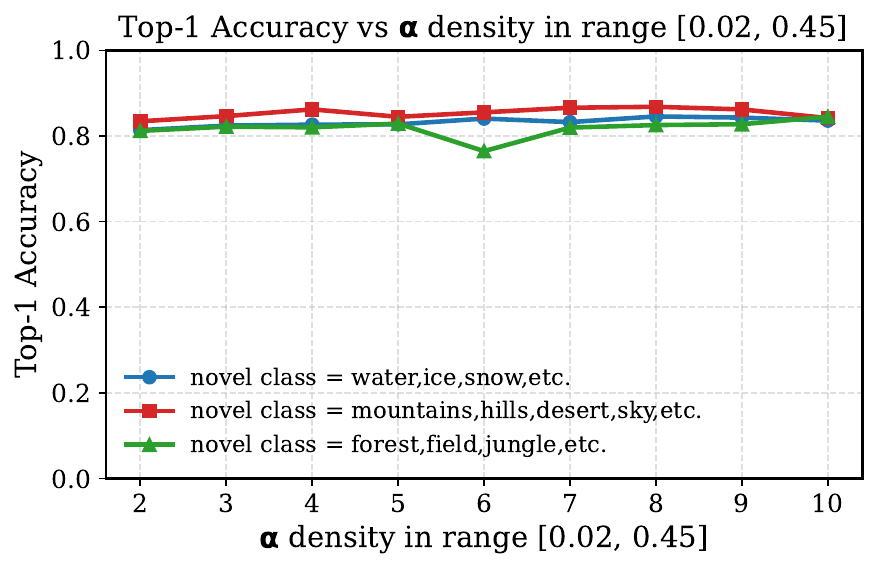}
        \label{fig:alpha_accuracy}
    \end{subfigure}

\end{figure*}

\begin{table*}[h!]
\centering
\caption{Effect of density of search grid $\boldsymbol{\alpha}$ on the performance of \ours{}. We consider the range [0.02, 0.45] to search the candidate $\alpha$ and vary the density (i.e. number of candidate $\alpha$s) in the same interval. 
We report AUROC, AUPRC, OSCR, and Top-1 Accuracy for different novel class groupings: (Group 0: [water, ice, snow, etc.]; Group 1: [mountains, hills, desert, sky, etc.]; Group 2:[forest, field, jungle, etc.]). We use \resnet{} features here that are pretrained on ImageNet1K.}
\label{tab:alpha_density_ablation}
\setlength{\tabcolsep}{6pt}
\renewcommand{\arraystretch}{1.15}
\begin{tabular}{c||c|c|c||c|c|c||c|c|c||c|c|c}
\hline
& \multicolumn{12}{c}{\textbf{Performance Metrics across 3 novel class groups.}} \\
\cline{2-13}
$\boldsymbol{\alpha}$ &
\multicolumn{3}{c||}{AUROC} &
\multicolumn{3}{c||}{AUPRC} &
\multicolumn{3}{c||}{OSCR} &
\multicolumn{3}{c}{Top-1 Acc.} \\
\cline{2-13}
density & 0 & 1 & 2 & 0 & 1 & 2 & 0 & 1 & 2 & 0 & 1 & 2 \\
\hline\hline
2  & 0.98 & 0.99 & 0.98 & 0.90 & 0.96 & 0.94 & 0.79 & 0.83 & 0.79 & 0.81 & 0.83 & 0.81 \\
3  & 0.98 & 0.99 & 0.98 & 0.93 & 0.96 & 0.94 & 0.81 & 0.84 & 0.80 & 0.82 & 0.85 & 0.82 \\
4  & 0.98 & 0.99 & 0.98 & 0.94 & 0.97 & 0.94 & 0.81 & 0.86 & 0.80 & 0.83 & 0.86 & 0.82 \\
5  & 0.98 & 0.99 & 0.98 & 0.93 & 0.97 & 0.95 & 0.81 & 0.84 & 0.81 & 0.83 & 0.84 & 0.83 \\
6  & 0.99 & 0.99 & 0.98 & 0.94 & 0.97 & 0.91 & 0.83 & 0.85 & 0.75 & 0.84 & 0.86 & 0.76 \\
7  & 0.98 & 0.99 & 0.98 & 0.93 & 0.97 & 0.94 & 0.82 & 0.86 & 0.80 & 0.83 & 0.87 & 0.82 \\
8  & 0.99 & 0.99 & 0.98 & 0.95 & 0.97 & 0.94 & 0.84 & 0.86 & 0.81 & 0.85 & 0.87 & 0.83 \\
9  & 0.99 & 0.99 & 0.99 & 0.95 & 0.97 & 0.95 & 0.83 & 0.85 & 0.81 & 0.84 & 0.86 & 0.83 \\
10 & 0.98 & 0.99 & 0.99 & 0.93 & 0.96 & 0.95 & 0.82 & 0.83 & 0.83 & 0.84 & 0.84 & 0.85 \\
\hline
\end{tabular}
\end{table*}

\begin{table*}[h!]
\centering
\caption{Effect of range of search grid $\boldsymbol{\alpha}$ on the performance of \ours{}. We keep the density (i.e. number of candidate $\alpha$s in the search interval) constant to 10 and vary the range. 
We report AUROC, AUPRC, OSCR, and Top-1 Accuracy for different novel class groupings: (Group 0: [water, ice, snow, etc.]; Group 1: [mountains, hills, desert, sky, etc.]; Group 2:[forest, field, jungle, etc.]). We use \resnet{} features here that are pretrained on ImageNet1K.}
\label{tab:alpha_range_ablation}
\begin{tabular}{c||c|c|c||c|c|c||c|c|c||c|c|c}
\hline
& \multicolumn{12}{c}{\textbf{Performance Metrics across 3 novel class groups.}} \\
\cline{2-13}
$\boldsymbol{\alpha}$ &
\multicolumn{3}{c||}{AUROC} &
\multicolumn{3}{c||}{AUPRC} &
\multicolumn{3}{c||}{OSCR} &
\multicolumn{3}{c}{Top-1 Acc.} \\
\cline{2-13}
range & 0 & 1 & 2 & 0 & 1 & 2 & 0 & 1 & 2 & 0 & 1 & 2 \\
\hline\hline
$\left[0.001, 0.01\right]$ & 0.98 & 0.99 & 0.98 & 0.90 & 0.96 & 0.94 & 0.79 & 0.83 & 0.79 & 0.81 & 0.83 & 0.81 \\
$\left[0.01, 0.1\right]$ & 0.98 & 0.99 & 0.98 & 0.93 & 0.96 & 0.94 & 0.81 & 0.84 & 0.80 & 0.82 & 0.85 & 0.82 \\
$\left[0.1, 1.0\right]$ & 0.98 & 0.99 & 0.98 & 0.94 & 0.97 & 0.94 & 0.81 & 0.86 & 0.80 & 0.83 & 0.86 & 0.82 \\
$\left[0.2, 0.45\right]$ & 0.97 & 0.98 & 0.99 & 0.59 & 0.87 & 0.95 & 0.72 & 0.77 & 0.83 & 0.75 & 0.79 & 0.85 \\
$\left[0.3, 0.45\right]$ & 0.98 & 0.99 & 0.99 & 0.70 & 0.95 & 0.94 & 0.75 & 0.78 & 0.76 & 0.77 & 0.78 & 0.77 \\
$\left[0.4, 0.45\right]$ & 0.95 & 0.99 & 0.98 & 0.58 & 0.95 & 0.87 & 0.71 & 0.78 & 0.75 & 0.75 & 0.79 & 0.77 \\
$\left[0.5, 1.0\right]$ & 0.96 & 0.99 & 0.98 & 0.69 & 0.96 & 0.90 & 0.74 & 0.77 & 0.76 & 0.77 & 0.78 & 0.77 \\
$\left[0.6, 1.0\right]$ & 0.97 & 0.99 & 0.99 & 0.61 & 0.96 & 0.97 & 0.75 & 0.77 & 0.76 & 0.78 & 0.78 & 0.76 \\
$\left[0.7, 1.0\right]$ & 0.97 & 0.99 & 0.99 & 0.61 & 0.92 & 0.91 & 0.74 & 0.78 & 0.72 & 0.77 & 0.78 & 0.73 \\
$\left[0.8, 1.0\right]$ & 0.97 & 0.99 & 0.99 & 0.63 & 0.92 & 0.96 & 0.74 & 0.78 & 0.73 & 0.76 & 0.78 & 0.73 \\
$\left[0.9, 1.0\right]$ & 0.97 & 0.99 & 0.99 & 0.74 & 0.99 & 0.94 & 0.73 & 0.77 & 0.75 & 0.76 & 0.78 & 0.76 \\
\hdashline
$\left[0.02, 0.45\right]$ (ours) & 0.98 & 0.0.99 & 0.0.99 & 0.93 & 0.96 & 0.95 & 0.82 & 0.83 & 0.83 & 0.84 & 0.84 & 0.85 \\
\hline
\end{tabular}
\end{table*}

\begin{table*}[h!]
\centering
\caption{Impact of joint learning on closed-set classification and novelty detection on SUN397 across three novel groups. We compare three model variants: (i) \source{}, using only closed-set classification heads ($w^c$); (ii) \conoc{}, using only novelty detection heads ($w^\alpha$) learned via constrained optimization; and (iii) the full \ours{} model trained with a multitask objective for both classification and novelty detection. All methods are evaluated on $\datatarget$ using ImageNet-pretrained \resnet{} features.}
\label{table:joint_learning_ablation_resnet}
\resizebox{\textwidth}{!}{%
\centering
\begin{tabular}{c || c || c | c | c || c }
\hline
Metric & Method & \multicolumn{3}{c||}{Novel Classes (natural outdoor scenes/places)} &  \\
\cline{3-5}
 & & $\alpha$ $=0.06\pm0.01$ & $\alpha$ $=0.06\pm0.01$ & $\alpha$ $=0.08\pm0.04$ & \\
 \cline{3-5}
 & & [water, ice, snow, etc.] & [mountains, hills, desert, sky, etc.] & [forest, field, jungle, etc.] & Summary \\
\hline\hline
\multirow{2}{*}{AUROC} & \conoc{} & $0.96\pm0.02$ & $0.97\pm0.04$ & $0.97\pm0.02$ &                                  $0.96\pm0.03$ \\
                       & \ours{} & $\boldsymbol{0.98\pm0.02}$ & $\boldsymbol{0.98\pm0.01}$ & $\boldsymbol{0.98\pm0.02}$ & $\boldsymbol{0.98\pm0.02}$ \\
\hline
\multirow{2}{*}{AUPRC} & \conoc{} & $0.77\pm0.11$ & $0.78\pm0.21$ & $0.87\pm0.08$ &                                  $0.80\pm0.14$ \\
                       & \ours{} & $\boldsymbol{0.92\pm0.04}$ & $\boldsymbol{0.93\pm0.05}$ & $\boldsymbol{0.89\pm0.15}$ & $\boldsymbol{0.91\pm0.09}$ \\
\hline
\multirow{2}{*}{Top-1 Acc.} & \source{} & $0.71\pm0.03$ & $0.73\pm0.05$ & $0.72\pm0.03$ &                                $0.72\pm0.03$ \\
                            & \ours{} & $\boldsymbol{0.77\pm0.11}$ & $\boldsymbol{0.78\pm0.21}$ & $\boldsymbol{0.87\pm0.08}$ & $\boldsymbol{0.80\pm0.14}$ \\
\hline
OSCR & \ours{} & $\boldsymbol{0.82\pm0.03}$ & $\boldsymbol{0.81\pm0.03}$ & $\boldsymbol{0.81\pm0.05}$ & $\boldsymbol{0.81\pm0.04}$\\
\hline
\end{tabular}
}

\end{table*}

\begin{table*}[h!]
\centering
\caption{Impact of joint learning on closed-set classification and novelty detection on SUN397 across three novel groups. We compare three model variants: (i) \source{}, using only closed-set classification heads ($w^c$); (ii) \conoc{}, using only novelty detection heads ($w^\alpha$) learned via constrained optimization; and (iii) the full \ours{} model trained with a multitask objective for both classification and novelty detection. All the methods are evaluated on $\datatarget$ and using pretrained CLIP \vit{} features.}
\label{table:joint_learning_ablation_vit}
\resizebox{\textwidth}{!}{%
\centering
\begin{tabular}{c || c || c | c | c || c }
\hline
Metric & Method & \multicolumn{3}{c||}{Novel Classes (natural outdoor scenes/places)} &  \\
\cline{3-5}
 & & $\alpha$ $=0.06\pm0.01$ & $\alpha$ $=0.06\pm0.01$ & $\alpha$ $=0.08\pm0.04$ & \\
 \cline{3-5}
 & & [water, ice, snow, etc.] & [mountains, hills, desert, sky, etc.] & [forest, field, jungle, etc.] & Summary \\
\hline\hline
\multirow{2}{*}{AUROC} & \conoc{} & $\boldsymbol{0.99\pm0.01}$ & $\boldsymbol{0.99\pm0.00}$ &                        $\boldsymbol{0.99\pm0.00}$ & $\boldsymbol{0.99\pm0.01}$ \\
                       & \ours{} & $\boldsymbol{0.99\pm0.02}$ & $\boldsymbol{0.99\pm0.01}$ & $\boldsymbol{0.99\pm0.01}$ & $\boldsymbol{0.99\pm0.01}$ \\
\hline
\multirow{2}{*}{AUPRC} & \conoc{} & $\boldsymbol{0.97\pm0.02}$ & $\boldsymbol{0.96\pm0.03}$ &                       $\boldsymbol{0.97\pm0.02}$ & $\boldsymbol{0.97\pm0.02}$ \\
                       & \ours{} & $0.95\pm0.04$ & $0.95\pm0.01$ & $0.96\pm0.03$ & $0.95\pm0.03$ \\
\hline
\multirow{2}{*}{Top-1 Acc.} & \source{} & $0.97\pm0.01$ & $0.97\pm0.01$ & $0.96\pm0.01$ &                                $0.97\pm0.01$ \\
                            & \ours{} & $\boldsymbol{0.98\pm0.01}$ & $\boldsymbol{0.98\pm0.00}$ & $\boldsymbol{0.97\pm0.01}$ & $\boldsymbol{0.98\pm0.01}$ \\
\hline
OSCR & \ours{} & $\boldsymbol{0.96\pm0.02}$ & $\boldsymbol{0.97\pm0.01}$ & $\boldsymbol{0.96\pm0.01}$ & $\boldsymbol{0.96\pm0.01}$\\
\hline
\end{tabular}
}

\end{table*}

\begin{figure}[t]
    \centering
    \caption{
        Synthetic experiment setup following the \cref{def:overparam_setting} to measure novelty detection performance as a function of angular separation between $\mu$ and $\eta$ vectors, $\theta$, (Top) and $r_\eta$ (Bottom). We measure AUROC (Left) and AUPRC (Right). Here, $\alpha=0.15$, $d=3000$, $\sigma=1/d$, $r_\mu=0.1$. We train on 2000 samples and test on 3000 samples.
    }
    \label{fig:theta_r_eta_sweeps}

    \begin{subfigure}[t]{0.48\linewidth}
        \centering
        \includegraphics[width=\linewidth]{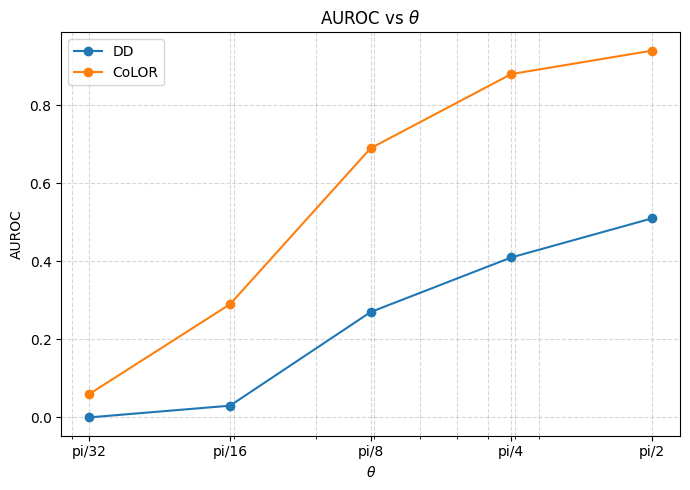}
        \caption{AUROC vs $\theta$}
        \label{fig:auroc_theta}
    \end{subfigure}
    \hfill
    \begin{subfigure}[t]{0.48\linewidth}
        \centering
        \includegraphics[width=\linewidth]{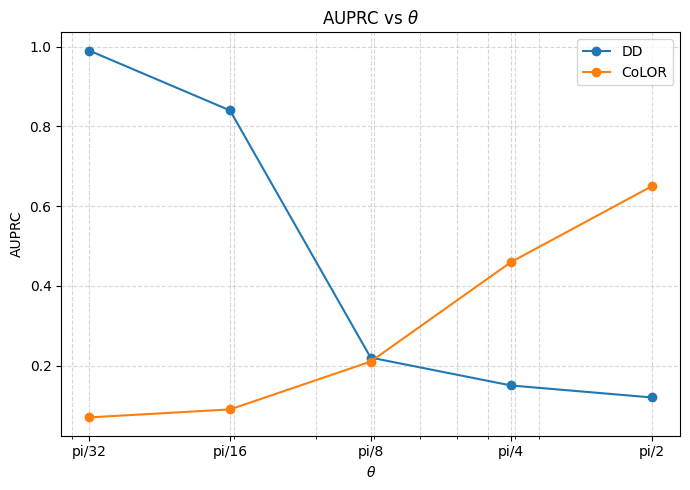}
        \caption{AUPRC vs $\theta$}
        \label{fig:auprc_theta}
    \end{subfigure}

    \vspace{0.8em}

    \begin{subfigure}[t]{0.48\linewidth}
        \centering
        \includegraphics[width=\linewidth]{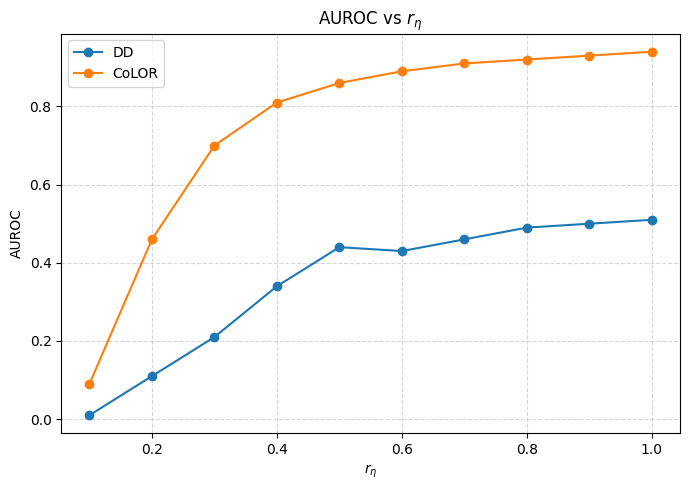}
        \caption{AUROC vs $r_\eta$}
        \label{fig:auroc_r_eta}
    \end{subfigure}
    \hfill
    \begin{subfigure}[t]{0.48\linewidth}
        \centering
        \includegraphics[width=\linewidth]{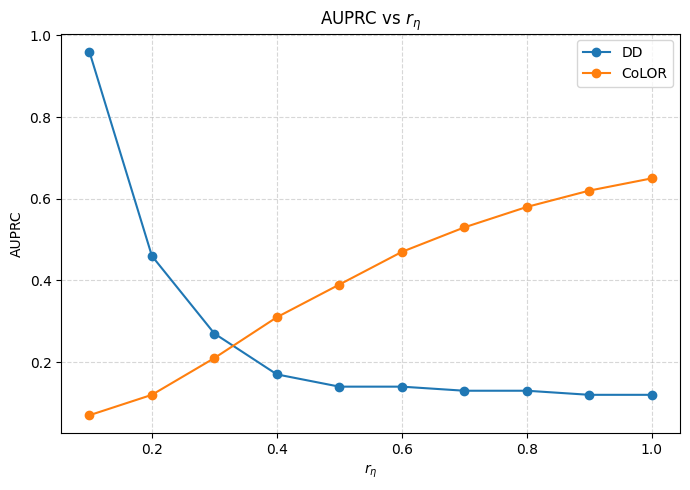}
        \caption{AUPRC vs $r_\eta$}
        \label{fig:auprc_r_eta}
    \end{subfigure}

\end{figure}

\begin{table}[t]
\centering
\caption{Novelty detection performance as a function of angular separation between $\mu$ and $\eta$ vectors, $\theta$, in synthetic experiment setup following the \cref{def:overparam_setting}. Here, $\alpha=0.15$, $d=3000$, $\sigma=1/d$, $r_\mu=0.1$, $r_\eta=1.0$. We train on 2000 samples and test on 3000 samples.}
\label{tab:theta_sweep}
\begin{tabular}{c||cc|cc}
\hline
\multirow{2}{*}{$\theta$}
& \multicolumn{2}{c|}{DD}
& \multicolumn{2}{c}{CoLOR} \\
\cline{2-5}
& AUROC & AUPRC & AUROC & AUPRC \\
\hline
$\pi/2$   & 0.51 & 0.12 & 0.94 & 0.65 \\
$\pi/4$   & 0.41 & 0.15 & 0.88 & 0.46 \\
$\pi/8$   & 0.27 & 0.22 & 0.69 & 0.21 \\
$\pi/16$  & 0.03 & 0.84 & 0.29 & 0.09 \\
$\pi/32$  & 0.00 & 0.99 & 0.06 & 0.07 \\
\hline
\end{tabular}
\end{table}

\begin{table}[t]
\centering
\caption{Novelty detection performance as a function of $r_\eta$ in synthetic experiment setup following the \cref{def:overparam_setting}. $\theta=\frac{\pi}{2}$ is the angular separation between $\mu$ and $\eta$ vectors. Here, $\alpha=0.15$, $d=3000$, $\sigma=1/d$, $r_\mu=0.1$. We train on 2000 samples and test on 3000 samples.}
\label{tab:rz_sweep}
\begin{tabular}{c||cc|cc}
\hline
\multirow{2}{*}{$r_\eta$}
& \multicolumn{2}{c|}{DD}
& \multicolumn{2}{c}{CoLOR} \\
\cline{2-5}
& AUROC & AUPRC & AUROC & AUPRC \\
\hline
1.0 & 0.51 & 0.12 & 0.94 & 0.65 \\
0.9 & 0.50 & 0.12 & 0.93 & 0.62 \\
0.8 & 0.49 & 0.13 & 0.92 & 0.58 \\
0.7 & 0.46 & 0.13 & 0.91 & 0.53 \\
0.6 & 0.43 & 0.14 & 0.89 & 0.47 \\
0.5 & 0.44 & 0.14 & 0.86 & 0.39 \\
0.4 & 0.34 & 0.17 & 0.81 & 0.31 \\
0.3 & 0.21 & 0.27 & 0.70 & 0.21 \\
0.2 & 0.11 & 0.46 & 0.46 & 0.12 \\
0.1 & 0.01 & 0.96 & 0.09 & 0.07 \\
\hline
\end{tabular}
\end{table}

\subsection{Analysis of toy example}
We have $N_{\mathcal{S}}=|\datasource|$, $N_{\mathcal{T},0} = (1-\alpha) |\datatarget|$, $N_{\mathcal{T},1} = \alpha |\datatarget|$ for examples from $\Psource = \N(\mu, \sigma (I_d-r_{\eta}^{-2}\eta\eta^\top))$, $P_{\mathcal{T}, 0} = \N(-\mu, \sigma(I_d-r_{\eta}^{-2}\eta\eta^\top)$ and $P_{\mathcal{T}, 1} =\N(\eta, \sigma (I_d - r^{-2}_{\mu}\mu\mu^\top))$ respectively. Here, $r_{\mu} = \| \mu \|$ and $r_{\eta}$ respectively, we set $\eta^\top \mu = 0$ and let the novelties and non-novel data lie in orthogonal subspaces to each other. This is for simplicity of analysis, and to mimic the case where at optimality we can perfectly detect the novelties.

Consider a linear model $\w \in{\reals^d}$, i.e. its bianry predictions are $\mathbf{1}_{\w^\top \rvx > 0}$ and the scores are $\w^\top \rvx$. We are interested in the AU-ROC obtained w.r.t novelty detection, i.e. distinguishing $P_{\mathcal{T}, 0}$ and  $P_{\mathcal{T}, 1}$. This is measured by $Q(\frac{\langle \w, 2\mu - \eta \rangle}{\| \w \|})$ where $Q$ is the Gaussian tail function.

We will compare two learning rules. Taking $\ell$ as the logistic loss, we are interested in:
\begin{align*}
\min_{w\in{\mathbb{R}^d}}{\sum_{\x_i\in \datasource}} \ell (w^\top x_i, 0) + \sum_{\x_i\in{\datatarget}} \ell (w^\top \tilde{x}_i, 1),
\end{align*}
and
\begin{align*}
\min_{w\in{\mathbb{R}^d}}{\sum_{\x_i\in{\datatarget}} \ell (w^\top \tilde{x}_i, 1)} ~\text{s.t. }~ w^\top x_i < 0  ~ \forall \x_i\in{\datasource}
\end{align*}

\subsection{Synthetic Experiment}
To study the impact of separability between novel class and known classes, we conduct a synthetic experiment following the theoretical setup defined in definition 2 of the main draft as it allows us to explicitly control separability. In this study, we consider a simple Gaussian mixture setup and vary the angular separation between known and novel class means as well as the absolute distance of the novel class mean ($r_\eta$) to gradually vary the novel class separability. This enables us to directly examine how CoLOR’s performance changes as the separability assumption is increasingly violated. Please refer to \cref{fig:theta_r_eta_sweeps}, \cref{tab:theta_sweep} and \cref{tab:rz_sweep} for results and further details.

%% file: main.bib
@article{blanchard2010semi,
  title={Semi-supervised novelty detection},
  author={Blanchard, Gilles and Lee, Gyemin and Scott, Clayton},
  journal={The Journal of Machine Learning Research},
  volume={11},
  pages={2973--3009},
  year={2010},
  publisher={JMLR. org}
}

@article{garg2021mixture,
  title={Mixture proportion estimation and pu learning: A modern approach},
  author={Garg, Saurabh and Wu, Yifan and Smola, Alexander J and Balakrishnan, Sivaraman and Lipton, Zachary},
  journal={Advances in Neural Information Processing Systems},
  volume={34},
  pages={8532--8544},
  year={2021}
}

@misc{garg22adaptation,
  doi = {10.48550/ARXIV.2207.13048},
  url = {https://arxiv.org/abs/2207.13048},
  author = {Garg, Saurabh and Balakrishnan, Sivaraman and Lipton, Zachary C.},
  title = {Domain Adaptation under Open Set Label Shift},
  publisher = {arXiv},
  year = {2022},
  copyright = {arXiv.org perpetual, non-exclusive license}
}

@inproceedings{bekker2019beyond,
  title={Beyond the selected completely at random assumption for learning from positive and unlabeled data},
  author={Bekker, Jessa and Robberechts, Pieter and Davis, Jesse},
  booktitle={Joint European Conference on Machine Learning and Knowledge Discovery in Databases},
  pages={71--85},
  year={2019},
  organization={Springer}
}

@article{bekker2020learning,
  title={Learning from positive and unlabeled data: A survey},
  author={Bekker, Jessa and Davis, Jesse},
  journal={Machine Learning},
  volume={109},
  number={4},
  pages={719--760},
  year={2020},
  publisher={Springer}
}

@article{chamon2022constrained,
  title={Constrained learning with non-convex losses},
  author={Chamon, Luiz FO and Paternain, Santiago and Calvo-Fullana, Miguel and Ribeiro, Alejandro},
  journal={IEEE Transactions on Information Theory},
  year={2022},
  publisher={IEEE}
}

@article{cotter2019optimization,
  title={Optimization with Non-Differentiable Constraints with Applications to Fairness, Recall, Churn, and Other Goals.},
  author={Cotter, Andrew and Jiang, Heinrich and Gupta, Maya R and Wang, Serena and Narayan, Taman and You, Seungil and Sridharan, Karthik},
  journal={J. Mach. Learn. Res.},
  volume={20},
  number={172},
  pages={1--59},
  year={2019}
}

@article{gerych2022recovering,
  title={Recovering The Propensity Score From Biased Positive Unlabeled Data},
  author={Gerych, Walter and Hartvigsen, Thomas and Buquicchio, Luke and Agu, Emmanuel and Rundensteiner, Elke},
  year={2022}
}

@inproceedings{elkan2008learning,
  title={Learning classifiers from only positive and unlabeled data},
  author={Elkan, Charles and Noto, Keith},
  booktitle={Proceedings of the 14th ACM SIGKDD international conference on Knowledge discovery and data mining},
  pages={213--220},
  year={2008}
}

@article{kiryo2017positive,
  title={Positive-unlabeled learning with non-negative risk estimator},
  author={Kiryo, Ryuichi and Niu, Gang and Du Plessis, Marthinus C and Sugiyama, Masashi},
  journal={Advances in neural information processing systems},
  volume={30},
  year={2017}
}

@article{krizhevsky2009learning,
  title={Learning multiple layers of features from tiny images},
  author={Krizhevsky, Alex and Hinton, Geoffrey and others},
  year={2009},
  publisher={Toronto, ON, Canada}
}

@article{bendavid2010adaptation,
	Author = {Ben-David, Shai and Blitzer, John and Crammer, Koby and Kulesza, Alex and Pereira, Fernando and Vaughan, Jennifer Wortman},
	Journal = {Machine Learning},
	Number = {1},
	Pages = {151--175},
	Title = {A theory of learning from different domains},
	Volume = {79},
	Year = {2010}}

@InProceedings{pmlr-v139-koh21a,
  title = 	 {WILDS: A Benchmark of in-the-Wild Distribution Shifts},
  author =       {Koh, Pang Wei and Sagawa, Shiori and Marklund, Henrik and Xie, Sang Michael and Zhang, Marvin and Balsubramani, Akshay and Hu, Weihua and Yasunaga, Michihiro and Phillips, Richard Lanas and Gao, Irena and Lee, Tony and David, Etienne and Stavness, Ian and Guo, Wei and Earnshaw, Berton and Haque, Imran and Beery, Sara M and Leskovec, Jure and Kundaje, Anshul and Pierson, Emma and Levine, Sergey and Finn, Chelsea and Liang, Percy},
  booktitle = 	 {Proceedings of the 38th International Conference on Machine Learning},
  pages = 	 {5637--5664},
  year = 	 {2021},
  editor = 	 {Meila, Marina and Zhang, Tong},
  volume = 	 {139},
  series = 	 {Proceedings of Machine Learning Research},
  month = 	 {18--24 Jul},
  publisher =    {PMLR},
  url = 	 {https://proceedings.mlr.press/v139/koh21a.html}
}

@book{quinonero2008dataset,
  title={Dataset shift in machine learning},
  author={Quinonero-Candela, Joaquin and Sugiyama, Masashi and Schwaighofer, Anton and Lawrence, Neil D},
  year={2008},
  publisher={MIT Press}
}

@inproceedings{duplessis2014analysis,
	Author = {du Plessis, Marthinus C and Niu, Gang and Sugiyama, Masashi},
	Booktitle = {Advances in Neural Information Processing Systems},
	Editor = {Z. Ghahramani and M. Welling and C. Cortes and N. Lawrence and K.Q. Weinberger},
	Publisher = {Curran Associates, Inc.},
	Title = {Analysis of Learning from Positive and Unlabeled Data},
	Volume = {27},
	Year = {2014}
}

@inproceedings{liu2018open,
  title={Open category detection with PAC guarantees},
  author={Liu, Si and Garrepalli, Risheek and Dietterich, Thomas and Fern, Alan and Hendrycks, Dan},
  booktitle={International Conference on Machine Learning},
  pages={3169--3178},
  year={2018},
  organization={PMLR}
}

@article{finlayson2021clinician,
  title={The clinician and dataset shift in artificial intelligence},
  author={Finlayson, Samuel G and Subbaswamy, Adarsh and Singh, Karandeep and Bowers, John and Kupke, Annabel and Zittrain, Jonathan and Kohane, Isaac S and Saria, Suchi},
  journal={New England Journal of Medicine},
  volume={385},
  number={3},
  pages={283--286},
  year={2021},
  publisher={Mass Medical Soc}
}

@article{wald2022malign,
  title={Malign Overfitting: Interpolation Can Provably Preclude Invariance},
  author={Wald, Yoav and Yona, Gal and Shalit, Uri and Carmon, Yair},
  journal={arXiv preprint arXiv:2211.15724},
  year={2022}
}

@inproceedings{xu2019open,
  title={Open-world learning and application to product classification},
  author={Xu, Hu and Liu, Bing and Shu, Lei and Yu, P},
  booktitle={The World Wide Web Conference},
  pages={3413--3419},
  year={2019}
}

@inproceedings{wald2023birds,
  title={Birds of an odd feather: guaranteed out-of-distribution (OOD) novel category detection},
  author={Wald, Yoav and Saria, Suchi},
  booktitle={Uncertainty in Artificial Intelligence},
  pages={2179--2191},
  year={2023},
  organization={PMLR}
}

@article{shimodaira2000improving,
  title={Improving predictive inference under covariate shift by weighting the log-likelihood function},
  author={Shimodaira, Hidetoshi},
  journal={Journal of statistical planning and inference},
  volume={90},
  number={2},
  pages={227--244},
  year={2000},
  publisher={Elsevier}
}

@article{amodei2016concrete,
  title={Concrete problems in AI safety},
  author={Amodei, Dario and Olah, Chris and Steinhardt, Jacob and Christiano, Paul and Schulman, John and Man{\'e}, Dan},
  journal={arXiv preprint arXiv:1606.06565},
  year={2016}
}

@inproceedings{hendrycks2016baseline,
  title={A Baseline for Detecting Misclassified and Out-of-Distribution Examples in Neural Networks},
  author={Hendrycks, Dan and Gimpel, Kevin},
  booktitle={International Conference on Learning Representations},
  year={2016}
}

@article{ruff2021unifying,
  title={A unifying review of deep and shallow anomaly detection},
  author={Ruff, Lukas and Kauffmann, Jacob R and Vandermeulen, Robert A and Montavon, Gr{\'e}goire and Samek, Wojciech and Kloft, Marius and Dietterich, Thomas G and M{\"u}ller, Klaus-Robert},
  journal={Proceedings of the IEEE},
  volume={109},
  number={5},
  pages={756--795},
  year={2021},
  publisher={IEEE}
}

@inproceedings{ni_justifying_2019,
	address = {Hong Kong, China},
	title = {Justifying {Recommendations} using {Distantly}-{Labeled} {Reviews} and {Fine}-{Grained} {Aspects}},
	url = {https://www.aclweb.org/anthology/D19-1018},
	doi = {10.18653/v1/D19-1018},
	language = {en},
	urldate = {2024-02-04},
	booktitle = {Proceedings of the 2019 {Conference} on {Empirical} {Methods} in {Natural} {Language} {Processing} and the 9th {International} {Joint} {Conference} on {Natural} {Language} {Processing} ({EMNLP}-{IJCNLP})},
	publisher = {Association for Computational Linguistics},
	author = {Ni, Jianmo and Li, Jiacheng and McAuley, Julian},
	year = {2019},
}

@inproceedings{he_deep_2016,
	address = {Las Vegas, NV, USA},
	title = {Deep {Residual} {Learning} for {Image} {Recognition}},
	isbn = {978-1-4673-8851-1},
	url = {http://ieeexplore.ieee.org/document/7780459/},
	doi = {10.1109/CVPR.2016.90},
	language = {en},
	urldate = {2024-02-04},
	booktitle = {2016 {IEEE} {Conference} on {Computer} {Vision} and {Pattern} {Recognition} ({CVPR})},
	publisher = {IEEE},
	author = {He, Kaiming and Zhang, Xiangyu and Ren, Shaoqing and Sun, Jian},
	month = jun,
	year = {2016},
	pages = {770--778},
}

@misc{liu_roberta_2019,
	title = {{RoBERTa}: {A} {Robustly} {Optimized} {BERT} {Pretraining} {Approach}},
	shorttitle = {{RoBERTa}},
	url = {http://arxiv.org/abs/1907.11692},
	urldate = {2024-02-05},
	publisher = {arXiv},
	author = {Liu, Yinhan and Ott, Myle and Goyal, Naman and Du, Jingfei and Joshi, Mandar and Chen, Danqi and Levy, Omer and Lewis, Mike and Zettlemoyer, Luke and Stoyanov, Veselin},
	month = jul,
	year = {2019},
	note = {arXiv:1907.11692 [cs]},
}

@article{maurer_benet_16,
  author  = {Andreas Maurer and Massimiliano Pontil and Bernardino Romera-Paredes},
  title   = {The Benefit of Multitask Representation Learning},
  journal = {Journal of Machine Learning Research},
  year    = {2016},
  volume  = {17},
  number  = {81},
  pages   = {1--32},
  url     = {http://jmlr.org/papers/v17/15-242.html}
}

@article{baxter_model_2000,
	title = {A {Model} of {Inductive} {Bias} {Learning}},
	url = {https://arxiv.org/abs/1106.0245v1},
	doi = {10.1613/jair.731},
	language = {en},
	urldate = {2024-02-10},
	journal = {Journal Of Artificial Intelligence Research},
	author = {Baxter, J.},
	year = {2000},
}

@inproceedings{wang_tent_2020,
	title = {Tent: {Fully} {Test}-{Time} {Adaptation} by {Entropy} {Minimization}},
	shorttitle = {Tent},
	url = {https://openreview.net/forum?id=uXl3bZLkr3c},
	language = {en},
	urldate = {2024-02-03},
	author = {Wang, Dequan and Shelhamer, Evan and Liu, Shaoteng and Olshausen, Bruno and Darrell, Trevor},
	month = oct,
	year = {2020},
}

@misc{zhang_model_free_2023,
	title = {Model-free {Test} {Time} {Adaptation} for {Out}-{Of}-{Distribution} {Detection}},
	url = {http://arxiv.org/abs/2311.16420},
	doi = {10.48550/arXiv.2311.16420},
	urldate = {2024-02-09},
	publisher = {arXiv},
	author = {Zhang, YiFan and Wang, Xue and Zhou, Tian and Yuan, Kun and Zhang, Zhang and Wang, Liang and Jin, Rong and Tan, Tieniu},
	month = nov,
	year = {2023},
	note = {arXiv:2311.16420 [cs]},
}

@misc{mcdermott2024closer,
      title={A Closer Look at AUROC and AUPRC under Class Imbalance}, 
      author={Matthew B. A. McDermott and Lasse Hyldig Hansen and Haoran Zhang and Giovanni Angelotti and Jack Gallifant},
      year={2024},
      eprint={2401.06091},
      archivePrefix={arXiv},
      primaryClass={cs.LG}
}

@misc{saito2018open,
      title={Open Set Domain Adaptation by Backpropagation}, 
      author={Kuniaki Saito and Shohei Yamamoto and Yoshitaka Ushiku and Tatsuya Harada},
      year={2018},
      eprint={1804.10427},
      archivePrefix={arXiv},
      primaryClass={cs.CV}
}

@inproceedings{ge2017openmax,
title = "Generative OpenMax for multi-class open set classification",
author = "Zongyuan Ge and Sergey Demyanov and Zetao Chen and Rahil Garnavi",
year = "2017",
doi = "10.5244/C.31.42",
language = "English",
isbn = "190172560X",
series = "British Machine Vision Conference 2017, BMVC 2017",
publisher = "British Machine Vision Association",
editor = "Krystian Mikolajczyk and Gabriel Brostow",
booktitle = "British Machine Vision Conference Proceedings 2017",
note = "British Machine Vision Conference 2017, BMVC 2017 ; Conference date: 04-09-2017 Through 07-09-2017",
url = "https://bmvc2017.london/, https://dblp.org/db/conf/bmvc/bmvc2017.html",
}

@InProceedings{Neal_2018_ECCV,
author = {Neal, Lawrence and Olson, Matthew and Fern, Xiaoli and Wong, Weng-Keen and Li, Fuxin},
title = {Open Set Learning with Counterfactual Images},
booktitle = {Proceedings of the European Conference on Computer Vision (ECCV)},
month = {September},
year = {2018}
}

@article{Lin_2019,
   title={A post-processing method for detecting unknown intent of dialogue system via pre-trained deep neural network classifier},
   volume={186},
   ISSN={0950-7051},
   url={http://dx.doi.org/10.1016/j.knosys.2019.104979},
   DOI={10.1016/j.knosys.2019.104979},
   journal={Knowledge-Based Systems},
   publisher={Elsevier BV},
   author={Lin, Ting-En and Xu, Hua},
   year={2019},
   month=dec, pages={104979} 
}

@article{Chen_2021,
   title={Adversarial Reciprocal Points Learning for Open Set Recognition},
   ISSN={1939-3539},
   url={http://dx.doi.org/10.1109/TPAMI.2021.3106743},
   DOI={10.1109/tpami.2021.3106743},
   journal={IEEE Transactions on Pattern Analysis and Machine Intelligence},
   publisher={Institute of Electrical and Electronics Engineers (IEEE)},
   author={Chen, Guangyao and Peng, Peixi and Wang, Xiangqian and Tian, Yonghong},
   year={2021},
   pages={1–1} 
}

@inproceedings{zeng-etal-2021-modeling,
    title = "Modeling Discriminative Representations for Out-of-Domain Detection with Supervised Contrastive Learning",
    author = "Zeng, Zhiyuan  and
      He, Keqing  and
      Yan, Yuanmeng  and
      Liu, Zijun  and
      Wu, Yanan  and
      Xu, Hong  and
      Jiang, Huixing  and
      Xu, Weiran",
    editor = "Zong, Chengqing  and
      Xia, Fei  and
      Li, Wenjie  and
      Navigli, Roberto",
    booktitle = "Proceedings of the 59th Annual Meeting of the Association for Computational Linguistics and the 11th International Joint Conference on Natural Language Processing (Volume 2: Short Papers)",
    month = aug,
    year = "2021",
    address = "Online",
    publisher = "Association for Computational Linguistics",
    url = "https://aclanthology.org/2021.acl-short.110",
    doi = "10.18653/v1/2021.acl-short.110",
    pages = "870--878",
}

@InProceedings{vaze2022openset,
      title={Open-Set Recognition: a Good Closed-Set Classifier is All You Need?},
      author={Sagar Vaze and Kai Han and Andrea Vedaldi and Andrew Zisserman},
      booktitle={International Conference on Learning Representations},
      year={2022}
}

@article{Esmaeilpour_Liu_Robertson_Shu_2022, 
    title={Zero-Shot Out-of-Distribution Detection Based on the Pre-trained Model CLIP}, 
    volume={36}, 
    url={https://ojs.aaai.org/index.php/AAAI/article/view/20610}, DOI={10.1609/aaai.v36i6.20610},
    number={6}, 
    journal={Proceedings of the AAAI Conference on Artificial Intelligence}, author={Esmaeilpour, Sepideh and Liu, Bing and Robertson, Eric and Shu, Lei}, 
    year={2022}, 
    month={Jun.}, 
    pages={6568-6576} 
}

@TECHREPORT{Krizhevsky09learningmultiplecifar,
    author = {Alex Krizhevsky},
    title = {Learning Multiple Layers of Features from Tiny Images},
    institution = {},
    year = {2009}
}

@INPROCEEDINGS{jianxiong2010sun397,
  author={Xiao, Jianxiong and Hays, James and Ehinger, Krista A. and Oliva, Aude and Torralba, Antonio},
  booktitle={2010 IEEE Computer Society Conference on Computer Vision and Pattern Recognition}, 
  title={SUN database: Large-scale scene recognition from abbey to zoo}, 
  year={2010},
  volume={},
  number={},
  pages={3485-3492},
  doi={10.1109/CVPR.2010.5539970}
}

@article{Russakovsky2015imagenet,
Author = {Olga Russakovsky and Jia Deng and Hao Su and Jonathan Krause and Sanjeev Satheesh and Sean Ma and Zhiheng Huang and Andrej Karpathy and Aditya Khosla and Michael Bernstein and Alexander C. Berg and Li Fei-Fei},
Title = {{ImageNet Large Scale Visual Recognition Challenge}},
Year = {2015},
journal   = {International Journal of Computer Vision (IJCV)},
doi = {10.1007/s11263-015-0816-y},
volume={115},
number={3},
pages={211-252}
}

@article{radford2021clip,
  author       = {Alec Radford and
                  Jong Wook Kim and
                  Chris Hallacy and
                  Aditya Ramesh and
                  Gabriel Goh and
                  Sandhini Agarwal and
                  Girish Sastry and
                  Amanda Askell and
                  Pamela Mishkin and
                  Jack Clark and
                  Gretchen Krueger and
                  Ilya Sutskever},
  title        = {Learning Transferable Visual Models From Natural Language Supervision},
  journal      = {CoRR},
  volume       = {abs/2103.00020},
  year         = {2021},
  url          = {https://arxiv.org/abs/2103.00020},
  eprinttype    = {arXiv},
  eprint       = {2103.00020},
  bibsource    = {dblp computer science bibliography, https://dblp.org}
}

@inproceedings{wong2020identifying,
  title={Identifying unknown instances for autonomous driving},
  author={Wong, Kelvin and Wang, Shenlong and Ren, Mengye and Liang, Ming and Urtasun, Raquel},
  booktitle={Conference on Robot Learning},
  pages={384--393},
  year={2020},
  organization={PMLR}
}

@inproceedings{filos2020can,
  title={Can autonomous vehicles identify, recover from, and adapt to distribution shifts?},
  author={Filos, Angelos and Tigkas, Panagiotis and McAllister, Rowan and Rhinehart, Nicholas and Levine, Sergey and Gal, Yarin},
  booktitle={International Conference on Machine Learning},
  pages={3145--3153},
  year={2020},
  organization={PMLR}
}

@inproceedings{bendale2016towards,
  title={Towards open set deep networks},
  author={Bendale, Abhijit and Boult, Terrance E},
  booktitle={Proceedings of the IEEE conference on computer vision and pattern recognition},
  pages={1563--1572},
  year={2016}
}

@inproceedings{david2010impossibility,
  title={Impossibility theorems for domain adaptation},
  author={David, Shai Ben and Lu, Tyler and Luu, Teresa and P{\'a}l, D{\'a}vid},
  booktitle={Proceedings of the Thirteenth International Conference on Artificial Intelligence and Statistics},
  pages={129--136},
  year={2010},
  organization={JMLR Workshop and Conference Proceedings}
}

@inproceedings{panareda2017open,
  title={Open set domain adaptation},
  author={Panareda Busto, Pau and Gall, Juergen},
  booktitle={Proceedings of the IEEE international conference on computer vision},
  pages={754--763},
  year={2017}
}

@article{dhamija2018oscr,
  title={Reducing network agnostophobia},
  author={Dhamija, Akshay Raj and G{\"u}nther, Manuel and Boult, Terrance},
  journal={Advances in Neural Information Processing Systems},
  volume={31},
  year={2018}
}

@article{ben2010theory,
  title={A theory of learning from different domains},
  author={Ben-David, Shai and Blitzer, John and Crammer, Koby and Kulesza, Alex and Pereira, Fernando and Vaughan, Jennifer Wortman},
  journal={Machine learning},
  volume={79},
  pages={151--175},
  year={2010},
  publisher={Springer}
}

@article{zafar2019fairness,
  title={Fairness constraints: A flexible approach for fair classification},
  author={Zafar, Muhammad Bilal and Valera, Isabel and Gomez-Rodriguez, Manuel and Gummadi, Krishna P},
  journal={Journal of Machine Learning Research},
  volume={20},
  number={75},
  pages={1--42},
  year={2019}
}

@misc{qu2023lmc,
      title={LMC: Large Model Collaboration with Cross-assessment for Training-Free Open-Set Object Recognition}, 
      author={Haoxuan Qu and Xiaofei Hui and Yujun Cai and Jun Liu},
      year={2023},
      eprint={2309.12780},
      archivePrefix={arXiv},
      primaryClass={cs.CV}
}

@InProceedings{Kong_2021_ICCV_opengan,
    author    = {Kong, Shu and Ramanan, Deva},
    title     = {OpenGAN: Open-Set Recognition via Open Data Generation},
    booktitle = {Proceedings of the IEEE/CVF International Conference on Computer Vision (ICCV)},
    month     = {October},
    year      = {2021},
    pages     = {813-822}
}

@misc{liang2021reallyneedaccesssource,
      title={Do We Really Need to Access the Source Data? Source Hypothesis Transfer for Unsupervised Domain Adaptation}, 
      author={Jian Liang and Dapeng Hu and Jiashi Feng},
      year={2021},
      eprint={2002.08546},
      archivePrefix={arXiv},
      primaryClass={cs.CV},
      url={https://arxiv.org/abs/2002.08546}, 
}

@article{VenkateswaraECP17deephashing,
  author       = {Hemanth Venkateswara and
                  Jose Eusebio and
                  Shayok Chakraborty and
                  Sethuraman Panchanathan},
  title        = {Deep Hashing Network for Unsupervised Domain Adaptation},
  journal      = {CoRR},
  volume       = {abs/1706.07522},
  year         = {2017},
  url          = {http://arxiv.org/abs/1706.07522},
  eprinttype    = {arXiv},
  eprint       = {1706.07522},
  bibsource    = {dblp computer science bibliography, https://dblp.org}
}

@InProceedings{wen24crossdomain,
  title = 	 {Cross-domain Open-world Discovery},
  author =       {Wen, Shuo and Brbic, Maria},
  booktitle = 	 {Proceedings of the 41st International Conference on Machine Learning},
  pages = 	 {52744--52761},
  year = 	 {2024},
  editor = 	 {Salakhutdinov, Ruslan and Kolter, Zico and Heller, Katherine and Weller, Adrian and Oliver, Nuria and Scarlett, Jonathan and Berkenkamp, Felix},
  volume = 	 {235},
  series = 	 {Proceedings of Machine Learning Research},
  month = 	 {21--27 Jul},
  publisher =    {PMLR},
  url = 	 {https://proceedings.mlr.press/v235/wen24b.html}
}

@INPROCEEDINGS{li2021domainconsensus,
  author={Li, Guangrui and Kang, Guoliang and Zhu, Yi and Wei, Yunchao and Yang, Yi},
  booktitle={2021 IEEE/CVF Conference on Computer Vision and Pattern Recognition (CVPR)}, 
  title={Domain Consensus Clustering for Universal Domain Adaptation}, 
  year={2021},
  volume={},
  number={},
  pages={9752-9761},
  doi={10.1109/CVPR46437.2021.00963}}

@inproceedings{wuyang2023anna,
  author={Li, Wuyang and Liu, Jie and Han, Bo and Yuan, Yixuan},
  booktitle={2023 IEEE/CVF Conference on Computer Vision and Pattern Recognition (CVPR)}, 
  title={Adjustment and Alignment for Unbiased Open Set Domain Adaptation}, 
  year={2023},
  volume={},
  number={},
  pages={24110-24119},
  doi={10.1109/CVPR52729.2023.02309}
}

@INPROCEEDINGS{choe2024osdasemanticsegmentation,
  author={Choe, Seun-An and Shin, Ah-Hyung and Park, Keon-Hee and Choi, Jinwoo and Park, Gyeong-Moon},
  booktitle={2024 IEEE/CVF Conference on Computer Vision and Pattern Recognition (CVPR)}, 
  title={Open-Set Domain Adaptation for Semantic Segmentation}, 
  year={2024},
  volume={},
  number={},
  pages={23943-23953},
  doi={10.1109/CVPR52733.2024.02260}
}

@InProceedings{Zhu2023unida,
    author    = {Zhu, Didi and Li, Yinchuan and Yuan, Junkun and Li, Zexi and Kuang, Kun and Wu, Chao},
    title     = {Universal Domain Adaptation via Compressive Attention Matching},
    booktitle = {Proceedings of the IEEE/CVF International Conference on Computer Vision (ICCV)},
    month     = {October},
    year      = {2023},
    pages     = {6974-6985}
}

@InProceedings{Hur2023unida,
    author    = {Hur, Sungsu and Shin, Inkyu and Park, Kwanyong and Woo, Sanghyun and Kweon, In So},
    title     = {Learning Classifiers of Prototypes and Reciprocal Points for Universal Domain Adaptation},
    booktitle = {Proceedings of the IEEE/CVF Winter Conference on Applications of Computer Vision (WACV)},
    month     = {January},
    year      = {2023},
    pages     = {531-540}
}

@article{pocevivciute2025out,
  title={Out-of-distribution detection in digital pathology: Do foundation models bring the end to reconstruction-based approaches?},
  author={Pocevi{\v{c}}i{\=u}t{\.e}, Milda and Ding, Yifan and Brom{\'e}e, Ruben and Eilertsen, Gabriel},
  journal={Computers in Biology and Medicine},
  volume={184},
  pages={109327},
  year={2025},
  publisher={Elsevier}
}

@article{redko2020survey,
  title={A survey on domain adaptation theory: learning bounds and theoretical guarantees},
  author={Redko, Ievgen and Morvant, Emilie and Habrard, Amaury and Sebban, Marc and Bennani, Youn{\`e}s},
  journal={arXiv preprint arXiv:2004.11829},
  year={2020}
}

@inproceedings{sagawa2020investigation,
  title={An investigation of why overparameterization exacerbates spurious correlations},
  author={Sagawa, Shiori and Raghunathan, Aditi and Koh, Pang Wei and Liang, Percy},
  booktitle={International Conference on Machine Learning},
  pages={8346--8356},
  year={2020},
  organization={PMLR}
}

@article{puli2023don,
  title={Don’t blame dataset shift! shortcut learning due to gradients and cross entropy},
  author={Puli, Aahlad and Zhang, Lily and Wald, Yoav and Ranganath, Rajesh},
  journal={Advances in Neural Information Processing Systems},
  volume={36},
  pages={71874--71910},
  year={2023}
}

@inproceedings{
nagarajan2021understanding,
title={Understanding the failure modes of out-of-distribution generalization},
author={Vaishnavh Nagarajan and Anders Andreassen and Behnam Neyshabur},
booktitle={International Conference on Learning Representations},
year={2021},
url={https://openreview.net/forum?id=fSTD6NFIW_b}
}

@article{belkin2019reconciling,
  title={Reconciling modern machine-learning practice and the classical bias--variance trade-off},
  author={Belkin, Mikhail and Hsu, Daniel and Ma, Siyuan and Mandal, Soumik},
  journal={Proceedings of the National Academy of Sciences},
  volume={116},
  number={32},
  pages={15849--15854},
  year={2019},
  publisher={National Academy of Sciences}
}

@article{muthukumar2021classification,
  title={Classification vs regression in overparameterized regimes: Does the loss function matter?},
  author={Muthukumar, Vidya and Narang, Adhyyan and Subramanian, Vignesh and Belkin, Mikhail and Hsu, Daniel and Sahai, Anant},
  journal={Journal of Machine Learning Research},
  volume={22},
  number={222},
  pages={1--69},
  year={2021}
}

@article{soudry2018implicit,
  title={The implicit bias of gradient descent on separable data},
  author={Soudry, Daniel and Hoffer, Elad and Nacson, Mor Shpigel and Gunasekar, Suriya and Srebro, Nathan},
  journal={Journal of Machine Learning Research},
  volume={19},
  number={70},
  pages={1--57},
  year={2018}
}

@book{vershynin2018high,
  title={High-dimensional probability: An introduction with applications in data science},
  author={Vershynin, Roman},
  volume={47},
  year={2018},
  publisher={Cambridge university press}
}

@book{ledoux2013probability,
  title={Probability in Banach Spaces: isoperimetry and processes},
  author={Ledoux, Michel and Talagrand, Michel},
  year={2013},
  publisher={Springer Science \& Business Media}
}

@misc{visda2017,
    Author = {Xingchao Peng and Ben Usman and Neela Kaushik and Judy Hoffman and Dequan Wang and Kate Saenko},
    Title = {VisDA: The Visual Domain Adaptation Challenge},
    Year = {2017},
    Eprint = {arXiv:1710.06924},
}

@ARTICLE{scholkopf2001estimating,
  author={Schölkopf, Bernhard and Platt, John C. and Shawe-Taylor, John and Smola, Alex J. and Williamson, Robert C.},
  journal={Neural Computation}, 
  title={Estimating the Support of a High-Dimensional Distribution}, 
  year={2001},
  volume={13},
  number={7},
  pages={1443-1471},
  doi={10.1162/089976601750264965}}

@misc{cao2023anomalydetectiondistributionshift,
      title={Anomaly Detection under Distribution Shift}, 
      author={Tri Cao and Jiawen Zhu and Guansong Pang},
      year={2023},
      eprint={2303.13845},
      archivePrefix={arXiv},
      primaryClass={cs.CV},
      url={https://arxiv.org/abs/2303.13845}, 
}

@article{Miller2021ClassAC,
  title={Class Anchor Clustering: A Loss for Distance-based Open Set Recognition},
  author={Dimity Miller and Niko S{\"u}nderhauf and Michael Milford and Feras Dayoub},
  journal={2021 IEEE Winter Conference on Applications of Computer Vision (WACV)},
  year={2021},
  pages={3569-3577},
  url={https://api.semanticscholar.org/CorpusID:230091082}
}

@article{
lalou2025skadabench,
title={{SKADA}-Bench: Benchmarking Unsupervised Domain Adaptation Methods with Realistic Validation On Diverse Modalities},
author={Yanis Lalou and Theo Gnassounou and Antoine Collas and Antoine de Mathelin and Oleksii Kachaiev and Ambroise Odonnat and Thomas Moreau and Alexandre Gramfort and R{\'e}mi Flamary},
journal={Transactions on Machine Learning Research},
issn={2835-8856},
year={2025},
url={https://openreview.net/forum?id=k9F63DV3Qe},
note={}
}
